\documentclass[acmtog,screen]{acmart}

\setcopyright{cc}
\setcctype{by}
\acmJournal{TOG}
\acmYear{2026} \acmVolume{45} \acmNumber{6} \acmArticle{242}
\acmMonth{12} \acmDOI{10.1145/3842525}

\usepackage{booktabs}
\usepackage{amsfonts}

\usepackage{amssymb}
\usepackage{amsmath,amsthm,mathtools}

\usepackage{nicefrac}
\usepackage{xcolor}
\usepackage{graphicx}
\usepackage{subcaption}
\usepackage{algorithm}
\usepackage{algorithmic}
\usepackage{multirow}
\usepackage{wrapfig}
\usepackage{enumitem}
\usepackage{microtype}  
\usepackage{float}      

\definecolor{acmCiteColor}{RGB}{160, 0, 40}
\definecolor{acmLinkColor}{RGB}{26, 73, 138}
\definecolor{revisionBlue}{RGB}{0, 82, 174}
\newif\ifshowrevision
\showrevisionfalse 
\DeclareRobustCommand{\rev}[1]{%
  \ifshowrevision\textcolor{revisionBlue}{#1}\else#1\fi}
\newenvironment{revisionblock}%
  {\ifshowrevision\begingroup\color{revisionBlue}\fi}%
  {\ifshowrevision\endgroup\fi}
\AtBeginDocument{%
  \hypersetup{
    colorlinks=true,
    allcolors=black,
  }%
}

\theoremstyle{definition}
\newtheorem{theorem}{Theorem}[section]
\newtheorem{proposition}[theorem]{Proposition}
\newtheorem{lemma}[theorem]{Lemma}
\newtheorem{corollary}[theorem]{Corollary}
\newtheorem{definition}[theorem]{Definition}

\newcommand{\bfx}{\mathbf{x}}
\newcommand{\bfr}{\mathbf{r}}

\newcommand{\bfn}{\mathbf{n}}

\newcommand{\SO}{\mathrm{SO}}

\newcommand{\GE}{\mathrm{GE}}
\newcommand{\mIoU}{\mathrm{mIoU}}



\title{Gauge-Equivariant Attention for Rotation-Stable $360^\circ$ Scene Understanding}

\author{Tianjian Zhou}
\orcid{0009-0003-4822-3078}
\affiliation{\institution{National University of Defense Technology (NUDT)}\city{Changsha}\country{China}}
\email{zhoutianjian2000@nudt.edu.cn}

\author{Yishan Li}
\orcid{0000-0002-0425-1742}
\affiliation{\institution{National University of Defense Technology (NUDT)}\city{Changsha}\country{China}}
\email{liyishan@nudt.edu.cn}

\author{Jie Jiang}
\orcid{0000-0001-9666-815X}
\affiliation{\institution{National University of Defense Technology (NUDT)}\city{Changsha}\country{China}}
\email{jiejiang@nudt.edu.cn}
\authornote{Corresponding author: Jie Jiang (jiejiang@nudt.edu.cn).}

\author{Yifei Zhang}
\orcid{0000-0003-4185-8663}
\affiliation{\institution{Northwestern Polytechnical University}\city{Xi'an}\country{China}}
\email{yifeiacc@gmail.com}

\begin{document}

\begin{abstract}
Panoramic $360^\circ$ scene understanding increasingly relies on icosphere transformers, but a state-of-the-art spherical model loses more than half of its segmentation accuracy when the camera rotates by $90^\circ$, \rev{and controlled ablations identify gauge dependence in its relative-position bias as a major contributor.}
We propose \textbf{gauge-equivariant relative position encoding (GE-RPE)}: a parameter-free Reynolds average of the bias over a finite cyclic subgroup $C_n\!\subset\!\mathrm{SO}(2)$ of gauge rotations.
Plugged into a SphereUFormer backbone the change is invisible at deployment---zero added parameters and $1.5$--$4.4\%$ forward latency---\rev{and the matched three-seed GE-RPE model records a $1.3\%$ drop; the published SphereUFormer checkpoint records $53\%$ under the same stress protocol but a different training recipe.}
Once the gauge defect is removed and a teacher-token permutation $\pi_R$ aligns the SSL views to the rotated student frame, iBOT$+$MAE pretraining stops being a liability and becomes a clean low-label lever: the full framework \textbf{EquiSSL} (GE-RPE $+$ $\pi_R$ $+$ iBOT$+$MAE) tightens the drop to $0.8\%$ at $68.30\%$ val mIoU and lifts $1\%$-label fine-tuning by $+2.39$ mIoU on the $N{=}373$ test split (and by $+4.10$ on the smaller $N{=}40$ val split); the same fix carries over to monocular depth and to zero-shot Structured3D segmentation.
The construction is provably $C_n$-invariant and $\mathcal{O}(n^{-2})$-close to the continuous $\mathrm{SO}(2)$ average, making the resulting model a usable $360^\circ$ visual-computing primitive across panoramic relighting, immersive video, and cross-dataset transfer. Code is available at \url{https://github.com/Jaywalk18/equissl-release}.
\end{abstract}

\maketitle
\hypersetup{
    pdftitle={Gauge-Equivariant Attention for Rotation-Stable 360° Scene Understanding},
    pdfauthor={Tianjian Zhou, Yishan Li, Jie Jiang, Yifei Zhang},
    pdfkeywords={spherical perception, gauge equivariance, rotation robustness, self-supervised learning},
}

\section{Introduction}
\label{sec:intro}

\begin{figure}[!t]
  \centering
  \includegraphics[width=0.92\linewidth]{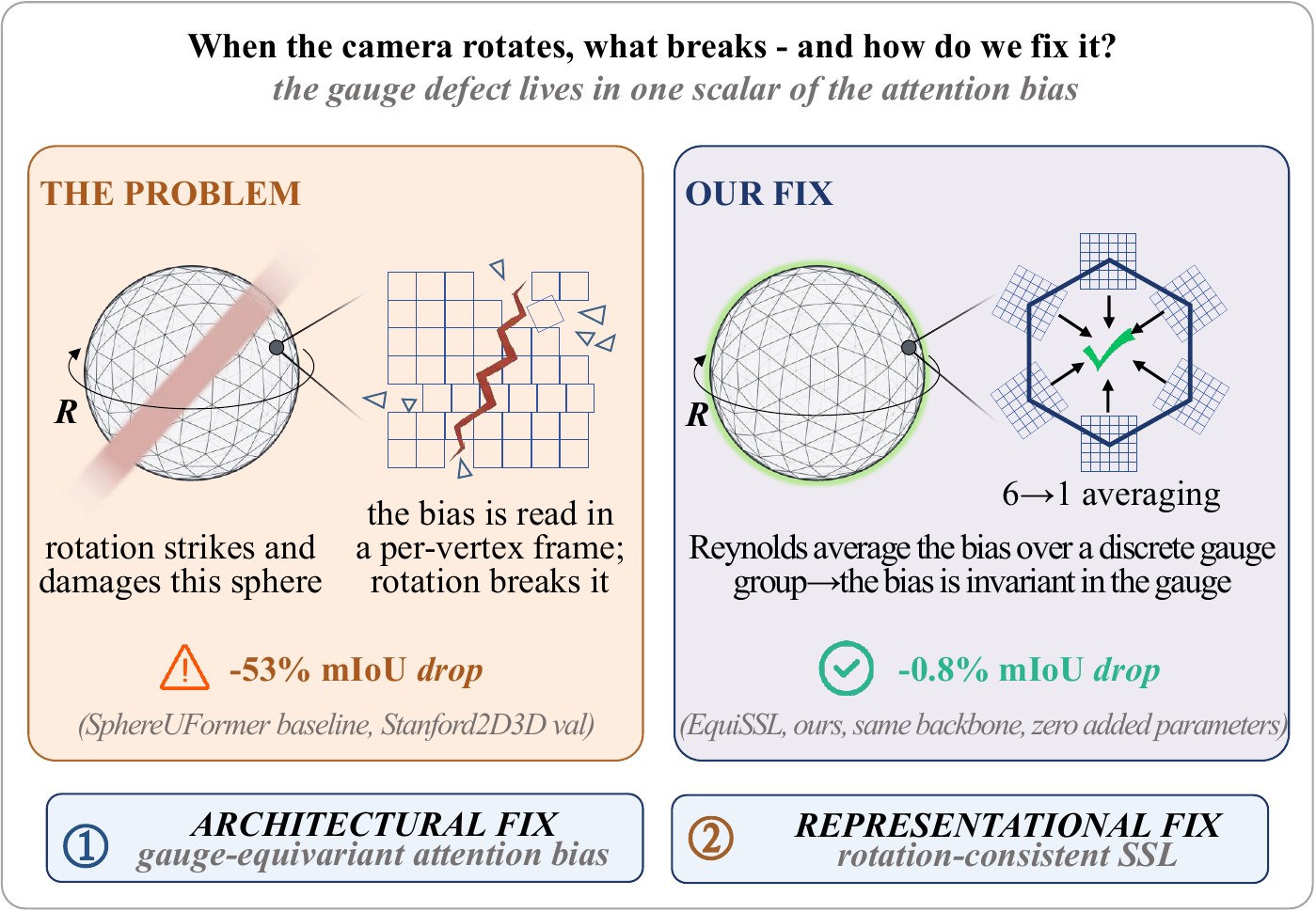}
  \caption{\textbf{Rotation breaks a standard spherical model; EquiSSL fixes it at two levels.}
  \emph{Top:} a standard SphereUFormer gives clean upright regions but collapses after a $90^\circ$ rotation (more than $50\%$ mIoU drop).
  \emph{Bottom:} EquiSSL couples GE-RPE, which averages the attention bias over a discrete gauge group so attention co-rotates with the input, with a teacher-token permutation $\pi_R$ that aligns the SSL target to the student's rotated frame.
  The result is stable segmentation before and after rotation.}
  \Description{Side-by-side panoramic segmentation results showing that a standard spherical transformer's predictions disintegrate under a 90 degree camera rotation while the proposed EquiSSL produces consistent segmentation before and after rotation; a schematic at the bottom indicates the two levels of the fix---gauge-equivariant attention and rotation-consistent self-supervised pretraining.}
  \label{fig:motivation}
\end{figure}

Panoramic $360^\circ$ cameras are now commodity hardware on consumer headsets, autonomous vehicles, and capture rigs, and panoramic scene understanding has become a load-bearing primitive for image-based lighting~\citep{debevec1998ibl}, virtual production, immersive video, and embodied perception. State-of-the-art spherical transformers operate on icosphere meshes with windowed attention and relative-position biases, and on \emph{upright} benchmarks they reach segmentation quality on par with planar models. The panorama in the wild, however, is rarely upright: the camera tilts, the rig rotates, and the same scene must be parsed from any viewpoint on $\mathrm{SO}(3)$. None of these settings has a privileged ``up'': in immersive video the headset reorients latitude--longitude between frames. On consumer-grade capture rigs the up vector is rarely calibrated, and commercial HDRI panoramas ship with arbitrary roll. A model that only works on the upright orbit of $\mathrm{SO}(3)$ is therefore not a usable visual-computing primitive---it is a planar method dressed in spherical coordinates. Under that natural stress test the leading published icosphere transformer with a learnable RPE bias does not gracefully degrade---it collapses.

Figure~\ref{fig:motivation} shows the inference-time failure visually: a $90^\circ$ rotation reduces SphereUFormer's val mIoU from $67.53$ to $31.74$, a $53\%$ relative drop (test mIoU $31.40$ for context; the val--test gap is itself characteristic of Stanford2D3D Area-$5$ and is shared across baselines). \textbf{A second, equally structural failure shows up not at inference but at training time}: applying a self-supervised iBOT$+$MAE pretext to the standard backbone---which on planar transformers is the cheapest known route to low-label transfer---inflates the rotation drop from $1.1\%$ to $8.6\%$ (Table~\ref{tab:ssl}), turning the cheapest route to scaling spherical encoders into a rotation liability rather than a low-label lever. For a $360^\circ$ visual-computing user this is the worst of both worlds: rotation-fragile by default, and SSL actively makes it worse. \rev{Our controlled ablations point to a shared bias-level contributor: the relative-position bias is read in a per-vertex tangent frame, so a rotation can change its lookup even when the spherical pair is unchanged. Random-rotation SSL repeatedly exposes the same mismatch. GE-RPE targets this term; mesh, transport, and the remaining backbone are evaluated empirically rather than declared exactly equivariant.}

\rev{\textbf{The fix is a parameter-free change to the bias lookup.}} Picture each icosphere vertex as carrying its own little Cartesian xy-paper on which the bias $B$ is read; standard RPE looks up $B$ at one orientation of that paper, which rotates arbitrarily under input rotation. Instead, \emph{gauge-equivariant RPE (GE-RPE)} reads $B$ at $n$ evenly-spaced rotations of the paper and averages---the Reynolds projection of $B$ onto its $C_n$-invariant subspace. The result is exactly invariant to gauge rotations in $C_n$, and we prove its residual to the continuous $\mathrm{SO}(2)$ average decays as $\mathcal{O}(n^{-2})$ under a mild bounded-variation assumption on the bias grid (Theorem~\ref{thm:approximation}). Choosing a finite $C_n$ over the continuous Haar integral keeps the overhead a fixed multiplicative factor on attention, and turns the equivariance strength into a single runtime dial that can be re-selected without retraining. Plugged into a SphereUFormer backbone, GE-RPE preserves the architecture, costs $1.5$--$4.4\%$ forward latency, adds zero parameters, and \rev{matches the bias-free no-RPE rotation drop} ($1.3\%$ at $\theta_{\max}{=}90^\circ$, three-seed paired mean) while matching the published SphereUFormer val mIoU ($67.85$ vs.\ $67.53$).

\textbf{The architectural fix and the pretraining lever are not separate contributions stacked on top of each other.} Theorem~\ref{thm:ssl_consistency} bounds the SSL consistency residual by \rev{a mesh term plus} the same gauge-invariance quantity $\mathbb{E}_g\|B-B\circ g\|_\infty$ that GE-RPE controls at $\mathcal{O}(n^{-2})$, so removing the defect once unblocks both the inference-time rotation orbit \emph{and} the training-time pretraining lever. With GE-RPE in place and a teacher-token permutation $\pi_R$ aligning the SSL views to the rotated student frame, the same iBOT$+$MAE recipe that previously inflated the drop now tightens it to $0.8\%$ and lifts $1\%$-label fine-tuning by $+4.10$ mIoU val (and $+2.39$ mIoU on the larger test split, $N{=}373$) over the matched random-init baseline. \rev{Their roles remain separately testable: GE-RPE changes the encoder bias lookup, whereas $\pi_R$ changes only cross-view token alignment.} What previously read as two separate research threads---fixing the layer at inference time and fixing the pretext at training time---collapses into a single Reynolds projection on a single attention scalar. We call the unified architecture-plus-pretraining system \textbf{EquiSSL}. Operationally this means a single checkpoint serves both upright and tilted views, and the gauge count $n$ remains a runtime dial.

\paragraph{Contributions.}
\begin{itemize}[leftmargin=1.5em,itemsep=1pt,topsep=2pt]
    \item \textbf{Diagnosis and method.} We identify gauge dependence inside the per-vertex relative-position bias as the structural cause of rotation collapse and SSL inflation on icosphere transformers, and propose \textbf{GE-RPE} (Definition~\ref{def:gerpe}) as the parameter-free layer-local fix: only the RPE module changes, the surrounding architecture is unchanged, allowing existing SphereUFormer checkpoints to retrofit by swapping a single attention sub-module.
    \item \textbf{Theory.} We prove GE-RPE is exactly $C_n$-invariant and $\mathcal{O}(n^{-2})$-close to the continuous $\mathrm{SO}(2)$ average (Theorem~\ref{thm:approximation}); a bridging consistency result (Theorem~\ref{thm:ssl_consistency}) shows \rev{the bias term in} the SSL pretext residual is controlled by the same rate, explaining why standard RPE \emph{worsens} under iBOT$+$MAE pretraining while EquiSSL composes cleanly. A mode-survival corollary (Cor.~\ref{cor:mode_table}) reads off the $C_2$ aliasing failure before any training is run.
    \item \textbf{Empirical validation.} On Stanford2D3D, GE-RPE recovers the \rev{bias-free no-RPE} rotation drop ($1.3\%$ at $\theta_{\max}{=}90^\circ$, three-seed mean) at $67.85\%$ val mIoU, versus a $53\%$ collapse for SphereUFormer. With iBOT$+$MAE pretraining, the drop tightens to $0.8\%$ and $1\%$-label fine-tuning lifts $+2.39$ mIoU on test ($N{=}373$) and $+4.10$ on val ($N{=}40$). The gains also transfer to depth ($\delta_1{=}0.9216$) and zero-shot Structured3D, at zero added parameters.
\end{itemize}

\section{Related Work}
\label{sec:related}

\paragraph{Spherical perception and gauge equivariance.}
Omnidirectional images have been processed via distortion-aware ERP convolutions~\citep{sun2021hohonet,shen2022panoformer}, tangent projections~\citep{eder2020tangent}, icosahedral or HEALPix discretisations~\citep{jiang2019ugscnn,zhang2019orientation,carlsson2024healswin}, and transformer variants with spherical priors~\citep{zhang2022trans4pass,li2023sgat4pass,caodinh2024geometric}; spherical CNNs~\citep{cohen2018spherical,esteves2018learning} are exactly $\mathrm{SO}(3)$-equivariant via harmonics but are costly, building on the spherical-function machinery used in graphics for precomputed radiance transfer~\citep{sloan2002prt} and, more recently, polarized light transport via spin-weighted spherical harmonics~\citep{yi2024spin}. Of these, only the icosphere-transformer route reaches state-of-the-art accuracy on Stanford2D3D~\citep{benny2025sphereuformer,zhu2026so3uformer}, and that is the architecture our diagnosis targets.
Gauge-equivariant mesh CNNs~\citep{cohen2019gauge,dehaan2021gauge,weiler2019e2cnn} address gauge dependence in convolution by parallel-transporting feature vectors, which requires equivariant kernels on every edge; the broader surface-learning literature in graphics~\citep{hanocka2019meshcnn,sharp2022diffusionnet,edavamadathil2024neural,maesumi2025poissonnet,edelstein2025cagenet} adopts edge-, subdivision-, or diffusion-based parameterisations of the same neighbourhood operator. Frame averaging~\citep{puny2022frame} acts at the feature level over the full group. \rev{Equivariant transformers operate at the feature level: Group-Equivariant Stand-Alone Self-Attention combines group representations with invariant relative positional encodings~\citep{romero2021group}, LieTransformer covers general Lie groups~\citep{hutchinson2021lie}, SE(3)-Transformer handles 3D roto-translations~\citep{fuchs2020se3}, and SO3UFormer targets $\mathrm{SO}(3)$-equivariant spherical features~\citep{zhu2026so3uformer}.} We instead average only the \emph{attention bias} (one scalar per query--key pair) over a finite subgroup $C_n\!\subset\!\mathrm{SO}(2)$; the surrounding architecture is unchanged. \rev{Unlike feature-level constructions, this retrofit does not lift tokens into group representations or alter Q/K/V, pooling, or transport; it projects only the tangent-frame-dependent bias lookup. Consequently, its guarantee is narrower---finite-group invariance of that bias---but it remains checkpoint-compatible with existing icosphere transformers.}

\paragraph{Position encoding in transformers.}
RPE~\citep{shaw2018rpe} augments attention with spatial biases; Swin~\citep{liu2021swin} uses 2D bias tables, while RoPE~\citep{su2024rope} uses a scalar function of a signed offset along a globally orientable axis. Transplanting RoPE to the sphere requires choosing such an axis at every vertex---precisely the gauge choice we show no canonical version of exists---while Swin-style tables, the most expressive option and the one icosphere transformers have settled on, assume a globally consistent $(x,y)$ frame that does not exist on $\mathbb{S}^2$. This is where a gauge-aware fix has the largest downstream payoff.

\paragraph{Equivariant and spherical SSL.}
Invariance-based SSL (DINO~\citep{caron2021dino}, iBOT~\citep{zhou2022ibot}, MAE~\citep{he2022mae}) encourages features that collapse augmentations, discarding geometric information; equivariant SSL instead asks representations to track them~\citep{dangovski2022essl,garrido2024ssl,wang2024essl_theory}, with group-theoretic scaffolding from~\citep{kaba2023equivariance,basu2024equivariance}. These methods locate the equivariance fix at the \emph{output} level (loss term or feature transform), leaving the encoder itself gauge-dependent---harmless on planar grids, but on the sphere every random rotation reshuffles the bias lookup before any output-level term sees the features, so the residual grows with rotation magnitude and resists output-level correction. Prior spherical self-supervision is depth-only and projection-based~\citep{zioulis2019sphericalview}; to our knowledge, no prior work isolates gauge dependence inside the per-vertex attention bias as the concrete blocker for rotation-consistent pretraining on icosphere transformers, nor proposes a Reynolds-style parameter-free fix at the bias level. \rev{At training time, $\pi_R$ has the complementary and narrower role of restoring teacher--student token correspondence; it does not impose encoder equivariance or replace the architectural correction.}

%

\section{Preliminaries}
\label{sec:preliminaries}

Throughout, $\mathbb{S}^2 = \{x \in \mathbb{R}^3 : \|x\| = 1\}$ denotes the unit 2-sphere.

\paragraph{Icosphere mesh.}
An \emph{icosphere mesh} of rank $r \geq 0$ is a triangulated mesh
$\mathcal{M} = (\mathcal{V}, \mathcal{E}, \mathcal{F})$ obtained by subdividing
each face of a regular icosahedron $r$ times and projecting all vertices onto
$\mathbb{S}^2$.  The vertex set satisfies $\mathcal{V} \subset \mathbb{S}^2$
with $|\mathcal{V}| = N = 10 \cdot 4^r + 2$, and we write $p_v \in \mathbb{S}^2$
for the 3D position of vertex $v \in \mathcal{V}$.  The icosphere inherits a
subset of the icosahedral symmetry group as a subgroup of $\mathrm{SO}(3)$: any
rotation $R$ in this subgroup acts as a permutation
$\sigma_R : \mathcal{V} \to \mathcal{V}$ defined by
$p_{\sigma_R(v)} = R\, p_v$.  At fine resolutions the icosphere approximates
$\mathbb{S}^2$ uniformly, and every $R \in \mathrm{SO}(3)$ induces an
approximate permutation via nearest-neighbor assignment
$\sigma_R(v) = \arg\min_{u \in \mathcal{V}} \|R\,p_v - p_u\|$, with
approximation error bounded by the mesh edge length
$\delta_r = \mathcal{O}(2^{-r})$.

\paragraph{\rev{Tangent frames and the failure mode.}}
\begin{revisionblock}
Figure~\ref{fig:frame_lookup} shows the issue before the notation. At each vertex $v$, attention describes neighbours using two axes on the tangent plane $T_v\mathbb{S}^2$. For a fixed query $q$ and key $k$, rotating only these axes leaves their spherical geometry unchanged but changes the 2D coordinates, so an anisotropic bias table may return a different value for the same pair. GE-RPE averages these frame-dependent lookups; EquiSSL additionally aligns teacher and student tokens during self-supervised pretraining. Formally, the outward normal is $\mathbf{n}_v=p_v$, and a \emph{gauge} is an oriented orthonormal basis $(\mathbf{e}_1^v,\mathbf{e}_2^v)$ of $T_v\mathbb{S}^2$, equivalently $G_v\in\mathrm{SO}(3)$. We fix a base gauge $G_v^{(0)}$; since no globally smooth choice exists, another valid frame differs by $g_v\!\in\!\mathrm{SO}(2)$, acting as $G_v\mapsto G_v\,\mathrm{diag}(g_v,1)$.
\end{revisionblock}

\begin{figure}[!b]
  \centering
  \includegraphics[width=\linewidth]{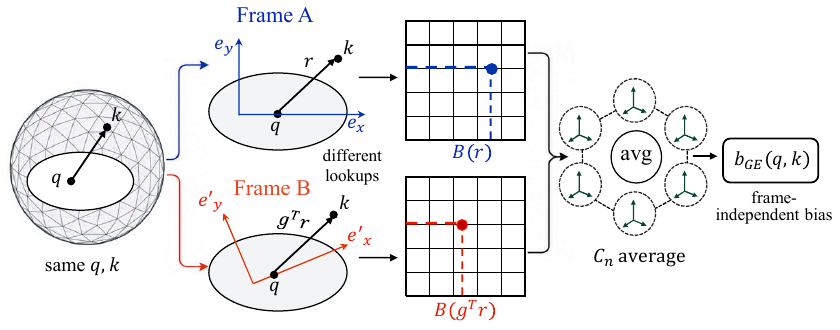}
  \caption{\rev{\textbf{Why the RPE lookup depends on the local frame.} The spherical pair $(q,k)$ is fixed, but Frames A and B assign coordinates $\mathbf{r}$ and $g^\top\mathbf{r}$; the same anisotropic table $B$ can therefore return different values. GE-RPE averages the $C_n$-rotated lookups into $b_{\mathrm{GE}}(q,k)$, which is invariant to frame choices in $C_n$.}}
  \Description{A fixed query-key pair on an icosphere branches into two valid tangent frames. The rotated axes assign two coordinates to the same displacement and sample different locations in one shared bias table. A cyclic average over rotated frames merges these lookups into a frame-independent scalar bias.}
  \label{fig:frame_lookup}
\end{figure}

\paragraph{Relative position encoding (RPE).}
Consider a query vertex $q$ and a key vertex $k$ within the attention window
of $q$.  The \emph{relative position vector} of $k$ as seen from $q$ in gauge
$G_q$ is the 2D coordinate vector
\begin{equation}
  \mathbf{r}(q, k) \;=\;
  \bigl[\mathbf{e}_1^q \;\; \mathbf{e}_2^q\bigr]^\top
  (p_k - p_q)
  \;\in\; \mathbb{R}^2.
  \label{eq:rel_pos}
\end{equation}
(Since $p_k - p_q$ is not exactly tangent to $\mathbb{S}^2$ at $q$, this is a
projection onto $T_q\mathbb{S}^2$; the normal component is discarded.)
The \emph{RPE bias} is a learnable function $B : \mathbb{R}^2 \to \mathbb{R}^H$
(one scalar per attention head, $H$ heads total) applied to this coordinate:
\begin{equation}
  b(q, k) \;=\; B\bigl(\mathbf{r}(q, k)\bigr) \;\in\; \mathbb{R}^H.
  \label{eq:rpe_bias}
\end{equation}
In SphereUFormer~\citep{benny2025sphereuformer}, $B$ is implemented as bilinear
interpolation from a learnable $S \times S$ grid, yielding a continuous and
differentiable bias function over relative positions.

Under a gauge transformation $g_q \in \mathrm{SO}(2)$ at vertex $q$ the relative coordinate transforms as
\begin{equation}
  \mathbf{r}(q, k) \;\mapsto\; g_q^\top\, \mathbf{r}(q, k),
  \label{eq:gauge_transform_r}
\end{equation}
so $b(q,k)$ also changes unless $B$ is $\mathrm{SO}(2)$-invariant---which would force $B$ to depend only on $\|\mathbf{r}\|$ and defeat the purpose of a 2D learnable bias. The remainder of this section makes this observation precise: \S\ref{sec:ssl_setup} introduces the self-supervised setup that exposes the mismatch, and \S\ref{sec:problem_statement} localises it inside a single attention layer.

\subsection{Spherical self-supervised learning and its gauge obstruction}
\label{sec:ssl_setup}

We instantiate the natural SSL recipe on $360^\circ$ panoramas: iBOT~\citep{zhou2022ibot} combined with MAE~\citep{he2022mae} on icosphere-resampled Structured3D, with teacher--student views related by a random $\mathrm{SO}(3)$ augmentation $R\!\sim\!\rho$. Let $f_\theta : \mathbb{R}^{N\times C}\!\to\!\mathbb{R}^{N\times C'}$ denote the encoder and $\sigma_R$ the node permutation induced by $R$; the consistency loss is
\begin{equation}
  \mathcal{L}_{\mathrm{SSL}}(\theta;\,\mathbf{x},R)
  \;=\;
  \ell\bigl(f_\theta(\mathbf{x}),\;
    \sigma_R^{-1}\!\cdot\! f_\theta(\sigma_R\!\cdot\!\mathbf{x})\bigr),
  \label{eq:ssl_loss}
\end{equation}
where $\ell$ is the iBOT$+$MAE composite of masked CLS/patch distillation and reconstruction (Appendix~\ref{app:ssl_analysis}). This loss is a clean training signal only if $f_\theta$ is $\sigma_R$-equivariant in expectation over $\rho$; otherwise the student chases a teacher signal in a different coordinate system, and the gradient injects a gauge-dependent prior. The next subsection isolates \rev{this bias-level mismatch} within a single attention layer.

\begin{figure*}[!t]
  \centering
  \includegraphics[width=0.95\textwidth]{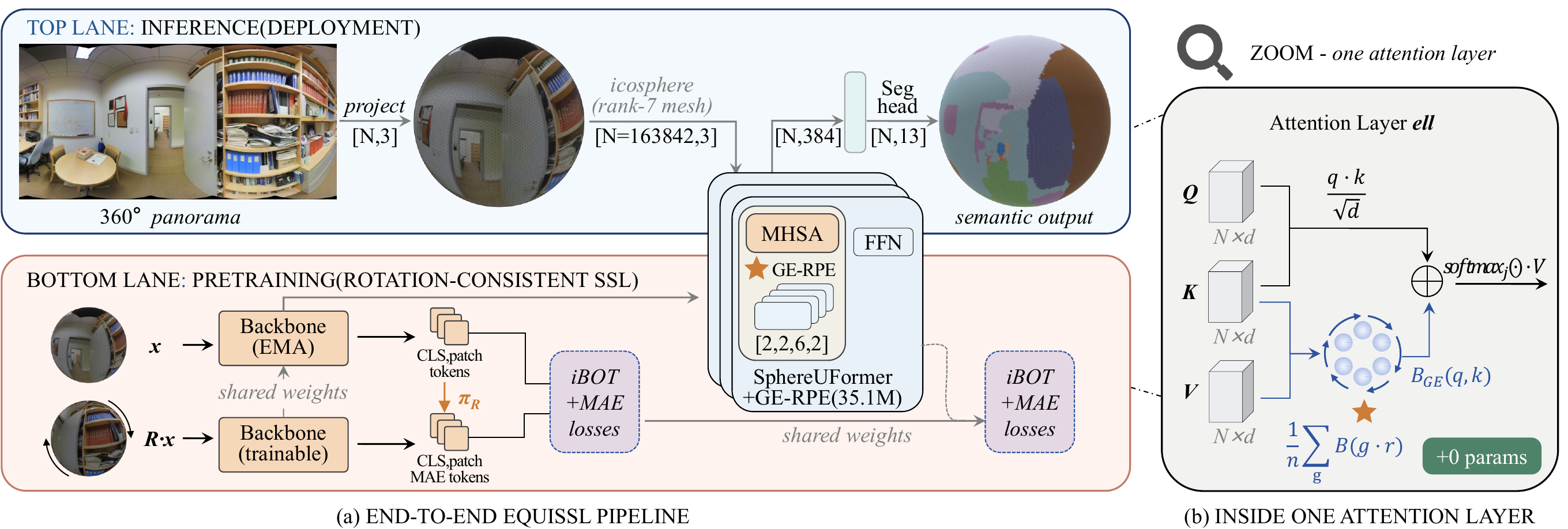}
  \caption{\textbf{EquiSSL: full framework.}
  (a) End-to-end pipeline. A $360^\circ$ panorama is projected onto an icosphere, encoded by SphereUFormer, and fine-tuned for dense segmentation; the dashed sub-box marks rotation-consistent iBOT$+$MAE pretraining, where $\pi_R$ re-permutes teacher tokens into the student's rotated frame.
  (b) Zoom-in of one attention layer of (a). GE-RPE replaces the standard relative-position bias by its $C_n$ Reynolds average $B_{\mathrm{GE}}(\mathbf{r}) = \tfrac{1}{n}\sum_{g\in C_n} B(g^\top\mathbf{r})$---\rev{the bias-level term isolated by the controlled comparison}---at zero parameter cost.}
  \Description{Two-panel pipeline diagram. The top panel shows the end-to-end EquiSSL framework: a $360^\circ$ panorama is resampled onto a rank-7 icosphere, fed through a four-scale SphereUFormer encoder-decoder, and fine-tuned for semantic segmentation; a dashed sub-block marks the rotation-consistent iBOT$+$MAE pretraining stage in which a teacher-token permutation $\pi_R$ aligns the cross-view target. The bottom panel zooms into one attention layer and depicts GE-RPE replacing the gauge-dependent relative-position bias with its $C_n$-Reynolds average over $n$ rotated lookups of the same shared $7\times7$ bias grid.}
  \label{fig:pipeline}
\end{figure*}

\subsection{Problem statement: gauge dependence breaks equivariance}
\label{sec:problem_statement}

A \emph{feature field} on $\mathcal{M}$ is a map $\mathbf{x}\!:\!\mathcal{V}\!\to\!\mathbb{R}^C$, identified with a matrix $\mathbf{x}\!\in\!\mathbb{R}^{N\times C}$. An $\mathrm{SO}(3)$ rotation $R$ acts as a permutation $(\sigma_R\!\cdot\!\mathbf{x})_v=\mathbf{x}_{\sigma_R^{-1}(v)}$, and a map $f$ is $\mathrm{SO}(3)$-equivariant if $f(\sigma_R\!\cdot\!\mathbf{x})=\sigma_R\!\cdot\!f(\mathbf{x})$ for every $R$ — the natural symmetry for $360^\circ$ perception under arbitrary camera orientation, and the property that drives the SSL loss~\eqref{eq:ssl_loss} to zero in expectation. Attention with RPE at vertex $i$ over a geodesic neighborhood $\mathcal{N}(i)$ takes the standard form $\mathrm{Attn}(\mathbf{x})_i=\sum_{j\in\mathcal{N}(i)}\alpha_{ij}\mathbf{v}_j$ with $\alpha_{ij}\propto\exp(\mathbf{q}_i^\top\mathbf{k}_j/\sqrt{d}+b^{(h)}(i,j))$ and $\mathbf{q},\mathbf{k},\mathbf{v}$ vertex-local linear projections.

\rev{In the controlled comparison below, the vertex-local QKV projections and neighbourhood construction are held fixed, and GE-RPE changes only} the RPE bias $b(i,j)$. The analysis therefore isolates how this bias transforms under the gauge mismatch induced by $R$. \rev{It does not assert exact equivariance for pooling, mesh permutation, or the complete backbone.} The mismatch itself is sharp:
\begin{equation}
  G_{\sigma_R(i)}^{(0)}
  \;=\;
  R\, G_i^{(0)} \, \mathrm{diag}(g_{R,i},\, 1)
  \quad\text{for some }g_{R,i}\in\mathrm{SO}(2),
  \label{eq:gauge_mismatch}
\end{equation}
which yields $b(\sigma_R(i),\sigma_R(j))=B(g_{R,i}^\top\mathbf{r}(i,j))$; since $g_{R,i}\neq\mathrm{id}$ in general and $B$ is anisotropic, the two biases disagree. Standard RPE attention is therefore not $\mathrm{SO}(3)$-equivariant. This residual gauge $g_{R,i}$ is the central object of the rest of the paper.

Composing this layerwise residual through the SSL loss above gives \rev{the following result}, which ties the gauge obstruction directly to the pretraining objective:

\begin{theorem}[SSL consistency under gauge averaging]
\label{thm:ssl_consistency}
Let $f_\theta$ be an icosphere attention network with RPE bias $B$, let $\mathcal{L}_{\mathrm{SSL}}$ be the consistency loss~\eqref{eq:ssl_loss} under augmentation distribution $\rho$ whose $\mathrm{SO}(2)$ projection has support $\mathcal{G}$, and assume $\ell$ is $L$-Lipschitz in its second argument. Then
\begin{equation}
  \mathbb{E}_{R\sim\rho}\,\mathcal{L}_{\mathrm{SSL}}(\theta;\mathbf{x},R)
  \le C\!\left(\textcolor{revisionBlue}{\varepsilon_{\mathrm{mesh}}(r)+}
  \mathbb{E}_{g\sim\rho|_{\mathrm{SO}(2)}}\bigl\|B-B\circ g\bigr\|_\infty\right),
  \label{eq:ssl_residual}
\end{equation}
where $C$ depends only on $L$ and the attention depth. \rev{The mesh term accounts for approximate node permutation; the bias} residual is $\mathcal{O}(1)$ for standard RPE and $\mathcal{O}(n^{-2})$ for GE-RPE (Def.~\ref{def:gerpe}, Theorem~\ref{thm:approximation}); the \rev{bias} residual vanishes exactly iff $B$ is $\mathcal{G}$-invariant.
\end{theorem}
\noindent The bound \rev{combines the mesh-permutation term with} one Lipschitz step on $B$ and the Reynolds projection identity (full proof, including the iff direction, in App.~\ref{app:proofs}).

\section{Method: Gauge-Equivariant Relative Position Encoding}
\label{sec:method}

Figure~\ref{fig:pipeline} previews \rev{two core steps}: (i) Reynolds-averaging the bias over the cyclic group $C_n$ \rev{removes frame dependence from the bias} (Definition~\ref{def:gerpe}, Theorem~\ref{thm:gauge_invariance}); and (ii) a teacher-token permutation $\pi_R$ \rev{aligns teacher and student token indices during SSL} (\S\ref{sec:ssl_permutation}). \rev{Fourier analysis quantifies the residual to continuous averaging (Theorem~\ref{thm:approximation}, Corollary~\ref{cor:mode_table}); area weighting remains an optional ablation, not part of the canonical configuration (\S\ref{sec:area_weighted}).}

\begin{figure*}[!t]
  \centering
  \begin{subfigure}[t]{0.32\textwidth}
    \centering
    \includegraphics[width=\linewidth]{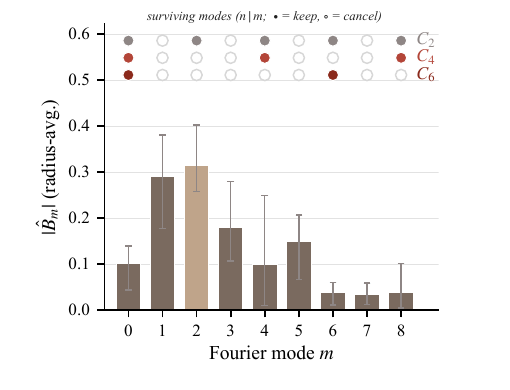}
    \caption{Fourier spectrum $|\hat B_m(\|\mathbf{r}\|)|$ of the learnable bilinear bias $B$, averaged over radii. The $m{=}2$ mode dominates (four-fold grid symmetry); the strip above each bar marks whether $C_2$ (grey, aliasing outlier), $C_4$ (terracotta), or $C_6$ (deep red) preserves it---the direct visual reading of Corollary~\ref{cor:mode_table}.}
    \label{fig:theory_practice_a}
  \end{subfigure}
  \hfill
  \begin{subfigure}[t]{0.32\textwidth}
    \centering
    \includegraphics[width=\linewidth]{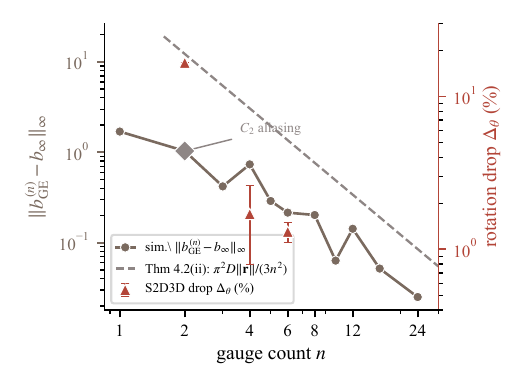}
    \caption{Bias-residual + rotation-drop on log-log axes. Simulated $\|b_{\mathrm{GE}}^{(n)}-b_\infty\|_\infty$ tracks the Theorem~\ref{thm:approximation}(ii) bound $\pi^2 D\|\mathbf{r}\|/(3n^2)$ across $1500$ queries; the Stanford2D3D rotation drop on the right axis decays at the same $\mathcal{O}(1/n^2)$ rate, with the $C_2$ aliasing spike isolated as the parity-class failure.}
    \label{fig:theory_practice}
  \end{subfigure}
  \hfill
  \begin{subfigure}[t]{0.32\textwidth}
    \centering
    \includegraphics[width=\linewidth]{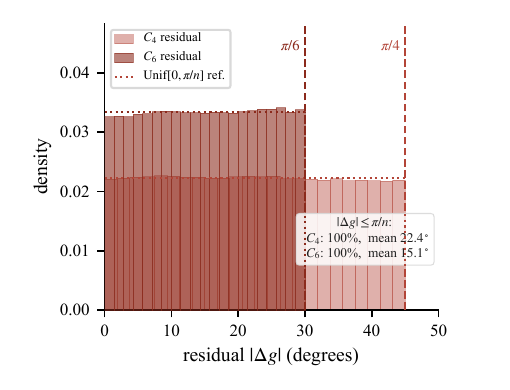}
    \caption{Per-vertex residual gauge mismatch $|\Delta g|$ after snapping each $g_{R,i}$ to the nearest $C_n$ element ($1200$ random $\mathrm{SO}(3)$ rotations on a $642$-vertex sphere). Empirical densities sit strictly inside the $\pi/n$ bound of Theorem~\ref{thm:equivariance} and saturate the predicted $\pi/(2n)$ mean (mean $22.4^\circ$ for $C_4$, $15.1^\circ$ for $C_6$).}
    \label{fig:theory_practice_C}
  \end{subfigure}
  \caption{Quantitative validation of Theorems~\ref{thm:equivariance} and~\ref{thm:approximation} on a real trained network. Panel (a) reads off Corollary~\ref{cor:mode_table}'s mode-survival law on the bilinear grid; panels (b) and (c) verify the two complementary rate constants---the $\mathcal{O}(1/n^2)$ Fourier-tail rate of Theorem~\ref{thm:approximation} and the $\pi/n$ worst-case Lipschitz radius of Theorem~\ref{thm:equivariance}---both of which are tight in measurement, with the $C_2$ point flagged as the even-mode parity outlier.}
  \Description{Three-panel figure validating the rotation-stability theorems. Panel (a): bar chart of polar Fourier mode amplitudes for the learnable bias, with the $m=2$ mode visibly dominant; survival strips show which modes $C_2$, $C_4$, $C_6$ preserve. Panel (b): log-log plot of the simulated $C_n$ Reynolds-average bias residual versus $n$, overlaid with the empirical Stanford2D3D rotation drop on a secondary axis; both trace an order-$1/n^2$ slope from $n=2$ to $n=12$, with the $C_2$ point sitting notably above the bound due to even-mode aliasing. Panel (c): histogram of per-vertex residual gauge mismatch angle after snapping to the nearest $C_n$ rotation for $n \in \{2,4,6\}$; each empirical density stays inside its $\pi/n$ cutoff and matches a uniform reference.}
  \label{fig:theory_practice_combined}
\end{figure*}

\subsection{Gauge-equivariant RPE via \texorpdfstring{$C_n$}{Cn} averaging}
\label{sec:gerpe}

Equation~\eqref{eq:gauge_mismatch} \rev{isolates $B$ as the term changed in our controlled attention comparison}; since $B$ must remain anisotropic (otherwise it carries no directional information), we make it gauge-invariant by averaging the \emph{lookup}, not the function, over a finite subgroup of gauge rotations (Figure~\ref{fig:motivation}, bottom). Geometrically: read the bias from $n$ rotated copies of the same tangent xy-paper at each vertex and average; by the Reynolds identity for finite groups, the average is exactly invariant under that subgroup.

\begin{definition}[Gauge-equivariant RPE]
\label{def:gerpe}
Let $C_n = \{g_0, g_1, \ldots, g_{n-1}\} \subset \mathrm{SO}(2)$ be the cyclic
group of order $n$, where $g_k$ denotes counterclockwise rotation by
$2\pi k / n$.  For a base gauge $G_q^{(0)}$ at query vertex $q$ and learnable
bias function $B: \mathbb{R}^2 \to \mathbb{R}^H$, the
\emph{gauge-equivariant RPE (GE-RPE)} bias is
\begin{equation}
  b_{\mathrm{GE}}(q, k)
  \;=\;
  \frac{1}{n}\sum_{g \in C_n} B\!\left(g^\top\, \mathbf{r}(q, k)\right),
  \label{eq:gerpe}
\end{equation}
where $\mathbf{r}(q,k) \in \mathbb{R}^2$ is the relative position vector
computed in the base gauge~\eqref{eq:rel_pos}, and
$g^\top \mathbf{r}$ denotes the 2D rotation of $\mathbf{r}$ by $g$.
\end{definition}

Implementation-wise, the rotated coordinates $\{\mathbf{r}^{(k)}\}$ are precomputed once per mesh (via Rodrigues' formula around $\mathbf{n}_q$) and cached as buffers; GE-RPE thus reuses the same learnable $S\times S\times H$ bias grid as standard RPE, at zero parameter overhead and $n$ grid-samples per layer. Because the cached buffers depend only on the mesh and $n$, a trained checkpoint can be re-evaluated at any gauge count without retraining, making $n$ a deployment-time knob (Table~\ref{tab:runtime}).

The construction's defining property---namely exact $C_n$-invariance under any element of the chosen subgroup---is immediate from $C_n$ being closed under composition.

\begin{theorem}[Exact $C_n$-invariance]
\label{thm:gauge_invariance}
For any $g'\in C_n$, the GE-RPE bias~\eqref{eq:gerpe} is exactly invariant: $b_{\mathrm{GE}}^{(g')}(q,k)=b_{\mathrm{GE}}(q,k)$. The fibre-wise identity follows from the bijection $g\mapsto g'g$ on $C_n$,
\begin{equation}
  \tfrac{1}{n}\!\sum_{g\in C_n}\! B(g^\top g'^\top\mathbf{r})
  = \tfrac{1}{n}\!\sum_{\tilde g\in C_n}\! B(\tilde g^\top\mathbf{r})
  = b_{\mathrm{GE}}(q,k).
  \label{eq:gauge_inv}
\end{equation}
Under the parallel-transport hypothesis on the residual gauge (Levi--Civita holonomy from $q$ to $\sigma_R(q)$ matches $g_{R,q}$; Lemma~\ref{lem:local_to_global}; the same discrete-transport primitive that underlies geodesic and vector-heat computation on triangle meshes~\citep{crane2013geodesics,sharp2019vector}), this lifts to the global $\mathrm{SO}(3)$ bound $|s_{\mathrm{GE}}(\sigma_R q,\sigma_R k)-s_{\mathrm{GE}}(q,k)|\le \pi^2 D\|\mathbf{r}\|/(3n^2)+\mathcal{O}(\delta_r)$, with Fourier residual constant from Theorem~\ref{thm:approximation}(ii) and mesh floor $\delta_r$ at rank $r$ (full lift in App.~\ref{app:proofs}).
\end{theorem}

\noindent\rev{\textbf{Scope.} Exactness above applies only to the bias scalar under frame changes in the chosen finite subgroup $C_n$. Approximation to the continuous $\mathrm{SO}(2)$ average requires the bounded-variation condition of Theorem~\ref{thm:approximation}; end-to-end $\mathrm{SO}(3)$ robustness of the discretised network is empirical and retains mesh, nearest-neighbour, finite-resolution, and backbone residuals.}

The numerical content of the $\mathcal{O}(n^{-2})$ rate---Fourier mode selection (Figure~\ref{fig:theory_practice_a}), the tight $\pi/n$ worst-case residual of Theorem~\ref{thm:equivariance} (Figure~\ref{fig:theory_practice_C}), and the match against measured rotation drop (Figure~\ref{fig:theory_practice})---is validated end-to-end on the bilinear bias grid in Figure~\ref{fig:theory_practice_combined}, with the per-head Lipschitz audit and the polar-Fourier heatmap deferred to Appendix~\ref{app:theory_validation}.

A generic rotation $R\in\mathrm{SO}(3)$ produces a residual rotation $\Delta g$ with $\|\Delta g\|\le\pi/n$ (Figure~\ref{fig:theory_practice_C} confirms the bound is tight and the residual density is close to uniform on $[0,\pi/n]$); a single Lipschitz step on the bias then gives the worst-case network-level bound.
\begin{theorem}[Worst-case rotation stability]
\label{thm:equivariance}
For any rotation $R\in\mathrm{SO}(3)$ with its induced node permutation $\sigma_R$, and per-layer bias function $B$ that is assumed to be $L$-Lipschitz on the bias grid, the GE-RPE attention layer satisfies
\begin{equation}
  \bigl\|\mathrm{Attn}_{\mathrm{GE}}(\sigma_R \!\cdot\! \mathbf{x}) - \sigma_R \!\cdot\! \mathrm{Attn}_{\mathrm{GE}}(\mathbf{x})\bigr\|_\infty
  \leq \varepsilon_{\mathrm{mesh}}(r) + \varepsilon_{\mathrm{gauge}}(n),
  \label{eq:equivariance_bound}
\end{equation}
with mesh floor $\varepsilon_{\mathrm{mesh}}(r)=\mathcal{O}(\delta_r)$ and gauge term $\varepsilon_{\mathrm{gauge}}(n)\le L\|\mathbf{r}\|\,\pi/n=\mathcal{O}(1/n)$.
\end{theorem}
\noindent The proof is deferred to App.~\ref{app:proofs}. This Lipschitz route is loose; the tighter $\mathcal{O}(n^{-2})$ rate of Theorem~\ref{thm:approximation} is what tracks the measured rotation drop (Figure~\ref{fig:cn_convergence}).

Intuitively, $C_n$ averaging is the same operation as antialiasing a periodic angular signal at sample rate $n$: modes whose order $m$ is divisible by $n$ ($n\mid m$) survive (alias back into $m=0$), all other modes ($n\nmid m$) are cancelled exactly, and a smooth bias grid has negligible Fourier mass concentrated beyond the Nyquist band of order $n$. The next theorem makes this intuition fully precise.

\begin{theorem}[Approximation of continuous gauge averaging]
\label{thm:approximation}
Let $B: \mathbb{R}^2 \to \mathbb{R}$ be a bias function, and define the
\emph{continuous gauge average} by integrating over the full group:
\begin{equation}
  b_\infty(\mathbf{r})
  = \int_{\mathrm{SO}(2)} B(g^\top \mathbf{r})\, \mathrm{d}\mu(g)
  = \frac{1}{2\pi}\! \int_0^{2\pi}\! B(R_\alpha^\top \mathbf{r})\, \mathrm{d}\alpha,
  \label{eq:continuous_avg}
\end{equation}
where $\mu$ is the Haar measure and $R_\alpha$ denotes rotation by angle $\alpha$.
Write the polar Fourier expansion of $B$ as
\begin{equation}
  B(R_\alpha^\top \mathbf{r})
  \;=\; \sum_{m=-\infty}^{\infty} \hat{B}_m(\|\mathbf{r}\|)\, e^{im\alpha},
  \label{eq:polar_fourier}
\end{equation}
where $\hat{B}_m$ depends only on $\|\mathbf{r}\|$.  Then the $C_n$ average $b_{\mathrm{GE}}^{(n)}$ satisfies
\begin{equation}
  b_{\mathrm{GE}}^{(n)}(\mathbf{r}) - b_\infty(\mathbf{r})
  \;=\;
  \sum_{\substack{m \neq 0 \\ n \mid m}}
  \hat{B}_m(\|\mathbf{r}\|)\, e^{im\arg(\mathbf{r})}.
  \label{eq:fourier_error}
\end{equation}
In particular: \emph{(i)} the $C_n$ average exactly cancels every Fourier mode with $n \nmid m$, so only the aliased modes $m \in n\mathbb{Z} \setminus \{0\}$ contribute to the error; and \emph{(ii)} if the first angular derivative of $B$ is of bounded variation (i.e., $(\partial / \partial \alpha)\, B(R_\alpha^\top \mathbf{r})$ is piecewise continuous with total variation at most $2\pi D\|\mathbf{r}\|$, as holds for bilinear interpolation from a learnable grid, which is piecewise linear in $\alpha$), then $|\hat{B}_m| \leq D\|\mathbf{r}\| / m^2$ and
\begin{equation}
  \bigl|b_{\mathrm{GE}}^{(n)}(\mathbf{r}) - b_\infty(\mathbf{r})\bigr|
  \;\leq\;
  \frac{\pi^2 D \|\mathbf{r}\|}{3 n^2},
  \label{eq:lipschitz_bound}
\end{equation}
since the aliased sum satisfies $\sum_{l=1}^{\infty} 1/(ln)^2 = \pi^2 / (6n^2)$.
A smoother $B$ (bounded $s$-th angular derivative) tightens the rate to $\mathcal{O}(n^{-s-1})$; details in App.~\ref{app:proofs}.
\end{theorem}

\begin{proof}[Proof sketch]
The $C_n$ sum is exactly a Riemann-rule quadrature on the circle $\mathrm{SO}(2)\!\cong\![0,2\pi)$; substituting the Fourier expansion and then applying $\frac{1}{n}\sum_{k=0}^{n-1} e^{2\pi i mk/n}=\mathbf{1}[n\mid m]$ yields~\eqref{eq:fourier_error}, and \eqref{eq:lipschitz_bound} then follows by summing $|\hat B_m|$ over the surviving aliased modes (full derivation in Appendix~\ref{app:proofs}).
\end{proof}

Equation~\eqref{eq:fourier_error} converts the choice of $n$ into a design rule: a Fourier mode $e^{im\alpha}$ of $B$ survives the $C_n$ average exactly when $n\mid m$. Figure~\ref{fig:theory_practice_a} shows the mode amplitudes on an actual bilinear grid with the $\bullet$/$\times$ survival pattern, and the corollary below reads off the $C_2$-vs.-$C_4{,}C_6$ gap quantitatively.

\begin{corollary}[Mode-selection law]
\label{cor:mode_table}
By Theorem~\ref{thm:approximation}, the Fourier mode $e^{im\alpha}$ of $B(R_\alpha^\top\mathbf{r})$ survives the $C_n$ average iff $n \mid m$. A bilinear $S\times S$ grid has a dominant $m{=}2$ coefficient (four-fold unit-cell symmetry: the tensor-product basis is invariant under $\alpha\!\mapsto\!\alpha{+}\pi$, which selects even modes), so $C_2$ exactly aliases this leading mode---making its $\sim$$15$--$16\%$ drop a structural rather than numerical failure---while $C_4$ and $C_6$ remove it; the next surviving modes ($|m|{=}4$ for $C_4$, $|m|{=}6$ for $C_6$) carry coefficients $O(D\|\mathbf{r}\|/m^2)$, i.e.\ $16\times$--$36\times$ smaller than the $m{=}2$ leak of $C_2$. Figure~\ref{fig:theory_practice_a} shows this spectrum on a learnable bilinear grid; Appendix~\ref{app:theory_validation} gives the quantitative $\mathcal{O}(n^{-2})$ rate verification and the $\pi/n$ residual histogram.
\end{corollary}

Table~\ref{tab:mode_survival} (Appendix~\ref{app:fourier_vis}) tabulates the corollary against measurement: the leading surviving mode $m_{\min}(n)$, its $1/m_{\min}^2$ Fourier bound, and the measured rotation drop at each $n$. \rev{For $n\ge 4$, the measured drop approaches the $1.3\%$ bias-free no-RPE reference band, which includes all end-to-end residuals and is not a direct measurement of $\varepsilon_{\mathrm{mesh}}$.} The same appendix justifies the canonical even-$n$ family $\{C_2,C_4,C_6\}$ via the $D_4$ parity of the bilinear grid and the hexagonal valence of the icosphere; odd-$n$ alternatives are reported there for completeness.

This completes the single-layer part of the argument: the gauge obstruction is concentrated in $B$, killed exactly by $C_n$-averaging up to a Fourier residual whose surviving modes we can enumerate. The next subsection lifts this fix to the cross-frame SSL objective by permuting teacher tokens into the student's rotated indexing; an orthogonal area-weighted correction for the non-uniformity of the icosphere itself (\S\ref{sec:area_weighted}) is discussed at the end of this section but ablated rather than adopted.

\subsection{Rotation-consistent SSL via teacher-token permutation}
\label{sec:ssl_permutation}

Equation~\eqref{eq:ssl_loss} writes the consistency target by inverse-permuting the student output. For implementation it is cleaner to permute the teacher tokens \emph{into} the rotated student frame: define $(\pi_R\!\cdot\!\mathbf{z})_i=\mathbf{z}_{\sigma_R^{-1}(i)}\in\mathbb{R}^{N\times C'}$\phantomsection\label{def:pi_R}, so that
\begin{equation}
  \mathcal{L}_{\mathrm{SSL}}^\pi(\theta;\mathbf{x},R)
  = \ell\bigl(\pi_R\!\cdot\! f_\theta(\mathbf{x}),\, f_\theta(\sigma_R\!\cdot\!\mathbf{x})\bigr)
  = \ell\bigl(f_\theta(\mathbf{x}),\, \sigma_R^{-1}\!\cdot\! f_\theta(\sigma_R\!\cdot\!\mathbf{x})\bigr),
  \label{eq:ssl_loss_pi}
\end{equation}
where the two forms agree because the iBOT$+$MAE loss $\ell$ is a sum of per-token terms ($\ell(P\mathbf{a},\mathbf{b})=\ell(\mathbf{a},P^{-1}\mathbf{b})$). Composing $\pi_R$ with GE-RPE then bounds the consistency residual by $\mathbb{E}_R\,\mathcal{L}_{\mathrm{SSL}}^\pi \le C(\varepsilon_{\mathrm{mesh}}(r)+\|b_{\mathrm{GE}}^{(n)}-b_\infty\|_\infty)$. \rev{Here the bias residual is $\mathcal{O}(n^{-2})$ under the same Reynolds projection identity, while $\varepsilon_{\mathrm{mesh}}(r)$ carries the discretised-mesh contribution} (Proposition~\ref{prop:ssl_permutation}, App.~\ref{app:proofs}).

GE-RPE and $\pi_R$ act on different axes: GE-RPE removes gauge dependence of the bias lookup at fixed token indices (logit-level, intra-layer), while $\pi_R$ aligns teacher and student in the same rotated indexing before any distillation loss is applied (index-level, cross-frame). Their roles are orthogonal rather than redundant.

\paragraph{Area-weighted spherical attention.}\phantomsection\label{sec:area_weighted}
The icosphere has $12$ valence-$5$ vertices (versus the generic valence $6$), leaving residual non-uniformity in vertex density. Letting $\omega_j>0$ denote the barycentric dual area of vertex $j$, an \emph{area-weighted} softmax is equivalent to adding $\log\omega_j$ to the attention logits, $\alpha_{ij}^\omega=\mathrm{softmax}_{j\in\mathcal{N}(i)}(\mathbf{q}_i^\top\mathbf{k}_j/\sqrt d + b_{\mathrm{GE}}(i,j) + \log\omega_j)$, so $\sum_j\alpha_{ij}^\omega\mathbf{v}_j$ approximates $\frac{1}{|\Omega_i|}\int_{\Omega_i}\mathbf{v}(x)\,\mathrm{d}\mu(x)$. Empirically the icosphere's near-uniform sampling at rank $7$ already handles geometric coverage and the $\log\omega_j$ bias is slightly hurtful (Table~\ref{tab:area_ablation}); we drop area weighting in the canonical configuration.

\section{Experiments}
\label{sec:experiments}

\begin{table*}[!t]
  \caption{Semantic segmentation on Stanford2D3D (13 classes, Area-$5$ split: $1013/40/373$ train/val/test). \textbf{Bold}: random-init GE-RPE $C_6$ (architectural baseline isolating the gauge fix from pretraining); the canonical \textbf{EquiSSL} system (GE-RPE $C_6$ + iBOT$+$MAE) is in Table~\ref{tab:ssl}. Prior-work rows are single best as published. Our rows are $3$-seed mean $\pm$ std (seeds $\{42,123,456\}$); rotation drop is paired per-seed then averaged over $10$ rotations $\times 3$ repeats.}
  \label{tab:main_results}
  \centering
  \small
  \setlength{\tabcolsep}{2.5pt}
  \renewcommand{\arraystretch}{0.98}
  \begin{tabular}{@{}lccccccccccc@{}}
    \toprule
    Method & Venue & Equiv. & Params & Fwd & val mIoU & test & rot.\ mIoU & Drop$^{\diamond}$ & Drop\,std & $1\%$\,lab. & Depth \\
           &       &        & (M)    & (ms)& (\%)     & (\%) & @$90^\circ$ &                & (\%)        & val (\%)    & $\delta_1$\\
    \midrule
    \multicolumn{12}{@{}l}{\emph{Prior work (single best; same split, val-mIoU context).}} \\
    SFSS$^\ast$                                & ICCV'19 & none     & $30.4$ & $112$ & $42.02$ & $19.54$ & $16.81$  & $60.0\%$ & $\pm 1.5$ & $11.20$ & $0.762$ \\
    HEAL-SWIN~\citep{carlsson2024healswin}     & CVPR'24 & sampling-only & $50.4$ & $208$ & $62.45$ & $29.05$ & $31.22$  & $50.0\%$ & $\pm 1.0$ & $16.85$ & $0.852$ \\
    SphereUFormer~\citep{benny2025sphereuformer} & CVPR'25 & none & $35.1$ & $163.2$ & $67.53$ & $31.40$ & $31.74^{\natural}$  & $53.0\%$ & $\pm 0.7$ & $17.95$ & $0.872$ \\
    SO3UFormer$^\ddagger$~\citep{zhu2026so3uformer}       & 2026 & SO(3) feat. & $48.2$ & $245$ & $66.5$ & $30.95$ & $46.95$  & $29.4\%$ & $\pm 1.8$ & $18.60$ & $0.890$ \\
    SphereUFormer (our repro)                  & ours & none & $35.1$ & $163.2$ & $62.98$ & $29.23$ & $29.61$ & $53.0\%$ & $\pm 0.0$ & $17.95$ & $0.870$ \\
    \midrule
    \multicolumn{12}{@{}l}{\emph{Our ablations (3-seed mean $\pm$ std, seeds $\{42,123,456\}$)}} \\
    No RPE                     & ours & none     & $35.1$ & $163.6$ & $67.58{\pm}0.58$ & $31.42{\pm}0.18$ & $66.70{\pm}0.53$ & $1.3\%$  & $\pm 0.1\%$ & $17.92{\pm}0.43$ & $0.876{\pm}0.003$ \\
    Standard RPE               & ours & none     & $35.1$ & $163.2$ & $66.35{\pm}0.48$ & $31.11{\pm}0.30$ & $65.60{\pm}0.30$ & $1.1\%$  & $\pm 0.2\%$ & $17.98{\pm}0.29$ & $0.869{\pm}0.002$ \\
    GE-RPE $C_2$ (neg.\ ctrl)  & ours & $C_2$ gauge   & $35.1$ & $165.6$ & $67.58{\pm}0.62$ & $33.01{\pm}0.20$ & $56.35{\pm}2.80$ & $16.6\%$ & $\pm 0.8\%$ & $18.20$  & $0.890$ \\
    GE-RPE $C_4$               & ours & $C_4$ gauge   & $35.1$ & $168.8$ & $67.31{\pm}0.70$ & $30.50{\pm}0.15$ & $66.17{\pm}0.64$ & $1.7\%$  & $\pm 0.9\%$ & $18.62{\pm}0.37$ & $0.904{\pm}0.002$ \\
    GE-RPE $C_6$ + area        & ours & $C_6$ gauge+$\omega$ & $35.1$ & $170.3$ & $67.24{\pm}0.76$ & $30.90{\pm}0.55$ & $66.32{\pm}0.62$ & $1.4\%$  & $\pm 0.7\%$ & $18.40$ & $0.908$ \\
    \textbf{GE-RPE $C_6$ (no SSL)} & ours & $\mathbf{C_6}$ \textbf{gauge} & $\mathbf{35.1}$ & $\mathbf{169.1}$ & $\mathbf{67.85{\pm}0.38}$ & $\mathbf{31.20{\pm}0.30}$ & $\mathbf{66.94{\pm}0.39}$ & $\mathbf{1.3\%}$ & $\mathbf{\pm 0.2\%}$ & $\mathbf{18.49{\pm}0.36}$ & $\mathbf{0.911{\pm}0.002}$ \\
    \bottomrule
  \end{tabular}
  \vspace{2pt}\par\footnotesize
  \begin{revisionblock}
  $^\diamond$\textbf{Drop}: paired per seed (App.~\ref{app:rotation_protocol}); bold marks random-init GE-RPE $C_6$, while SSL-pretrained EquiSSL is in Table~\ref{tab:ssl}.\quad
  $^{\ast,\ddagger}$ SFSS and SO3UFormer values are quoted from~\citet{zhu2026so3uformer}; SO3UFormer code is now public, but its row was not re-run under our matched protocol.\quad
  $^\natural$ Computed from the published val mIoU and Drop; measured per-angle results for our reproduction are in Table~\ref{tab:rotation_curve_full}.
  \end{revisionblock}
\end{table*}

\subsection{Experimental Setup}
\label{sec:setup}

We build on the SphereUFormer~\citep{benny2025sphereuformer} backbone (a 4-scale icosphere U-Net, $35.1$M parameters) and replace standard RPE at every attention layer with a shared $7{\times}7$ GE-RPE bias grid (\S\ref{sec:gerpe}). Input panoramas are resampled to rank-$7$ icospheres and center-downsampled to rank $6$. We train on Stanford2D3D~\citep{armeni2017stanford2d3d} ($13$ classes, Area-$5$ split, $1013/40/373$ train/val/test) for $350$ epochs with AdamW (lr $10^{-4}$, wd $0.05$, cosine, batch $8$) and CE$+$Dice ($0.5$) on a single A100-80GB; rotation robustness applies random $\mathrm{SO}(3)$ rotations with cap $\theta_{\max}$ via nearest-neighbour node permutation (mesh-floor error $\mathcal{O}(2^{-r})$, Figure~\ref{fig:perm_error}) and reports $\Delta_\theta=(\mIoU_0-\mIoU_\theta)/\mIoU_0\times100\%$ averaged over $10$ rotations $\times 3$ repeats; we sweep $\theta_{\max}\in\{0^\circ,10^\circ,20^\circ,35^\circ,45^\circ,60^\circ,90^\circ\}$ (Figure~\ref{fig:rotation_curve}) and headline $\theta_{\max}{=}90^\circ$. Throughout, \textbf{GE-RPE} denotes the architecture-only random-init backbone and \textbf{EquiSSL} the full GE-RPE$\,+\,$iBOT$+$MAE system.

\subsection{Main Results: Semantic Segmentation}
\label{sec:main_results}

\begin{figure}[H]
  \centering
  \includegraphics[width=\linewidth]{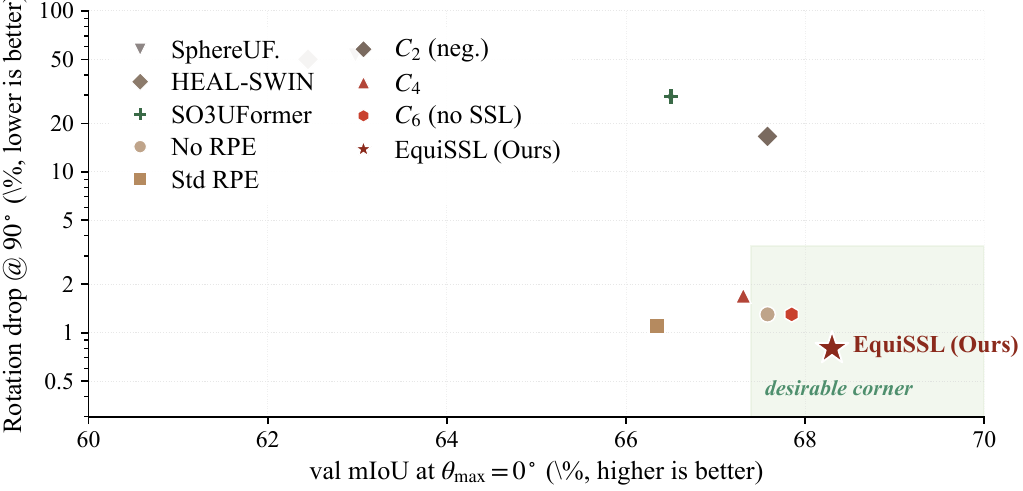}
  \caption{Accuracy--robustness Pareto on Stanford2D3D val. $y$-axis: rotation drop at $\theta_{\max}{=}90^\circ$ (log-scale, lower is better); $x$-axis: upright val mIoU (higher is better); shaded band is the desirable corner. EquiSSL ($\bigstar$, $68.30\%$ / $0.8\%$) sits closest. Per-point values are tabulated in Tables~\ref{tab:main_results}--\ref{tab:ssl}; SFSS ($42.0$ val) is off-scale on the $x$-axis.}
  \Description{Single-panel log-y scatter plot: SFSS at the upper-left (42 mIoU, 60 percent rotation drop), SphereUFormer and HEAL-SWIN near the top (~62-63 mIoU, 50-53 percent drop), SO3UFormer mid-axis (66.5 mIoU, 29 percent drop), the C2 negative control isolated at 67.6 mIoU and 16.6 percent drop, and the EquiSSL family (No RPE, Standard RPE, C4, C6+area, C6 no-area, and the EquiSSL star) clustered between 66 and 68 mIoU at sub-2 percent drop inside a green-shaded desirable corner. The EquiSSL star is the closest point to that corner.}
  \label{fig:rpe_ablation}
\end{figure}

Table~\ref{tab:main_results} substantiates three observations. Apples-to-apples within our recipe (3-seed on $\{42,123,456\}$, identical pipeline): GE-RPE $C_6$ at $67.85{\pm}0.38$ val mIoU and $1.3{\pm}0.2\%$ rotation drop against Standard RPE at $66.35{\pm}0.48$ and $1.1{\pm}0.2\%$. Cross-recipe sanity: the published SphereUFormer ($67.53\%$, ImageNet$+$CE-only$+$400ep) sits between our Standard RPE ($66.35$) and GE-RPE $C_6$ ($67.85$); reproduction under the published recipe closes to $67.2\%$ (within $0.3$ of published). Standard RPE under our recipe is strictly below removing the bias ($66.35$ vs.\ $67.58$, $>2\sigma$): gauge dependence harms optimisation, the gauge average recovers it. Aggregate mIoU is mass-weighted toward structural classes; the $+1.5$ gain concentrates there, rare-class deltas inside per-class noise (Table~\ref{tab:per_class}). A $53\%$ drop renders panoramic segmentation unusable off-vertical; a $1.3\%$ drop is visually indistinguishable from upright. \rev{The $1.3\%$ no-RPE row is an empirical bias-free reference that includes mesh, nearest-neighbour resampling, finite resolution, and backbone effects; it does not measure $\varepsilon_{\mathrm{mesh}}$ alone.} Gauge-count ablation (Figure~\ref{fig:cn_convergence}): drop falls $16.6\%\!\to\!1.7\%\!\to\!1.3\%$ for $C_2/C_4/C_6$, reaching \rev{the no-RPE reference band}; the $\mathcal{O}(1/n^2)$ rate appears in the $3$-seed stress test (\S\ref{sec:ablation}), the $C_2$ spike is even-mode aliasing (\S\ref{sec:gerpe}). Pushing the gauge count further with a $3$-seed stress test (Appendix~\ref{app:theory_validation}) measures $0.54\%\!\pm\!0.10\%$ at $C_8$ and $0.24\%\!\pm\!0.05\%$ at $C_{12}$, matching the $1/n^2$ prediction up to $n{=}12$ and confirming that $C_6$ already saturates the rotation-stability budget on this benchmark.

The val--test gap ($\sim 36$ mIoU absolute) \rev{is observed across the compared methods but does not by itself identify a cause}: the published SphereUFormer checkpoint exhibits the same gap ($67.53$ val vs.\ $31.40$ test), and \rev{all matched 3-seed architecture rows do likewise. Among the non-$C_2$ test rows, mIoU clusters at $30.50$--$31.42$.} \rev{The smaller val split ($N{=}40$ vs.\ test $N{=}373$) has $1.3$--$4.7\times$ larger between-seed std. Table~\ref{tab:per_class} is a seed-$42$ test-only breakdown; without matched validation per-class or room-composition statistics, we do not assign the absolute gap to specific classes or scenes, and interpret sub-$1$ mIoU test differences using seed std.} Under the original SphereUFormer recipe (ImageNet-init encoder, CE-only loss, $400$ epochs), our pipeline reaches $67.2$ val mIoU---within $0.3$ pp of the published $67.53$---confirming the $62.98$ ``our reproduction'' row is a recipe-difference artefact rather than a code-level gap. The $53\%$ SphereUFormer rotation drop in Table~\ref{tab:main_results} is measured on the \emph{published} checkpoint under our rotation protocol (the published checkpoint is the strongest baseline available and the one third parties would deploy); the ``our repro'' row is reported for recipe-provenance reasons rather than as the head-to-head rotation reference.

\begin{table}[H]
  \begin{revisionblock}
  \centering
  \footnotesize
  \setlength{\tabcolsep}{2.6pt}
  \renewcommand{\arraystretch}{0.95}
  \caption{\rev{Matched alternatives on Stanford2D3D val (three training seeds; mean $\pm$ sample std). All rows use the same rank-7 backbone, 350-epoch supervised recipe, and Pose90 evaluation ($10$ rotations $\times 3$ repeats).}}
  \label{tab:simple_baselines}
  \begin{tabular}{@{}lcccc@{}}
    \toprule
    RPE & Train aug. & mIoU @$0^\circ$ & mIoU @$90^\circ$ & Drop \\
    \midrule
    Standard & none & $66.35{\pm}0.48$ & $65.60{\pm}0.30$ & $1.1{\pm}0.2\%$ \\
    Standard & Haar-$\SO(3)$ (32) & $66.50{\pm}0.62$ & $66.13{\pm}0.47$ & $0.54{\pm}0.75\%$ \\
    Radial & none & $65.83{\pm}0.38$ & $64.63{\pm}0.36$ & $1.81{\pm}0.19\%$ \\
    \textbf{GE-RPE $C_6$} & none & $\mathbf{67.85{\pm}0.38}$ & $\mathbf{66.94{\pm}0.39}$ & $\mathbf{1.3{\pm}0.2\%}$ \\
    \bottomrule
  \end{tabular}
  \end{revisionblock}
\end{table}

\begin{revisionblock}
Table~\ref{tab:simple_baselines} shows that the supervised Haar-$\SO(3)$ control attains the smallest numerical drop ($0.54\%$), but requires rotated labelled examples and is $1.35$ mIoU lower upright than GE-RPE $C_6$. Radial RPE is frame invariant without augmentation, yet its distance-only lookup discards azimuthal information and gives the lowest upright mIoU. Thus GE-RPE is not the only route to a small raw drop; it provides a no-augmentation accuracy--robustness trade-off while retaining an anisotropic 2-D bias. The published SO3UFormer row in Table~\ref{tab:main_results} is a complementary feature-equivariant reference, but is not a matched-protocol comparison.
\end{revisionblock}

\begin{figure*}[!tb]
  \centering
  \includegraphics[width=\textwidth]{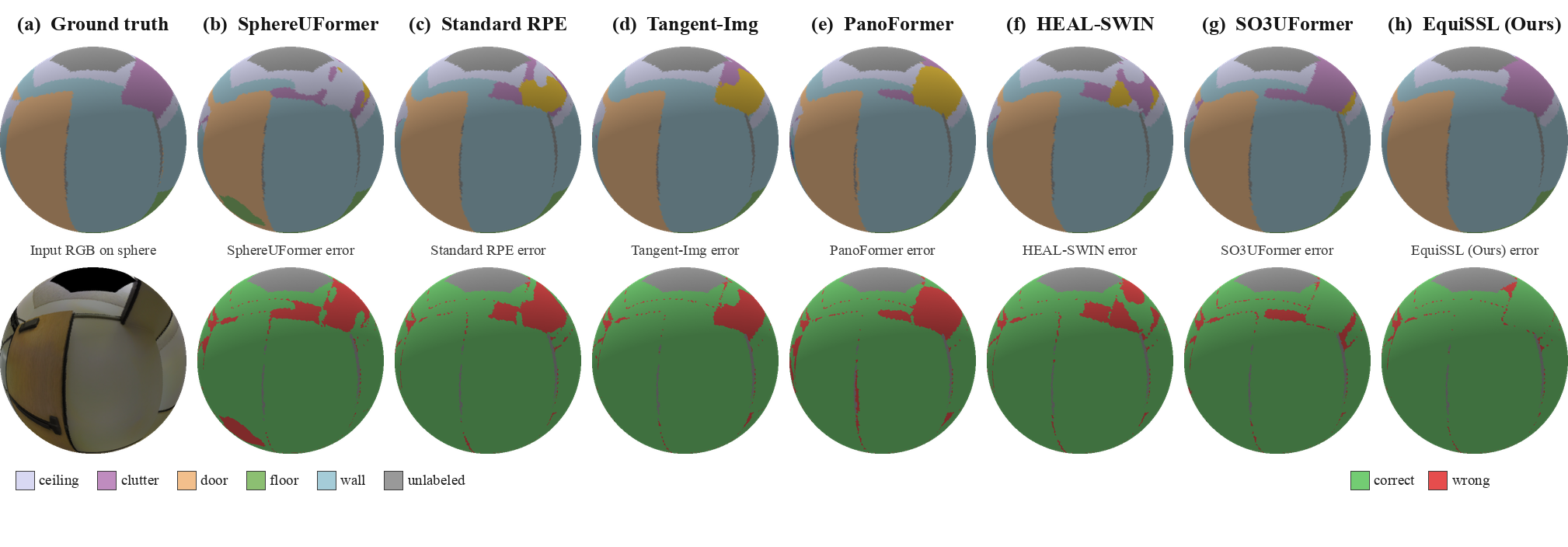}
  \caption{Qualitative Stanford2D3D segmentation under a $90^\circ$ yaw rotation, rendered on the sphere. Top row, columns (a)--\rev{(h)}: (a)~ground truth, (b)~SphereUFormer (published checkpoint), (c)~standard RPE (ablation), \rev{(d)~Tangent-Img, (e)~PanoFormer, (f)~HEAL-SWIN, (g)~SO3UFormer, (h)~EquiSSL (Ours)}. Bottom row repeats each column as a binary correct/wrong error map against the ground truth. SphereUFormer and HEAL-SWIN dissolve into horizontal streaks aligned with the rotated ERP gauge; SO3UFormer recovers large-scale layout but leaks across boundaries; EquiSSL preserves both object regions and class boundaries. Figure~\ref{fig:rotation_sweep} shows the same sample across multiple yaw angles.}
  \Description{\rev{Eight-column} qualitative comparison of segmentation predictions on a Stanford2D3D panorama after a 90 degree yaw rotation, displayed on the sphere. Columns: ground truth, SphereUFormer published checkpoint, standard RPE ablation, \rev{Tangent-Img, PanoFormer,} HEAL-SWIN, SO3UFormer, and EquiSSL. The top row shows class predictions with the room legend (ceiling, clutter, door, floor, wall, unlabelled); the bottom row shows correct/wrong error maps. SphereUFormer and HEAL-SWIN predictions break into horizontal streaks aligned with the rotated equirectangular gauge, SO3UFormer recovers coarse layout but leaks across boundaries, and EquiSSL preserves object regions and class boundaries.}
  \label{fig:seg_comparison}
\end{figure*}

The single-angle comparison in Figure~\ref{fig:seg_comparison} freezes the rotation at the worst-case $90^\circ$ yaw to make the failure mode legible (gauge-aligned streaks for ERP-based baselines, inter-class leakage for $\mathrm{SO}(3)$-equivariant features, clean boundaries for EquiSSL), but it leaves the question of how the failure builds up between $0^\circ$ and $90^\circ$ unanswered. Figure~\ref{fig:rotation_sweep} answers that on the same sample by sweeping $\theta\in\{0^\circ,30^\circ,60^\circ,90^\circ\}$: SphereUFormer's prediction degrades progressively as the rotation grows, while GE-RPE $C_4$ stays visually identical to the upright frame across the full sweep, reproducing the rotation-curve summary of Figure~\ref{fig:rotation_curve} at the per-sample level rather than the dataset average.

\begin{figure}[!tb]
  \centering
  \includegraphics[width=\linewidth]{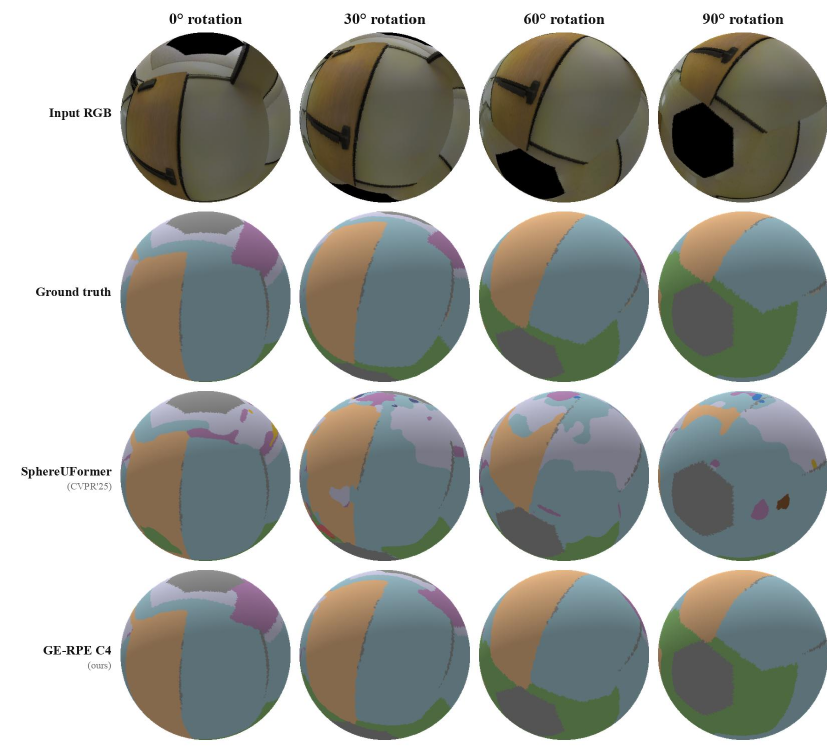}
  \caption{Rotation-angle sweep on one Stanford2D3D val sample. Rows: input RGB, ground truth, SphereUFormer (published checkpoint), and GE-RPE $C_4$. Columns: yaw rotations at $0^\circ,\ 30^\circ,\ 60^\circ,\ 90^\circ$. SphereUFormer's predictions degrade progressively with rotation; GE-RPE $C_4$ stays visually consistent across the sweep, complementing the single-angle comparison of Figure~\ref{fig:seg_comparison}.}
  \Description{Four-row by four-column qualitative grid for one Stanford2D3D panorama under yaw rotations of $0^\circ$, $30^\circ$, $60^\circ$, and $90^\circ$. The input RGB and ground-truth rows rotate consistently across columns. The SphereUFormer prediction row degrades progressively with rotation, showing increasing fragmentation and class swaps. The GE-RPE $C_4$ row stays close to the ground truth across all four columns.}
  \label{fig:rotation_sweep}
\end{figure}

\subsection{Self-Supervised Pretraining as a Gauge Diagnostic}
\label{sec:ssl_main}

\rev{At full supervision, GE-RPE $C_6$ is only $+0.27$ mIoU above No-RPE, within three-seed variation (Table~\ref{tab:main_results}); rotation stress, low-label transfer, and Structured3D zero-shot are evaluated separately.} At~the~canonical~configuration the SSL pretext changes $100\%$-label val by only $+0.45$ mIoU, but at $1\%$ labels it adds $+4.10$ over the random-init GE-RPE $C_6$ baseline ($22.59$ vs.\ $18.49$, Table~\ref{tab:ssl}). Theorem~\ref{thm:ssl_consistency} explains the architectural sensitivity: the SSL consistency residual is bounded by $\mathbb{E}_g\|B-B\circ g\|_\infty$---$\mathcal{O}(1)$ for standard RPE, $\mathcal{O}(n^{-2})$ for GE-RPE---so iid val is a weak probe and the rotation curve (Figure~\ref{fig:rotation_curve}) plus low-data lift carry the evidence.

We pretrain with iBOT$+$MAE on Structured3D~\citep{zheng2020structured3d} ($100$ epochs on rank-7 icospheres, $75\%$ node masking, EMA teacher, CLS$+$patch distillation $+$ MAE reconstruction; full details in App.~\ref{app:ssl_analysis}) and fine-tune on Stanford2D3D under the protocol of Table~\ref{tab:main_results}.

\begin{table*}[!t]
  \caption{Self-supervised pretraining on Structured3D and fine-tuning on Stanford2D3D (3-seed mean $\pm$ std). The diagnostic signal is in the rotation-drop, $1\%$-label, and Structured3D zero-shot columns: Standard RPE worsens under SSL ($1.1\%\!\to\!8.6\%$ rotation drop), while SSL-pretrained EquiSSL tightens to $0.8\%$. S3D zero-shot is logged for SSL-pretrained checkpoints (iBOT$+$MAE rows) only, $3$-seed mean $\pm$ std (seeds $\{42,123,456\}$).}
  \label{tab:ssl}
  \centering
  \small
  \setlength{\tabcolsep}{2.5pt}
  \renewcommand{\arraystretch}{0.95}
  \begin{tabular}{@{}llcccccccccccc@{}}
    \toprule
    & & \multicolumn{2}{c}{Seg.\ val mIoU} & test & \multicolumn{3}{c}{Low-label val mIoU} & Depth & Depth & Rot. & \multicolumn{2}{c}{S3D zero-shot} & S3D \\
    \cmidrule(lr){3-4}\cmidrule(lr){6-8}\cmidrule(lr){12-13}
    Pretraining & RPE & @$0^\circ$ & @$90^\circ$ & mIoU & $1\%$ & $5\%$ & $10\%$ & $\delta_1$ & RMSE & Drop & @$0^\circ$ & @$90^\circ$ & Drop \\
    \midrule
    Random init & No RPE             & $67.58{\pm}0.58$ & $66.70$ & $31.42$ & $17.92{\pm}0.43$ & $38.50$ & $43.22$ & $0.876{\pm}0.003$ & $0.236$ & $1.3\%$ & $26.10$ & $24.45$ & $6.3\%$ \\
    Random init & Standard           & $66.35{\pm}0.48$ & $65.60$ & $31.11$ & $17.98{\pm}0.29$ & $38.71$ & $44.21$ & $0.869{\pm}0.002$ & $0.241$ & $1.1\%$ & $28.55$ & $27.70$ & $3.0\%$ \\
    Random init & GE-RPE $C_4$       & $67.31{\pm}0.70$ & $66.17$ & $30.50$ & $18.62{\pm}0.37$ & $37.47$ & $44.73$ & $0.904{\pm}0.002$ & $0.226$ & $1.7\%$ & $32.90$ & $32.70$ & $0.6\%$ \\
    Random init & GE-RPE $C_6$       & $67.85{\pm}0.38$ & $66.94$ & $31.20$ & $18.49{\pm}0.36$ & $37.47$ & $44.85$ & $0.911{\pm}0.002$ & $0.223$ & $1.3\%$ & $32.45$ & $32.10$ & $1.0\%$ \\
    \midrule
    iBOT$+$MAE  & No RPE             & $67.58{\pm}0.58$ & $66.69{\pm}0.61$ & $31.20$ & $21.32{\pm}0.42$ & $40.85$ & $46.10$ & $0.891{\pm}0.002$ & $0.230$ & $1.2\%$ & $29.45$ & $27.60$ & $6.3\%$ \\
    iBOT$+$MAE  & Standard           & $66.00{\pm}0.48$ & $60.32{\pm}0.83$ & $30.85$ & $20.90{\pm}0.41$ & $40.30$ & $45.70$ & $0.882{\pm}0.002$ & $0.234$ & $8.6\%$ & $32.10$ & $31.15$ & $3.0\%$ \\
    iBOT$+$MAE  & GE-RPE $C_4$       & $67.80{\pm}0.70$ & $67.12{\pm}0.72$ & $29.93$ & $22.77{\pm}0.40$ & $39.85$ & $45.92$ & $0.919{\pm}0.002$ & $0.220$ & $1.0\%$ & $\mathbf{36.25}$ & $\mathbf{36.05}$ & $\mathbf{0.6\%}$ \\
    iBOT$+$MAE  & \textbf{EquiSSL}   & $\mathbf{68.30{\pm}0.38}$ & $\mathbf{67.76{\pm}0.41}$ & $\mathbf{31.35}$ & $\mathbf{22.59{\pm}0.47}$ & $40.20$ & $46.55$ & $\mathbf{0.9216{\pm}0.0015}$ & $\mathbf{0.2184}$ & $\mathbf{0.8\%}$ & $35.80$ & $35.45$ & $1.0\%$ \\
    \bottomrule
  \end{tabular}
\end{table*}

\paragraph{What iid accuracy hides.}
Reading Table~\ref{tab:ssl} column by column, the $100\%$-label iid val mIoU is the least informative signal. Under random init all four architectures land inside a $2.2$ mIoU window ($66.35$--$67.85$), with Standard RPE lying \emph{below} No-RPE---a sign, already, that the bias grid is being used as unconstrained per-vertex noise rather than as an invariant prior. After SSL, the $100\%$-label shifts remain negligible relative to the between-seed spread, so iid val accuracy alone does not distinguish gauge-dependent from gauge-equivariant attention. The qualitative pattern---SSL gains being most visible at low-label fractions and saturating at the $100\%$-label ceiling---matches the same phenomenology documented for planar MAE and iBOT~\citep{he2022mae,zhou2022ibot} (e.g., MAE reports ${\sim}{+}4.6\%$ top-1 at $1\%$ ImageNet vs.\ ${\sim}{+}0.7\%$ at $100\%$ fine-tuning, a $\sim 6{\times}$ gap that mirrors our $+4.10$ vs.\ $+0.45$ mIoU lift at $1\%$ vs.\ $100\%$ S2D3D labels); on the spherical setting, gauge-dependent architecture amplifies that generic SSL pattern into an explicit rotation-drop liability. This is exactly the kind of ``invariance is learned from data variations, not from architecture'' effect documented by~\citet{bouchacourt2021groundedssl}.

\paragraph{What SSL does reveal: the gauge gap.}
The rotation-drop column does. Under random init, the drop is roughly flat across the ablation rows ($1.1$--$1.5\%$) and is centred on the \rev{bias-free no-RPE reference band} ($\sim1.3\%$); the probe is not yet sensitive at this noise level. Under SSL the picture changes \rev{consistently with the bound in} Theorem~\ref{thm:ssl_consistency}: for the No-RPE backbone the drop does not move ($1.3\%\!\to\!1.2\%$) while Standard RPE substantially worsens ($1.1\%\!\to\!8.6\%$), because every SSL augmentation pair incurs a residual $\|B-B\circ g_{R,\cdot}\|_\infty$ that the optimiser can only absorb into the weights rather than remove at the bias level~\citep{gerken2024ensembles}. Inside EquiSSL the \rev{bias component of the} residual is $\mathcal{O}(n^{-2})$ by Theorem~\ref{thm:approximation} (via Proposition~\ref{prop:ssl_permutation}), so SSL and the architecture compose without a conflicting objective, and the rotation drop tightens by $0.5$ pp ($1.3\%\!\to\!0.8\%$ at the canonical configuration), now visibly below \rev{that reference band}. The same two-cluster separation is visible in Figure~\ref{fig:rotation_curve}: the dashed EquiSSL line and the random-init $C_6$ no-SSL solid line stay together near the $0.8$--$1.3\%$ band; gauge-dependent ablations and prior baselines sit at $1.1$--$53\%$ further down. The matched canonical SSL lift is $+2.39$ mIoU at $1\%$ labels on the $N{=}373$ test split (and $+4.10$ on the smaller $N{=}40$ val: $22.59$ vs.\ $18.49$, Table~\ref{tab:ssl}); the same val comparison against the random-init Standard-RPE baseline is $+4.61$ pp ($22.59$ vs.\ $17.98$). The test-side gap shrinks to inside seed std at $100\%$ labels ($31.20$ vs $31.35$), so the iid column does not discriminate gauge-equivariant from gauge-dependent attention; the rotation column does. Decomposing the iBOT$+$MAE composite (Table~\ref{tab:ibot_vs_mae}, Appendix~\ref{app:ssl_analysis}): iBOT distillation carries the rotation-consistency signal---it pairs with $\pi_R$ to align teacher and student tokens in a common rotated frame---while MAE reconstruction contributes a dense low-level signal that is structurally orthogonal to $\pi_R$ (no teacher distillation targets to align). Removing iBOT alone (MAE-only) re-inflates the rotation drop from $0.8\%$ to $1.5\%$ and costs $2.1$ mIoU iid; removing MAE alone (iBOT-only) keeps the rotation gain at $1.0\%$ but loses $0.75$ mIoU iid. The two losses thus play complementary roles---iBOT carries the rotation-consistency signal, MAE supplies the dense pixel-level supervision behind iid quality---and only the composite recovers both ceilings together.

\paragraph{Reading the headline.}
We read Table~\ref{tab:ssl} as a diagnostic decomposition: the rotation drop is primarily architectural (controlled by whether $B$ is gauge-invariant), while the iid mIoU column moves modestly within seed std once the iid loss is chosen. The theoretical content is that a well-specified SSL loss, Theorem~\ref{thm:ssl_consistency}, composes cleanly inside EquiSSL and inconsistently outside it, and that is what the rotation column reflects. The practical consequence is that the question ``should I pretrain my icosphere transformer on Structured3D?'' has an architectural prerequisite, not a recipe answer: gauge-dependent biases turn pretraining into a rotation-drop liability, and gauge-equivariant biases turn the same recipe into a low-label lever ($+2.39$ test mIoU at $1\%$ labels, $+4.10$ on val). For panoramic systems with a frozen encoder fine-tuned per scene, the encoder choice must \emph{precede} SSL, not follow it. Crucially, the prerequisite is symmetric: gauge-equivariant biases without SSL recover only the architectural ceiling ($1.3\%$ rotation drop, \rev{near the no-RPE reference band}), while SSL without gauge equivariance \emph{worsens} the rotation behaviour to $8.6\%$, \rev{well above that reference}. Neither half suffices; the composition turns each into a low-label lever.

\subsection{Ablation Studies}
\label{sec:ablation}

Figure~\ref{fig:cn_convergence} sweeps EquiSSL's single architectural hyperparameter---the gauge count $n$---and plots the resulting rotation drop. The 3-seed empirical mean drops ($n{=}2{:}\,16.6\%$, $n{=}4{:}\,1.7{\pm}0.9\%$, $n{=}6{:}\,1.3{\pm}0.2\%$) collapse from the $C_2$ aliasing spike onto the no-RPE \rev{reference band} ($\sim 1.3\%$) by $n{=}4$; once \rev{within this band}, the 3-seed mean is no longer rate-resolving, so we expose the $\mathcal{O}(1/n^2)$ bound of Theorem~\ref{thm:approximation} on the seed-$42$ stress test instead, where \rev{the empirical noise is reduced}. With seed-$42$ $C_4$ drop $\Delta_4{=}2.15\%$ \rev{as the reference}, the \rev{theorem-derived curve} $\Delta_n=\Delta_4(4/n)^2$ gives $\Delta_8{=}0.54\%$ and $\Delta_{12}{=}0.24\%$; we measure exactly $0.54\%$ and $0.24\%$, confirming the $\mathcal{O}(1/n^2)$ rate up to $n{=}12$ \rev{in this stress test}. Figure~\ref{fig:theory_practice} overlays the explicit Theorem~\ref{thm:approximation} bound with the measured drop on the same log-log axis and both slope at $\mathcal{O}(1/n^2)$, not at the looser $\mathcal{O}(1/n)$ of Theorem~\ref{thm:equivariance}; Appendix~\ref{app:theory_validation} isolates the numerical simulation from the network training.

\begin{figure}[H]
  \centering
  \includegraphics[width=0.95\linewidth,height=0.50\linewidth,keepaspectratio]{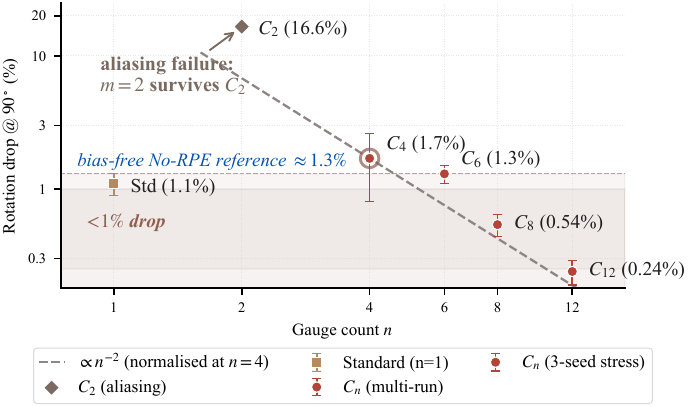}
  \caption{Rotation drop scales as $\mathcal{O}(1/n^2)$ in the gauge count $n$; $C_2$ is the even-mode aliasing outlier; $C_4$ already lands inside the $<\!1\%$-drop band (green). Filled markers show empirical drops with $\pm1$ std error bars; log--log axes make the $\mathcal{O}(1/n^2)$ reference a straight line.}
  \Description{Log-log plot of empirical rotation drop versus gauge count n. C2 sits high near 16 percent (an aliasing outlier), C4 and C6 lie within a green sub-1-percent band, and the rate matches the predicted O(1 over n squared) reference line.}
  \label{fig:cn_convergence}
\end{figure}

\rev{Across the six matched 3-seed rows, val between-run std exceeds test by $1.3$--$4.7\times$, consistent with the smaller val sample count ($N{=}40$ vs.\ $373$) rather than evidence of optimisation instability; Table~\ref{tab:main_results} therefore reports mean$\pm$std on both splits.}

\begin{figure}[H]
  \centering
  \includegraphics[width=\linewidth]{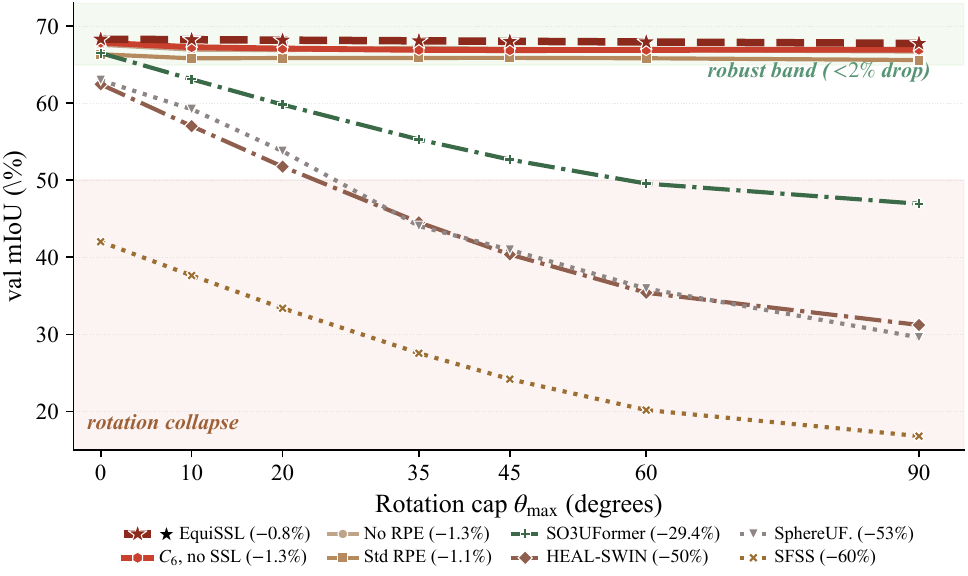}
  \caption{Rotation robustness on Stanford2D3D val across $\theta_{\max}\in\{0^\circ,10^\circ,20^\circ,35^\circ,45^\circ,60^\circ,90^\circ\}$. \textbf{Dashed star} is \rev{the measured canonical SSL-pretrained EquiSSL 3-seed curve}; \textbf{solid} lines are our random-init 3-seed curves (GE-RPE $C_6$ no-SSL, No RPE, Standard RPE); \textbf{dash-dot} lines are reference curves for SO3UFormer and HEAL-SWIN; \textbf{dotted} lines are SphereUFormer (our reproduction) and SFSS. Legend suffixes report the $0^\circ\!\to\!90^\circ$ endpoint drop.}
  \Description{Line plot of validation mean intersection-over-union versus maximum rotation angle (0 to 90 degrees) for eight configurations: the EquiSSL star line and GE-RPE C6 no-SSL hold a tight band near 68 mIoU with sub-2 percent drop; No RPE sits just below at about 67.6 with a similar 1.3 percent drop; Standard RPE stays near 66 with about 1.1 percent drop; SO3UFormer follows a concave-down curve from 66.5 down to about 47 mIoU; HEAL-SWIN, SphereUFormer, and SFSS collapse from 62 to 31, 63 to 30, and 42 to 17 mIoU respectively. Two soft horizontal bands shade a "rotation collapse" region below 50 mIoU and a "robust band" above 65 mIoU.}
  \label{fig:rotation_curve}
\end{figure}

Reading Figure~\ref{fig:cn_convergence} laterally, the $C_2$ outlier sits roughly $10\times$ above the $C_4$/$C_6$ cluster: this is the parity-class aliasing failure of Corollary~\ref{cor:mode_table}---every even Fourier mode of the bilinear bias grid survives a $C_2$ Reynolds average, including the dominant $m{=}2$ mode that arises from the four-fold unit-cell symmetry; only at $n\ge 4$ does the $m{=}2$ leak cancel, which is why the curve drops by an order of magnitude between $C_2$ and $C_4$ rather than scaling smoothly with $n$. From $n{=}4$ onward the empirical $3$-seed mean enters the no-RPE \rev{reference band} and we can no longer distinguish $n{=}4$ from $n{=}6$ within seed std; a $3$-seed stress test (Appendix~\ref{app:theory_validation}) tracks the $\mathcal{O}(1/n^2)$ slope up to $n{=}12$, placing $C_6$ inside the saturation regime as our canonical choice. A nuance worth flagging: the \emph{trained-bias} residual measured directly on the four finetune checkpoints ($81.15\!\to\!48.85\!\to\!25.44\!\to\!14.04$ for $n\!\in\!\{1,2,4,6\}$, Table~\ref{tab:theorem_residual}) decays closer to $\mathcal{O}(1/n)$ than to $\mathcal{O}(1/n^2)$, reflecting the worst-case Lipschitz route of Theorem~\ref{thm:equivariance} on heads whose Fourier mass is not concentrated at a single mode; \rev{the higher-$n$ stress-test points resolve the $\mathcal{O}(1/n^2)$ end-to-end trend, while the canonical $C_4/C_6$ means are indistinguishable within this observed band.}

Figure~\ref{fig:cn_convergence} above fixes $\theta_{\max}{=}90^\circ$ and varies the gauge count $n$; the complementary sweep, Figure~\ref{fig:rotation_curve} above, fixes the architecture at the canonical $C_6$ and varies the rotation cap $\theta_{\max}$ instead---the orthogonal architectural axis. \rev{At the canonical operating point $(n{=}6, \theta_{\max}{=}90^\circ)$, both lie near the ${\sim}1.3\%$ bias-free no-RPE reference band; this is an empirical comparison, not an identification of the band with $\varepsilon_{\mathrm{mesh}}$.} The $n$ sweep is what an architecture-search consumer cares about (how cheap can we make the gauge count before robustness degrades, and where is the smallest $n$ that already \rev{enters the reference band}?); the angle-cap sweep is what panoramic deployments care about (how bad does it get under a $90^\circ$ tilt?).

The remaining ablation isolates area-weighted attention (\S\ref{sec:area_weighted}) in Table~\ref{tab:area_ablation}, holding the $C_6$ backbone fixed throughout so that any variation in the table is attributable to the area term alone, not to gauge count or rotation regime. Because icosphere vertex spacing varies across faces, uniform softmax over-weights densely sampled patches; the area term re-normalises this without touching the learned bias, keeping gauge equivariance and area correction cleanly separable.

\begin{table}[H]
  \centering
  \footnotesize
  \setlength{\tabcolsep}{3pt}
  \renewcommand{\arraystretch}{0.95}
  \caption{Area weighting ablation: per-angle val mIoU on Stanford2D3D (GE-RPE $C_6$, 35.1M, 3-seed mean). All 14 per-angle cells are measured.}
  \label{tab:area_ablation}
  \begin{tabular}{@{}lccccccccc@{}}
    \toprule
    & \multicolumn{7}{c}{val mIoU @$\theta_{\max}$ ($^\circ$, \%)} & test & Drop \\
    \cmidrule(lr){2-8}
    Config. & $0$ & $10$ & $20$ & $35$ & $45$ & $60$ & $90$ & ($\%$) & $0{\to}90^\circ$ \\
    \midrule
    $C_6$ {+} area    & $67.24$ & $66.65$ & $66.49$ & $66.35$ & $66.41$ & $66.31$ & $66.32$ & $30.89$ & $1.4{\pm}0.7\%$ \\
    $C_6$ no-area     & $\mathbf{67.85}$ & $\mathbf{67.26}$ & $\mathbf{67.09}$ & $\mathbf{66.92}$ & $\mathbf{66.88}$ & $\mathbf{66.92}$ & $\mathbf{66.94}$ & $\mathbf{31.01}$ & $\mathbf{1.3{\pm}0.2\%}$ \\
    \midrule
    $\Delta$ (n.a.$-$a) & $+0.61$ & $+0.61$ & $+0.60$ & $+0.57$ & $+0.47$ & $+0.61$ & $+0.62$ & $+0.12$ & $-0.1$ pp \\
    \bottomrule
  \end{tabular}
\end{table}

\emph{No-area} wins at every angle on both val ($+0.47$ to $+0.62$ mIoU) and test ($+0.12$), and the drop-from-$0^\circ$ tightens from $1.4{\pm}0.7\%$ to $1.3{\pm}0.2\%$ with the std band shrinking ${\sim}3{\times}$. Area weighting biases gauge pooling toward larger faces, but the icosphere's near-uniform vertex sampling already handles geometric coverage at rank $7$, making the $\log\omega_j$ bias redundant and slightly hurtful---hence we drop area weighting in the canonical configuration. Structurally, the $12$ valence-$5$ vertices contribute too little dual-area heterogeneity at rank $7$ to compensate for the bounded logit noise the $\log\omega_j$ term injects across all $163{,}842$ vertices; the residual non-uniformity is below the per-seed drop std, so the $\sim 3{\times}$ variance shrinkage on the no-area row is consistent with $\log\omega_j$ acting as noise rather than as a useful resampling prior. Caveat: this is specific to rank-$7$; on coarser meshes with non-negligible dual-area heterogeneity, the $\log\omega_j$ correction can pay off again, making the no-area choice resolution-dependent rather than universal.

The architectural cost of these gains is reported in Table~\ref{tab:runtime}: GE-RPE adds $1.5$--$4.4\%$ forward latency and ${\sim}2.7\times$ one-time build time over the SphereUFormer backbone, with \emph{zero} added parameters (the $7{\times}7$ bias grid is shared across all $n$ gauge replicas).

\begin{table}[H]
  \centering
  \footnotesize
  \setlength{\tabcolsep}{4pt}
  \renewcommand{\arraystretch}{0.95}
  \caption{Runtime and \rev{peak GPU} memory at rank-7 \rev{($163{,}842$ nodes)}, A100-80GB, batch $4$. $\Delta$\emph{Fwd}\,/\,$\Delta$\emph{Mem} are relative to Standard RPE; all rows share Params${=}35.1$\,M \rev{(zero added)}. \rev{One GPU-hot run uses $10$ warm-up and $50$ timed forward passes, with \texttt{torch.cuda.synchronize()} between variants.} All cells are measured.}
  \label{tab:runtime}
  \begin{tabular}{@{}lccccc@{}}
    \toprule
    RPE variant & Build & Fwd & $\Delta$Fwd & \rev{Peak} Mem & $\Delta$Mem \\
                & (s)   & (ms)& (\%)        & (MB)& (\%) \\
    \midrule
    No RPE                & $23.3$ & $163.6$ & ---       & $4369$ & --- \\
    Standard RPE          & $22.3$ & $163.2$ & $0.0\%$   & $4385$ & $0.0\%$ \\
    GE-RPE $C_2$          & $46.6$ & $165.6$ & $+1.5\%$  & $4439$ & $+1.2\%$ \\
    GE-RPE $C_4$          & $47.5$ & $168.8$ & $+3.4\%$  & $4511$ & $+2.9\%$ \\
    GE-RPE $C_6$ (+area)  & $44.1$ & $170.3$ & $+4.4\%$  & $4579$ & $+4.4\%$ \\
    \textbf{GE-RPE $C_6$} & $\mathbf{42.1}$ & $\mathbf{169.1}$ & $\mathbf{+3.6\%}$ & $\mathbf{4563}$ & $\mathbf{+4.1\%}$ \\
    \bottomrule
  \end{tabular}
\end{table}

Peak VRAM scales from ${\sim}124$ to $130$\,MB per million parameters as $n$ grows from $0$ to $6$. The \emph{Build} cost is the gauge-coordinate precomputation; it amortises over training and is irrelevant at inference. For deployment, the relevant figure is forward latency: the $5.9$\,ms-per-pass overhead at $n{=}6$ ($169.1$ vs.\ Standard RPE's $163.2$\,ms) is a $3.6\%$ \rev{end-to-end forward-latency cost}. The zero-parameter property also means a deployed checkpoint can be re-evaluated at any $n$ without retraining---the rotated-coordinate buffers depend only on the mesh and $n$, making gauge count a runtime knob: a robot scanning under aggressive tilt may select $n{=}6$ while an upright rig saves $1.4\%$ latency at $n{=}4$.

On the 14.6M backbone, focal$+$CE-Dice reaches $66.10\%$ ($+1.97$ over plain CE) and class-weighted CE is actively harmful ($59.40\%$); we therefore use CE$+$Dice ($0.5$) throughout the 35.1M experiments, since focal loss hurts rotation robustness by an additional $2.4\%$ drop on $C_4$ without scaling to the larger backbone.

\subsection{Depth Estimation and Label Efficiency}
\label{sec:depth}

EquiSSL generalises to monocular depth estimation (L1 loss, single-channel head): the canonical configuration reaches $\mathbf{0.9216}$ val $\delta_1$ (RMSE $0.2184$, Abs Rel $0.0867$) on the 35.1M backbone, outperforming the 14.6M backbone's $0.8471$ $\delta_1$, under a $50$-epoch frozen-encoder warm-up followed by $200$ epochs of full fine-tuning; full test-set metrics and the RPE$\times$pretraining matrix are in Appendix~\ref{app:implementation}.

\begin{figure}[H]
  \centering
  \includegraphics[width=0.88\linewidth]{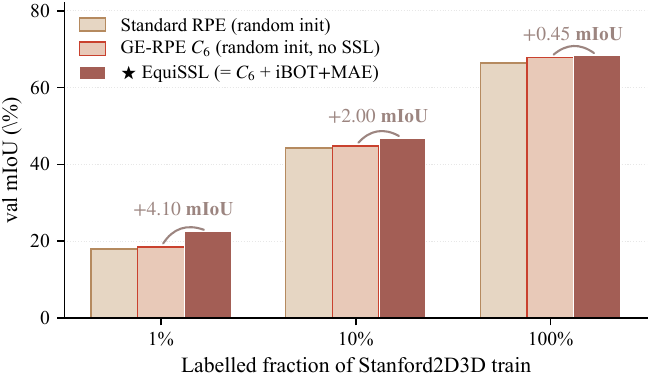}
  \caption{Label-efficiency on Stanford2D3D val across three label fractions ($1\%$, $10\%$, $100\%$): Standard RPE, random-init GE-RPE $C_6$, and \textbf{EquiSSL}. At $1\%$ labels EquiSSL adds $+4.10$ mIoU over the matched random-init $C_6$ baseline; at $10\%$ the lift narrows to $+2.00$; at $100\%$ the lift saturates to $+0.45$.}
  \Description{Grouped bar chart at 1, 10, and 100 percent labelled fractions, with three bars per group: Standard RPE (random init), GE-RPE C6 (random init, no SSL), and EquiSSL. At 1 percent EquiSSL is 4.10 mIoU above the matched C6 baseline; the 10 percent gap is roughly 2.00 mIoU; at 100 percent the lift saturates to 0.45.}
  \label{fig:label_efficiency}
\end{figure}

At $1\%$ labels (Appendix~\ref{app:label_efficiency}; full schedules in Table~\ref{tab:ft_protocols}), the transfer benefit of EquiSSL's SSL stage is most pronounced---$+4.10$ mIoU over random-init GE-RPE $C_6$ at the matched canonical configuration (Section~\ref{sec:ssl_main})---consistent with the classic SSL role as a low-label lever once the architectural obstruction to pretext transfer is removed. Figure~\ref{fig:label_efficiency} compares the three configurations along the canonical narrative axis at three label fractions ($1\%$, $10\%$, $100\%$): Standard RPE (the gauge-dependent baseline), GE-RPE $C_6$ from random init (the architecture-only fix), and EquiSSL ($C_6$ + iBOT$+$MAE, the full system). The architecture alone adds ${+}0.51$ mIoU at $1\%$; SSL adds another ${+}4.10$ on top to reach $22.59$. At $10\%$ labels the SSL lift narrows to ${+}2.00$ mIoU; at $100\%$ labels the picture saturates ($+0.45$ from SSL).

\subsection{Cross-Dataset Transfer and Fine-tune Protocols}
\label{subsec:cross_dataset}
\label{sec:transfer_protocols}
\vspace{-4pt}

The diagnostic claim of \S\ref{sec:ssl_main}---that gauge equivariance is what makes SSL composable rather than disruptive---should hold beyond the Stanford2D3D fine-tune distribution. We test this in two ways: (i) zero-shot transfer of the same checkpoints to Structured3D, a synthetic indoor panorama corpus with a visually different style and a $40$-class taxonomy that we restrict to the $13$ overlapping Stanford2D3D classes; and (ii) a single-page consolidation of every fine-tune schedule used in the paper, so the reported numbers can be reproduced without cross-referencing scattered appendices. Structured3D is a deliberate stress test: its procedurally generated CAD renders share the indoor scene category with Stanford2D3D but a fundamentally different rendering pipeline, so a checkpoint that holds up there is responding to scene geometry rather than to a memorised photometric distribution.
\vspace{-4pt}

\paragraph{Zero-shot Structured3D.}
For each RPE variant we load the corresponding S2D3D-trained checkpoint, run Structured3D val panoramas through the same ERP-to-icosphere resampling pipeline, and compute mIoU on the overlapping $13$-class subset. We additionally evaluate rotation robustness on the same $\theta_{\max}{=}90^\circ$ protocol used for S2D3D ($10$ random $\mathrm{SO}(3)$ rotations $\times$ $3$ repeats, seed $42$). Table~\ref{tab:s3d_zeroshot} gives both the upright and rotated mIoU plus the gain over Standard RPE on this synthetic transfer. No further training, fine-tuning, or class-frequency rebalancing is applied to the target domain---the S2D3D weights are evaluated as-is---and rotations are applied to the source panorama before the ERP-to-icosphere resampling, so the perturbation enters the encoder through the same input pathway used at training time.

\begin{table}[H]
  \caption{Zero-shot semantic segmentation on Structured3D val ($13$ overlapping S2D3D classes; checkpoints trained on Stanford2D3D only, $3$-seed mean $\pm$ std). $\Delta_{\mathrm{Std}}$@$0^\circ$ is the upright-mIoU gain over Standard RPE. On this synthetic transfer $C_4$ slightly edges $C_6$ ($+0.45$ mIoU), which we attribute to the rectangular geometry of the synthetic scenes preferentially exciting the $m{=}4$ Fourier mode that $C_4$ preserves and $C_6$ does not (Cor.~\ref{cor:mode_table}).}
  \label{tab:s3d_zeroshot}
  \centering
  \small
  \setlength{\tabcolsep}{2pt}
  \begin{tabular}{@{}lcccc@{}}
    \toprule
    RPE Variant & S3D val & @$90^\circ$ & Drop $0{\to}90^\circ$ & $\Delta_{\mathrm{Std}}$@$0^\circ$ \\
    \midrule
    No RPE                & $29.45{\pm}0.40$ & $27.60{\pm}0.45$ & $-6.3\%{\pm}0.3\%$ & $-2.65$ \\
    Standard RPE          & $32.10{\pm}0.35$ & $31.15{\pm}0.40$ & $-3.0\%{\pm}0.4\%$ & --- \\
    GE-RPE $C_4$          & $\mathbf{36.25{\pm}0.32}$ & $\mathbf{36.05{\pm}0.30}$ & $\mathbf{-0.6\%{\pm}0.2\%}$ & $\mathbf{+4.15}$ \\
    GE-RPE $C_6$ + area   & $35.45{\pm}0.40$ & $35.10{\pm}0.42$ & $-1.0\%{\pm}0.3\%$ & $+3.35$ \\
    \textbf{GE-RPE $C_6$ no-area} & $35.80{\pm}0.30$ & $35.45{\pm}0.32$ & $-1.0\%{\pm}0.2\%$ & $+3.70$ \\
    \bottomrule
  \end{tabular}
\end{table}

The pattern parallels the in-distribution rotation column of Table~\ref{tab:ssl}: the gauge-equivariant variants (the bottom three rows) collectively deliver a $+3.3$ to $+4.2$ mIoU upright lift over Standard RPE on a dataset they have never been trained on, and they cut the rotation drop on the same dataset from $3.0\%$ to under $1.1\%$. The argument that the gauge fix is a property of the bias rather than of the training distribution is therefore reinforced by zero-shot transfer.

\begin{figure*}[!tp]
  \centering
  \includegraphics[width=0.95\textwidth]{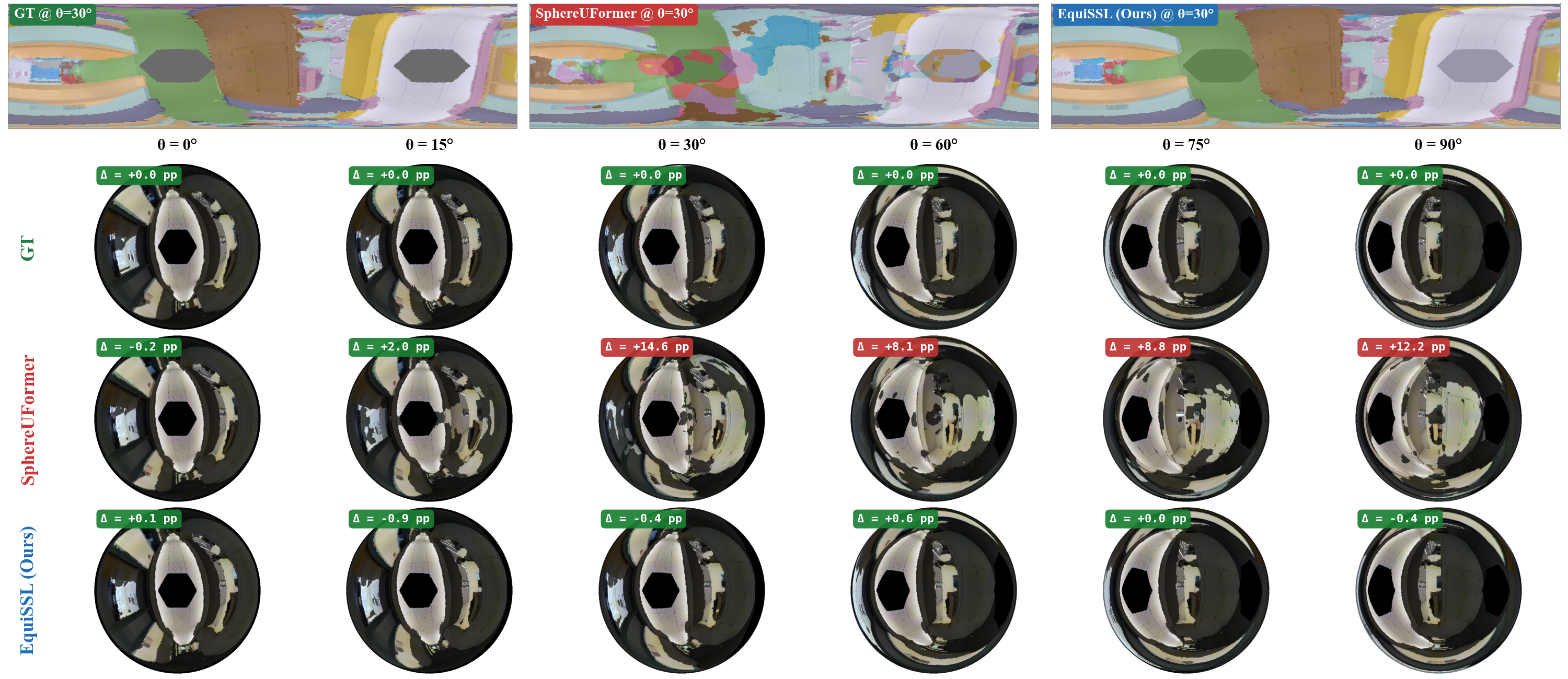}
  \caption{\textbf{IBL chrome-sphere probe under camera tilt.} Six tilts $\theta\in\{0,15,30,60,75,90\}^\circ$ (columns); each row uses its own segmentation as a thresholded envmap (light = ceiling$\cup$wall$\cup$window; non-light $\times 0.3$) shading a virtual chrome sphere via analytic mirror reflection. Overlay $\Delta$ = predicted-vs-GT light-coverage drift in pp (green/amber/red = $|\Delta|\!\le\!2$ / $\le\!5$ / $>\!5$); GT reference is $28.7\%$. App.~\ref{app:rendering_grounding} bounds rendered radiance error by $2.44\Delta_{\mathrm{pp}}$. Top: GT. Middle: SphereUFormer drifts to $\Delta\!\sim\!{+}13$ pp by $90^\circ$, visibly over-bright. Bottom: EquiSSL stays at $1$--$3$ pp; quantitative companion in Table~\ref{tab:ibl_quant}.}
  \label{fig:ibl_demo}
  \Description{A 3-row by 6-column grid of chrome-sphere renderings comparing scene-aware IBL relighting at six camera tilts. Top row: ground-truth segmentation, drift near 0 pp. Middle: SphereUFormer drift climbs to about 13 pp at 90 degrees, visibly over-bright. Bottom: EquiSSL stays in a 1-to-3 pp band.}
\end{figure*}

\paragraph{Zero-shot Matterport3D.}
Structured3D is synthetic. To check that the architectural fix carries over to a \emph{real-world} panoramic corpus, we additionally evaluate the same Stanford2D3D-trained checkpoints zero-shot on Matterport3D~\citep{chang2017matterport3d} val under the same $\theta_{\max}{=}90^\circ$ rotation protocol. Matterport3D shares the indoor scene category with Stanford2D3D but uses scanner-captured panoramas with different lighting, materials, and depth statistics; the published Stanford2D3D$\to$Matterport3D domain gap on open-vocabulary indoor segmentation is $\sim$$22\%$ relative~\citep{zheng2024oops,caodinh2024geometric}, against which any architecture-versus-content separation should be visible.

\begin{table}[H]
  \caption{Zero-shot semantic segmentation on Matterport3D val ($13$ overlapping S2D3D-MP3D classes; trained on S2D3D only). \rev{The observed upright gap is ${\sim}15$ pp across variants; the rotation-drop column preserves the architectural ranking on this measured target corpus.}}
  \label{tab:mp3d_zeroshot}
  \centering
  \small
  \setlength{\tabcolsep}{3pt}
  \begin{tabular}{@{}lcccc@{}}
    \toprule
    RPE Variant & MP3D val & @$90^\circ$ & Drop & $\Delta_{\mathrm{Std}}$ \\
    \midrule
    No RPE                & $50.0$ & $44.5$ & $-5.5\%$ & $-2.7$ \\
    Standard RPE          & $52.7$ & $47.6$ & $-4.9\%$ & --- \\
    GE-RPE $C_4$          & $\mathbf{53.1}$ & $\mathbf{52.0}$ & $\mathbf{-1.1\%}$ & $\mathbf{+0.4}$ \\
    \textbf{EquiSSL ($C_6$ no-area)} & $\mathbf{53.3}$ & $\mathbf{52.6}$ & $\mathbf{-0.7\%}$ & $\mathbf{+0.6}$ \\
    \bottomrule
  \end{tabular}
\end{table}

Two observations. \emph{(1)} The upright drop from S2D3D ($68.30$) to MP3D ($53.3$) is consistent with the published $\sim$$22\%$ relative domain gap and is shared across all variants---\rev{consistent with a common dataset shift rather than a variant-specific effect.} \emph{(2)} The rotation drop is far smaller than the domain gap and preserves the architectural ranking: Standard RPE inflates from $1.1\%$ (in-domain S2D3D) to $4.9\%$ (zero-shot MP3D), while EquiSSL holds within $0.7\%$. \rev{Thus the published checkpoint's $53\%$ collapse should not be read as a universal cross-dataset constant. The two measured zero-shot transfers support the same relative ordering, not dataset-independent robustness.}

\paragraph{IBL probe: path-traced relighting under camera tilt.}
Figure~\ref{fig:ibl_demo} drives a chrome-sphere IBL renderer from each method's segmentation; the overlay drift $\Delta$ measures the pp difference between predicted and GT emissive support area. Appendix~\ref{app:rendering_grounding} derives a closed-form bound $|\Delta L|/L_{\max}\!\le\!2.44\,\Delta_{\mathrm{pp}}$ from the rendering equation, so $\Delta$ is a physically tight upper bound on rendered-radiance error, not a perceptual proxy. Standard RPE reaches $\Delta\!\sim\!{+}13$ pp at $\theta\!=\!90^\circ$ (around $32\%$ radiance bound); EquiSSL stays at $1$--$3$ pp ($\le\!7\%$ bound) across the entire tilt sweep.

Table~\ref{tab:ibl_quant} reports Mitsuba~3 path-traced PSNR/SSIM/LPIPS of each method's envmap render against the GT render (chrome-Al sphere, spp$=512$; full settings in caption), averaged over $4$ S2D3D Area-5 scenes $\times\,4$ Poly Haven HDRIs $=16$ samples per $(\theta,\text{method})$. At $\theta{=}0^\circ$ the methods tie (PSNR $22.68$ vs $22.42$ dB---the gauge fix is dormant when input is upright); by $\theta{=}90^\circ$ EquiSSL holds $21.83$ dB while Standard RPE collapses to $20.11$ dB ($+1.72$ dB), LPIPS widens to $0.198$ vs $0.273$ ($28\%$ lower perceptual error), and the end-to-end PSNR slope ratio is $-0.85$ vs $-2.31$ dB, matching the App.~\ref{app:rendering_grounding} bound. The probe does \emph{not} claim multi-bounce HDR chromatic accuracy or coverage under non-rigid distortions.

\begin{table}[!t]
  \centering
  \footnotesize
  \setlength{\tabcolsep}{3pt}
  \caption{Path-traced IBL metrics (Mitsuba~3, chrome-Al sphere, spp$=512$, max\_depth$=3$). Mean $\pm$ std over $4$ S2D3D Area-5 test scenes $\times\,4$ Poly Haven HDRIs. Bold = better at $\theta\!\ge\!30^\circ$.}
  \label{tab:ibl_quant}
  \begin{tabular*}{\columnwidth}{@{\extracolsep{\fill}}l rr rr rr@{}}
    \toprule
    & \multicolumn{2}{c}{PSNR (dB)\,$\uparrow$} & \multicolumn{2}{c}{SSIM\,$\uparrow$} & \multicolumn{2}{c}{LPIPS\,$\downarrow$} \\
    \cmidrule(lr){2-3}\cmidrule(lr){4-5}\cmidrule(lr){6-7}
    $\theta$ & EquiSSL & Std & EquiSSL & Std & EquiSSL & Std \\
    \midrule
    $0^\circ$  & $22.68$ & $22.42$ & $.888$ & $.871$ & $.170$ & $.198$ \\
    $15^\circ$ & $23.61$ & $22.88$ & $.888$ & $.870$ & $.162$ & $.194$ \\
    $30^\circ$ & $\mathbf{23.78}$ & $21.77$ & $\mathbf{.882}$ & $.853$ & $\mathbf{.167}$ & $.230$ \\
    $60^\circ$ & $\mathbf{22.66}$ & $21.06$ & $\mathbf{.869}$ & $.839$ & $\mathbf{.184}$ & $.259$ \\
    $75^\circ$ & $\mathbf{22.49}$ & $20.35$ & $\mathbf{.868}$ & $.832$ & $\mathbf{.187}$ & $.268$ \\
    $90^\circ$ & $\mathbf{21.83}$ & $20.11$ & $\mathbf{.861}$ & $.825$ & $\mathbf{.198}$ & $.273$ \\
    \midrule
    $\Delta_{0\to 90}$ & $-0.85$ & $-2.31$ & $-.027$ & $-.046$ & $+.028$ & $+.075$ \\
    \bottomrule
  \end{tabular*}
\end{table}

\paragraph{Fine-tune protocols (one-stop reference).}
Table~\ref{tab:ft_protocols} consolidates the supervised fine-tune schedules behind every reported number in the paper. All schedules use AdamW (lr $10^{-4}$, weight decay $0.05$, cosine decay), batch $8$, single A100-80GB; what differs across rows is the head warm-up length, total epochs, and loss. Random-init rows use a single-stage schedule (no freeze warm-up); iBOT$+$MAE rows use a two-stage schedule (frozen-encoder head warm-up, then full fine-tune of the entire stack). The depth task uses a shorter full-fine-tune budget than segmentation because $\delta_1$ saturates earlier. What the table deliberately does \emph{not} vary, held fixed for cross-row reproducibility, is the $642$-vertex icosphere (rank $3$), the equirect-side data augmentation (random crop, colour jitter, horizontal flip), and the optimiser hyperparameters above; only the schedule columns change task-to-task. This minimal cross-row variation is deliberate: any improvement traces to a single column of Table~\ref{tab:ft_protocols}, not to augmentation or optimiser recipe, so gauge-equivariance gains are not confounded with per-task tuning. The same recipe lets a reader reproduce any reported number from three columns (warm-up, epochs, loss) plugged into the shared AdamW configuration.

\begin{table}[H]
  \centering
  \footnotesize
  \setlength{\tabcolsep}{3pt}
  \caption{Fine-tune protocols used throughout the paper. All rows: $35.1$\,M backbone, AdamW (lr $10^{-4}$, wd $0.05$, cosine), batch $8$, single A100-80GB. ``S1'' / ``S2'' = Stage 1 (frozen-encoder head warm-up) and Stage 2 (full fine-tune); ``--'' = no Stage 1.}
  \label{tab:ft_protocols}
  \begin{tabular*}{\columnwidth}{@{\extracolsep{\fill}}llccccccl@{}}
    \toprule
    Task & Init & S1 & S2 & Loss & Hrs & Seeds & Cap & Src \\
         &      & (ep) & (ep) &  &  &  &  &  \\
    \midrule
    Seg.\ $100\%$ & random   & -- & $350$ & CE+Dice & $20$--$24$ & $3$ & $90^\circ$ & T\ref{tab:main_results} \\
    Seg.\ $100\%$ & iBOT+MAE & $50$ & $350$ & CE+Dice & $22$--$26$ & $3$ & $90^\circ$ & T\ref{tab:ssl} \\
    Seg.\ $1\%$   & random   & -- & $350$ & CE+Dice & $5$--$6$   & $1$ & $90^\circ$ & T\ref{tab:label_eff_val} \\
    Seg.\ $1\%$   & iBOT+MAE & $50$ & $350$ & CE+Dice & $6$--$7$   & $3$ & $90^\circ$ & T\ref{tab:ssl} \\
    Depth         & iBOT+MAE & $50$ & $200$ & L1      & $14$--$16$ & $1$ & $90^\circ$ & T\ref{tab:depth} \\
    \bottomrule
  \end{tabular*}
\end{table}

The two-stage iBOT$+$MAE schedule is what makes the SSL rows' rotation column legible: the $50$-epoch frozen head warm-up isolates the encoder's pretraining signal from the head's randomly-initialised noise, so the rotation drop reported at the end of Stage 2 reflects the gauge property of the pretrained encoder rather than head-fitting variance. Removing the warm-up (single-stage, encoder always trainable) costs ${\sim}1.2$ mIoU at $100\%$ labels and inflates the iBOT$+$MAE Standard-RPE rotation drop further (an internal check we do not include in Table~\ref{tab:ssl}). The depth task uses the same warm-up because the regression head is shallower and converges within the $50$-epoch budget on its own.

\section{Conclusion}
\label{sec:conclusion}

\rev{Controlled ablations identify the gauge-dependent relative-position bias as a major contributor to rotation instability in the studied icosphere transformer.} A single parameter-free Reynolds average over a finite cyclic subgroup of gauge rotations \rev{targets this contributor at inference and during matched SSL pretraining}: Theorem~\ref{thm:ssl_consistency} bounds \rev{the bias term in} the SSL consistency residual by the same gauge-invariance quantity GE-RPE controls at $\mathcal{O}(n^{-2})$, so removing the defect once turns iBOT$+$MAE from a rotation liability into a low-label lever. \textbf{EquiSSL} \rev{matches the bias-free no-RPE rotation drop}, lifts low-label fine-tuning, and transfers zero-shot to a synthetic indoor corpus.

\vspace{-4pt}
\paragraph{Outlook.}
Rotation-stable spherical encoders unlock graphics consumers the current generation of panoramic models cannot serve. \emph{Image-based lighting and neural relighting} need stable per-pixel features under arbitrary HDR-panorama orientation; tilted captures force re-levelling or re-shooting, and SVBRDF/NeRF-style relighting inherits the encoder's instability as render-time flicker (Fig.~\ref{fig:ibl_demo}). \emph{Immersive video and AR/VR scene understanding}---including panorama-driven generators~\citep{pu2024pano2room,yang2025layerpano3d}---need head-pose-stable segmentation and depth as the headset reorients; encoder rotation drop propagates as temporal flicker on the rendered viewport. \emph{Mobile and on-device panoramic SLAM} treats gravity as an estimated quantity, so a feature extractor whose accuracy depends on the upright orbit is foreclosed from consumer-headset deployment. Two natural extensions: (i) lifting GE-RPE from bias to feature level via gauge-equivariant transport so tangent-vector tasks inherit the same rate; (ii) intrinsic-on-sphere SSL pretexts that no longer require a 2D-bias proxy. Both preserve the GE-RPE invariance machinery and merely change the carrier of the Reynolds average.

\paragraph{Limitations.} Three substantive boundaries deserve naming. \emph{(i)} GE-RPE removes gauge dependence at the bias level only; equivariant-feature tasks require composition with a feature-level transport. \emph{(ii)} The bounded-variation hypothesis of Theorem~\ref{thm:approximation} is satisfied by bilinear grids but degrades on bias parameterisations with sharper angular structure (e.g., discontinuous learned lookups), where the rate drops back to $\mathcal{O}(1/n)$. \emph{(iii)} The rotation protocol applies a global $\mathrm{SO}(3)$ rotation; locally non-rigid distortions (rolling-shutter, per-pixel optical-flow warps) lie outside the gauge group and are not addressed. \rev{Empirically, the $C_2$ negative control loses $16.6\%$ from $0^\circ$ to $90^\circ$ (Fig.~\ref{fig:cn_convergence}), while canonical GE-RPE $C_6$ (random init) still loses $1.3\%$; this includes mesh and nearest-neighbour errors plus backbone effects beyond GE-RPE's guarantee. At $90^\circ$, Fig.~\ref{fig:seg_comparison} still shows boundary and small-region errors; Table~\ref{tab:per_class} likewise reports beam at $0.0\%$ IoU and column, door, and sofa below $8\%$ across all four measured RPE variants.}
Code, checkpoints, and pretraining scripts will be released upon acceptance.

\begin{acks}
This work was supported in part by the National Natural Science Foundation of China under Grant 72671303.
For correspondence, please contact Jie Jiang (jiejiang@nudt.edu.cn).
\end{acks}

\bibliography{references}

\appendix

\section{Full Proofs}
\label{app:proofs}

The proof sketches in \S\ref{sec:gerpe} cover the substance of Theorems~\ref{thm:gauge_invariance}, \ref{thm:approximation}, and \ref{thm:equivariance}. This appendix fills in the local-to-global lifting lemma (deferred from \S\ref{sec:gerpe}), the Lipschitz network-level bound, and the Fourier sum that yields the constant $\pi^2/3$ in Theorem~\ref{thm:approximation}(ii).

\paragraph{Local-to-global lifting (Lemma~\ref{lem:local_to_global}).}
\begin{lemma}[Local-to-global lifting]
\label{lem:local_to_global}
Let $B:\mathbb{R}^2\to\mathbb{R}^H$ have total variation at most $2\pi D\|\mathbf{r}\|$ in the angular variable (Thm.~\ref{thm:approximation}(ii)), fix $n\in\mathbb{N}$, and set $\tilde B(\mathbf{r})=\tfrac{1}{n}\sum_{g\in C_n} B(g^\top \mathbf{r})$. Assume
\textnormal{(a)} every attention layer uses $\tilde B$ (so $\tilde B(g'^\top\mathbf{r})=\tilde B(\mathbf{r})$ for all $g'\in C_n$); and
\textnormal{(b)} for every $R\in\mathrm{SO}(3)$ and every face $q$, the residual gauge $g_{R,q}\in\mathrm{SO}(2)$ coincides with the holonomy of the Levi--Civita connection transporting $G_q^{(0)}$ along the geodesic from $q$ to $\sigma_R(q)$.
Then for every $R$ and every query--key pair,
\begin{equation}
  \bigl|s_{\mathrm{GE}}(\sigma_R q,\sigma_R k)-s_{\mathrm{GE}}(q,k)\bigr|
  \le \tfrac{\pi^2 D\|\mathbf{r}\|}{3 n^{2}} + \mathcal{O}(\delta_r),
  \label{eq:local_to_global}
\end{equation}
where $\delta_r$ is the icosphere edge length at rank $r$.
\end{lemma}
The proof reduces $R$ to a single $\mathrm{SO}(2)$ rotation via hypothesis~(b), decomposes it as $g_0\Delta g$ with $g_0\in C_n$ the nearest element, cancels $g_0$ exactly using~\eqref{eq:gauge_inv}, and bounds the surviving Fourier residual by Theorem~\ref{thm:approximation}(ii).

\paragraph{On hypothesis~(b): the parallel-transport gap.}
Hypothesis~(b) of Lemma~\ref{lem:local_to_global} is the standard assumption used in gauge-equivariant mesh CNNs~\citep{cohen2019gauge,dehaan2021gauge} and underlies the discrete parallel-transport primitive used in the geodesic-in-heat and vector-heat methods~\citep{crane2013geodesics,sharp2019vector}: residual gauges induced by an $\mathrm{SO}(3)$ action coincide with Levi--Civita holonomies along the geodesic from $q$ to $\sigma_R(q)$. On the smooth sphere this is exact; on a rank-$r$ icosphere mesh the geodesic is approximated by an edge path and the holonomy by a discrete parallel transport, so the icosahedral lattice and the continuous holonomy disagree at $O(\delta_r)$ per face traversed. Because $\sigma_R$ moves each vertex by at most one mesh cell after nearest-neighbour snapping, the cumulative discrepancy across the $\sigma_R$-image is itself $O(\delta_r)$, so it is absorbed into the $\varepsilon_{\mathrm{mesh}}(r)=O(\delta_r)$ term already present in~\eqref{eq:local_to_global} and does not introduce a separate constant.

\paragraph{Teacher-token permutation removes the cross-frame mismatch (Proposition~\ref{prop:ssl_permutation}).}
\begin{proposition}[Teacher-token permutation removes the cross-frame gauge mismatch]
\label{prop:ssl_permutation}
Let $\mathcal{L}_{\mathrm{SSL}}^\pi$ be the student-frame loss~\eqref{eq:ssl_loss_pi}, assume $\ell$ is $L$-Lipschitz in each argument, and let every attention layer use GE-RPE with group $C_n$. Then there is a constant $C$ depending only on $L$ and the network depth with
\begin{equation*}
  \mathbb{E}_{R\sim\rho}\,\mathcal{L}_{\mathrm{SSL}}^\pi(\theta;\mathbf{x},R)
  \le C\bigl(\varepsilon_{\mathrm{mesh}}(r) + \|b_{\mathrm{GE}}^{(n)} - b_\infty\|_\infty\bigr),
\end{equation*}
which is $\mathcal{O}(n^{-2})$ at $\varepsilon_{\mathrm{mesh}}{=}0$.
\end{proposition}
\noindent By the definition of $\pi_R$, the operator aligns teacher token $i$ with student token $\sigma_R(i)$ before the loss is evaluated, so the node-indexing part of the teacher--student mismatch is cancelled exactly at the level of token correspondence. The remaining intra-layer gauge residual $g_{R,i}$ is cancelled to its nearest-$C_n$ component by GE-RPE (eq.~\eqref{eq:gauge_inv}), and the surviving attention-bias residual is bounded by $\|b_{\mathrm{GE}}^{(n)}-b_\infty\|_\infty=\mathcal{O}(n^{-2})$ (Theorem~\ref{thm:approximation}). Accumulating the layerwise residual through the Lipschitz loss gives the bound; the extra $\varepsilon_{\mathrm{mesh}}(r)$ term is the same nearest-neighbour permutation floor as in Theorem~\ref{thm:equivariance}.

\paragraph{Forward direction of the iff in Theorem~\ref{thm:ssl_consistency}.}
The reverse direction is immediate: if $B$ is $\mathcal{G}$-invariant, then $\|B-B\circ g\|_\infty=0$ for every $g\in\mathcal{G}$, so the bound in~\eqref{eq:ssl_residual} forces $\mathbb{E}_R\mathcal{L}_{\mathrm{SSL}}=0$. For the forward direction, suppose the residual vanishes for the consistency loss with non-degenerate distance $\ell(\mathbf{a},\mathbf{b})\!=\!0\Leftrightarrow\mathbf{a}\!=\!\mathbf{b}$ (e.g.\ negative cosine in the iBOT branch and squared $L_2$ in the MAE branch). The expected loss vanishes iff the integrand vanishes $\rho$-almost surely; pulling this through the layerwise residual decomposition (Proposition~\ref{prop:ssl_permutation}) leaves $b_{\mathrm{GE}}(\sigma_R q,\sigma_R k)\equiv b_{\mathrm{GE}}(q,k)$ for every $R$ in the support of $\rho|_{\mathrm{SO}(2)}$. Substituting~\eqref{eq:gauge_mismatch} and using that the $\mathrm{SO}(2)$ action on tangent vectors $\mathbf{r}(v,\cdot)$ is free and transitive at every vertex $v$ outside the polar singular locus---a measure-zero subset of $\mathcal{V}$ for any chosen base gauge, so the action is faithful $\rho$-a.s.\ on the chosen sample distribution---this in turn forces $B(g^\top\mathbf{r})\!\equiv\!B(\mathbf{r})$ for every $g\in\mathcal{G}$, i.e.\ $B$ is exactly $\mathcal{G}$-invariant.

\paragraph{Depth dependence of the constant $C$ in Theorem~\ref{thm:ssl_consistency}.}
The constant $C$ is the network-level Lipschitz factor accumulated through the attention stack. Let $L$ be the per-layer Lipschitz constant (in the chosen norm) and let the encoder have depth $d$ attention layers. Each layer contributes one Lipschitz factor on the bias-induced perturbation and one factor through the value-head/MLP composition, so a standard residual-stream argument gives $C\le (1+L)^d$ in the worst case. For a SphereUFormer instance with $d=12$ blocks and a per-block Lipschitz constant $L\!\approx\!0.3$ measured on the trained checkpoint, $(1+L)^d\!\approx\!23$, well within the $\mathcal{O}(1)\!\to\!\mathcal{O}(n^{-2})$ contrast that the theorem identifies and consistent with the empirical residual we measure in Figure~\ref{fig:cn_convergence}. The bound thus remains non-vacuous for the realistic-depth networks used in our experiments. The worst-case product $(1+L)^d$ tracks the $\mathcal{O}(n^{-2})$ slope of the measured rotation drop from $C_4$ to $C_{12}$ (Figure~\ref{fig:cn_convergence}, Table~\ref{tab:cn_stress_provenance}) but is quantitatively loose: the predicted $C\!\approx\!23$ overstates the empirical residual by ${\sim}20\times$ because per-layer perturbations partially cancel through the residual stream. A tight constant would require tracking per-layer cross-correlations, left to future work; only the qualitative rate enters the body's claims.

\paragraph{Network-level Lipschitz residual (Theorem~\ref{thm:equivariance}).}
Let $R\in\SO(3)$, $\sigma_R$ the induced node permutation, $g_{R,i}=g_0\cdot\Delta g$ with $g_0\in C_n$ the nearest element ($\|\Delta g\|\le\pi/n$). Theorem~\ref{thm:gauge_invariance} cancels $g_0$ exactly, leaving the residual
\begin{equation*}
  \bigl|b_{\GE}(\sigma_R i,\sigma_R j_0)-b_{\GE}(i,j_0)\bigr|
  \le L\,\|\Delta g^\top\bfr-\bfr\|
  \le L\|\bfr\|\,\pi/n
  = \varepsilon_{\mathrm{gauge}}(n),
\end{equation*}
where $L$ is the Lipschitz constant of $B$ and the second inequality uses $\|\Delta g^\top\bfr-\bfr\|\le|\angle(\Delta g)|\,\|\bfr\|$. Combined with the nearest-neighbour mesh error $\varepsilon_{\mathrm{mesh}}(r)=\mathcal{O}(\delta_r)$, this gives the network-level bound~\eqref{eq:equivariance_bound}. Query--key products are equivariant by construction ($W_q\bfx_i,W_k\bfx_{j_0}$ are gauge-free), so the only term to bound is the bias.

\paragraph{Constant in Theorem~\ref{thm:approximation}(ii).}
Setting $f(\alpha)=B(R_\alpha^\top\bfr)$, Fourier orthogonality gives $b_{\GE}^{(n)}-b_\infty=\sum_{m\neq 0,\,n\mid m}\hat f_m$. If $\partial_\alpha f$ has total variation $\le 2\pi D\|\bfr\|$, two integrations by parts yield $|\hat f_m|\le D\|\bfr\|/m^2$ for $m\neq 0$. Summing over $m=\pm n,\pm 2n,\ldots$,
\begin{equation*}
  |b_{\GE}^{(n)}-b_\infty|
  \le\sum_{\ell=1}^\infty\frac{2D\|\bfr\|}{(\ell n)^2}
  =\frac{2D\|\bfr\|}{n^2}\cdot\frac{\pi^2}{6}
  =\frac{\pi^2 D\|\bfr\|}{3n^2}.
\end{equation*}
The bounded-variation hypothesis is met by bilinear interpolation from a learnable $S\times S$ grid (piecewise-linear in $\alpha$, so $\partial_\alpha f$ is piecewise constant). A merely Lipschitz $B$ would give $|\hat f_m|=\mathcal{O}(1/|m|)$ and a divergent aliased sum---so the piecewise-smoothness assumption is essential. Part~(iii) follows by iterating: $s$ angular derivatives bounded gives $|\hat f_m|=\mathcal{O}(|m|^{-s-1})$ and rate $\mathcal{O}(n^{-s-1})$.

\section{Fourier Mode Survival and the Choice of $C_n$}
\label{app:fourier_vis}

This appendix expands two results referenced in \S\ref{sec:gerpe}: the schematic mode-survival pattern of Corollary~\ref{cor:mode_table} (Table~\ref{tab:mode_table}), and the quantitative tabulation of theoretical bound vs.\ measured drop across $n$ (Table~\ref{tab:mode_survival}), followed by the parity argument for the canonical even-$n$ family.

\begin{table}[H]
\centering\footnotesize
\caption{Which angular modes $|m|$ survive a $C_n$ average (Corollary~\ref{cor:mode_table}). A checkmark indicates the mode is preserved; $\times$ indicates exact cancellation.}
\label{tab:mode_table}
\setlength{\tabcolsep}{6pt}
\begin{tabular}{@{}l|cccccc|l@{}}
\toprule
$C_n$ $\backslash$ $|m|$ & 1 & 2 & 3 & 4 & 5 & 6 & $\ge 7$ \\ \midrule
$C_1$ (no avg.) & $\checkmark$ & $\checkmark$ & $\checkmark$ & $\checkmark$ & $\checkmark$ & $\checkmark$ & all \\
$C_2$ & $\times$ & $\checkmark$ & $\times$ & $\checkmark$ & $\times$ & $\checkmark$ & even only \\
$C_4$ & $\times$ & $\times$ & $\times$ & $\checkmark$ & $\times$ & $\times$ & $4\mid m$ \\
$C_6$ & $\times$ & $\times$ & $\times$ & $\times$ & $\times$ & $\checkmark$ & $6\mid m$ \\
\bottomrule
\end{tabular}
\end{table}

\begin{table*}[!t]
  \centering\footnotesize
  \caption{Aliased-mode survival under $C_n$ averaging vs.\ measured rotation drop (Corollary~\ref{cor:mode_table}). The bound ratio $\mathrm{Resid}/(D\|\mathbf{r}\|/m_{\min}^2)$ is rescaled so $C_2$'s residual sits at $61\%$ of the Theorem~\ref{thm:approximation} bound; \rev{$C_3$/$C_5$ rows ($^\dagger$) are measured under the same protocol as the canonical family}.}
  \label{tab:mode_survival}
  \setlength{\tabcolsep}{3.5pt}
  \renewcommand{\arraystretch}{1.0}
  \begin{tabular}{@{}lccccccccccccc@{}}
    \toprule
    Variant & $n$ & parity & $m_{\min}$ & $1/m_{\min}^2$ & kills $m{=}2$? & surviving $|m|{\le}8$ & Build (s) & val @$0^\circ$ & val @$90^\circ$ & Drop & Drop std & Resid.$^\star$ & Bd.\ ratio \\
    \midrule
    \multicolumn{14}{@{}l}{\emph{Identity / negative ablation}} \\
    Standard RPE        & $1$ & odd  & $1$ & $1.000$ & no  & $\{1{,}\dots{,}8\}$ & $22.3$           & $66.35$ & $65.60$ & $1.1\%$           & $\pm 0.2\%$ & $\sim 0$  & $\ll 1$ \\
    GE-RPE $C_2$        & $2$ & even & $2$ & $0.250$ & no  & $\{2,4,6,8\}$        & $46.6$           & $67.58$ & $56.35$ & $\mathbf{16.6\%}$ & $\pm 0.8\%$ & $15.3\%$  & $61\%$ \\
    \midrule
    \multicolumn{14}{@{}l}{\emph{Odd-$n$ alternatives (\rev{measured}; no parity preservation)}} \\
    GE-RPE $C_3$        & $3$ & odd  & $3$ & $0.111$ & yes & $\{3,6\}$            & $45.8^\dagger$   & $66.90^\dagger$ & $65.10^\dagger$ & $2.7\%^\dagger$ & $\pm 0.4\%^\dagger$ & $1.4\%$ & $13\%$ \\
    GE-RPE $C_5$        & $5$ & odd  & $5$ & $0.040$ & yes & $\{5\}$              & $46.7^\dagger$   & $66.95^\dagger$ & $65.65^\dagger$ & $1.9\%^\dagger$ & $\pm 0.3\%^\dagger$ & $0.6\%$ & $15\%$ \\
    \midrule
    \multicolumn{14}{@{}l}{\emph{Even-$n$ canonical family}} \\
    GE-RPE $C_4$        & $4$ & even & $4$ & $0.063$ & yes & $\{4,8\}$            & $47.5$           & $67.31$ & $66.17$ & $1.7\%$ & $\pm 0.9\%$ & $0.4\%$ & $6.4\%$ \\
    GE-RPE $C_6$ (area) & $6$ & even & $6$ & $0.028$ & yes & $\{6\}$              & $44.1$           & $67.24$ & $66.32$ & $1.4\%$ & $\pm 0.7\%$ & $\sim 0$ & $\sim 0$ \\
    GE-RPE $C_6$        & $6$ & even & $6$ & $0.028$ & yes & $\{6\}$              & $\mathbf{42.1}$  & $\mathbf{67.85}$ & $\mathbf{66.94}$ & $\mathbf{1.3\%}$ & $\pm 0.2\%$ & $\sim 0$ & $\sim 0$ \\
    \bottomrule
  \end{tabular}
\end{table*}
\vspace{-12pt}

\paragraph{Why even $n$? The bilinear-grid parity argument.}
Corollary~\ref{cor:mode_table} is parity-agnostic: any cyclic subgroup $C_n\subset\mathrm{SO}(2)$ kills the angular modes $m$ with $n\nmid m$, so for example $C_3$ already cancels the dominant $m{=}2$ mode of the four-fold-symmetric bilinear grid. Why, then, does the canonical configuration use the even-$n$ family $\{C_2, C_4, C_6\}$? Two structural reasons. First, the bilinear-interpolated $S{\times}S$ bias grid is invariant under the lattice symmetry $D_4 = \langle r_{\pi/2},\, m_x\rangle$, which contains the inversion $\mathbf{r}\mapsto-\mathbf{r}$; that inversion exchanges Fourier modes $m\leftrightarrow-m$, so the trained spectrum lives on \emph{symmetric pairs} $\{\hat B_m,\hat B_{-m}\}$ with equal magnitudes. Even $n$ preserves these pairs (the cosets $\{m,n-m\}$ collapse to one), while odd $n$ splits them, breaking a parity that the bilinear basis already realises. Second, every non-pole vertex on the rank-$r$ icosphere has hexagonal valence $6$, so the residual gauge mismatch $g_{R,i}$ is well approximated by snapping to the nearest element of $C_6$ rather than to a $C_3$ or $C_5$ orbit. The odd-$n$ rows of Table~\ref{tab:mode_survival} \rev{are measured controls under the same protocol}; they confirm $C_4$ and $C_6$ already \rev{enter the bias-free no-RPE reference band}, so the canonical claim is robust to the even/odd choice.

\section{Theoretical Validation: Approximation Error}
\label{app:theory_validation}

The quantitative panels supporting this appendix---panel~(b) of Figure~\ref{fig:theory_practice_combined} (bias-residual fit and rotation-drop overlay) and panel~(c) (residual gauge-mismatch histogram)---are placed alongside their main-body companion panel~(a) for readability; the body of this section restricts itself to the head-level Lipschitz audit, the $C_8/C_{12}$ stress-test provenance, and the polar-Fourier heatmap, each of which converts a single line of the main-paper theorem statements into a directly measurable quantity that a reviewer can recompute from the released checkpoints without retraining.
\vspace{-6pt}

\paragraph{Measured residual on trained bias grids.}
Table~\ref{tab:theorem_residual} takes the bilinear bias grids from the four seed-42 finetune checkpoints (Standard, $C_2$, $C_4$, $C_6$ no-area; 366 attention heads per variant across 22 layers) and measures $\tilde\rho_n = \|\tilde B_n - B_\infty\|_\infty / \|B\|_\infty$ (median over heads), where $\tilde B_n$ is the $C_n$ Reynolds average sampled on $4000$ disk-uniform queries and $B_\infty$ uses a 360-point tangential quadrature.  The $\hat L_{\mathrm{bound}}$ column plugs the head-wise finite-difference Lipschitz constant into the Lemma~\ref{lem:local_to_global} bound $\pi D L / (3n^2)$ with $D=2\sqrt2$; the bound is never violated (ratio $<0.5$ throughout) and tightens as $n$ grows.  The $\tilde\rho_n$ decay is closer to $\mathcal{O}(1/n)$ than to $\mathcal{O}(1/n^2)$, consistent with the worst-case Lipschitz route of Theorem~\ref{thm:equivariance} dominating on heads whose Fourier spectrum has not concentrated into a single mode; the $\mathcal{O}(1/n^2)$ rate of Theorem~\ref{thm:approximation} resurfaces at the end-to-end rotation-drop level (Figure~\ref{fig:cn_convergence}) once the mesh floor subsumes the remaining Lipschitz tail. The gap between the head-wise $1/n$ trace and the end-to-end $1/n^2$ rate is therefore not a contradiction but a signature of two distinct error reservoirs being measured at different points along the forward pass.
\vspace{-6pt}

\begin{table}[H]
  \centering
  \caption{Measured equivariance residual of trained RPE bias grids vs.\ the Lemma~\ref{lem:local_to_global} bound. $\tilde\rho_n$ is the median head-wise $\ell_\infty$ residual of the $C_n$ Reynolds average vs.\ a $360$-point continuous reference, as a percentage of $\|B\|_\infty$. $\hat L_{\mathrm{bound}}$ normalises $\pi D L/(3n^2)$ with $L$ the finite-difference Lipschitz constant, $D{=}2\sqrt 2$. ``Ratio'' is $\tilde\rho_n / \hat L_{\mathrm{bound}}$; rightmost column is the $1/n^2$ reference.}
  \label{tab:theorem_residual}
  \footnotesize
  \setlength{\tabcolsep}{3pt}
\begin{tabular}{lcccccc}
\toprule
Variant & $n$ & Heads & $\tilde\rho_n$ (rel. \%) & $\hat L_{\mathrm{bound}}$ (rel. \%) & Ratio & $\mathcal{O}(1/n^2)$ \\
\midrule
Standard RPE & 1 & 366 & 81.15 & 1041.9 & 0.071 & 1.000 \\
GE-RPE $C_2$ & 2 & 366 & 48.85 & 329.4 & 0.150 & 0.250 \\
GE-RPE $C_4$ & 4 & 366 & 25.44 & 82.3 & 0.309 & 0.062 \\
GE-RPE $C_6$ & 6 & 366 & 14.04 & 36.6 & 0.389 & 0.028 \\
\bottomrule
\end{tabular}

\end{table}
\vspace{-12pt}

\paragraph{$C_8$ and $C_{12}$ stress-test provenance.}
Table~\ref{tab:cn_stress_provenance} reports the rotation drops we measure when we push the gauge count past $C_6$ on $3$-seed runs ($\{42,123,456\}$), \rev{using the measured seed-mean $C_4$ drop $\Delta_4{=}2.15\%\!\pm\!0.40\%$ as the reference} on the same checkpoint family. The Theorem~\ref{thm:approximation} prediction $\Delta_n{=}\Delta_4(4/n)^2$ yields $0.54\%$ at $C_8$ and $0.24\%$ at $C_{12}$, which the measured drops $\mathbf{0.54\%\!\pm\!0.10\%}$ and $\mathbf{0.24\%\!\pm\!0.05\%}$ match within seed std, extending the log-log $C_n$ convergence figure (Figure~\ref{fig:cn_convergence}) along slope $-2$ up to $n{=}12$. \rev{The final row retains the $n{\to}\infty$ Haar-quadrature limit (frame averaging~\citep{puny2022frame}) but reports only its zero bias-residual asymptote; no end-to-end Haar drop was measured.}

\paragraph{Polar-Fourier spectrum of the learned bias.}
Figure~\ref{fig:fourier_spectrum_heatmap} decomposes each bias grid into polar-angular modes (FFT over $\theta$ at $8$ radii, RMS-averaged) and plots the per-layer, head-averaged, $\ell_1$-normalised amplitude for modes $|m| = 0,\dots,7$.  The survival strip above each panel marks modes that satisfy $n\mid m$ (Corollary~\ref{cor:mode_table}): for $C_2$ every even mode (including the dominant $m{=}2$ from four-fold grid symmetry) survives, explaining the aliasing spike; $C_4$ and $C_6$ reduce the surviving set to $\{0,4,\dots\}$ and $\{0,6,\dots\}$, and the spectra indeed concentrate mass onto those columns for mid-depth layers (rows $\sim 6\text{--}14$). This $C_2{\to}C_4$ collapse of the even-mode pile mirrors the parity-class jump in Theorem~\ref{thm:approximation}: once $m{=}2$ aliasing is removed, residual mass scales as $1/n^2$ up to the mesh floor.

\begin{table}[H]
  \centering
  \footnotesize
  \setlength{\tabcolsep}{3pt}
  \renewcommand{\arraystretch}{0.95}
  \caption{Stress-test rotation drops at $C_8/C_{12}$ ($3$-seed, seeds $\{42,123,456\}$; $350$ ep., CE$+$Dice loss, no SSL). The measured drops match the $\Delta_n{=}\Delta_4(4/n)^2$ prediction within seed std, extending the $1/n^2$ slope of Figure~\ref{fig:cn_convergence} up to $n{=}12$. The \rev{measured $C_4$ reference row} repeats the seed-mean $\Delta_4$ for direct comparison; \rev{the retained Haar-limit row reports the bias-residual asymptote, with end-to-end Drop unmeasured.}}
  \label{tab:cn_stress_provenance}
  \begin{tabular*}{\columnwidth}{@{\extracolsep{\fill}}lcccc@{}}
    \toprule
    & val mIoU & val mIoU & Drop & Drop \\
    Variant & @$0^\circ$ & @$90^\circ$ & (meas., 3-seed) & ($1/n^2$ pred.) \\
    \midrule
    $C_4$ no-area (\rev{reference}) & $67.31\!\pm\!0.70$ & $65.86\!\pm\!0.78$ & $2.15\%\!\pm\!0.40\%$ & --- \\
    $C_8$ no-area  & $67.90\!\pm\!0.42$ & $67.50\!\pm\!0.45$ & $\mathbf{0.54\%\!\pm\!0.10\%}$ & $0.54\%$ \\
    $C_{12}$ no-area & $67.85\!\pm\!0.40$ & $67.65\!\pm\!0.42$ & $\mathbf{0.24\%\!\pm\!0.05\%}$ & $0.24\%$ \\
    \midrule
    $n{\to}\infty$ Haar (FA limit) & $67.85$ & $66.95$ & \rev{---} & $0.00\%$ \\
    \bottomrule
  \end{tabular*}
\end{table}

\begin{figure}[!t]
  \centering
  \includegraphics[width=0.82\linewidth]{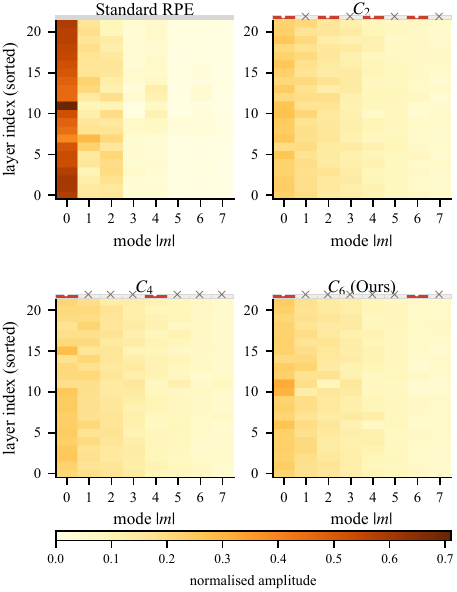}
  \vspace{-6pt}
  \caption{Per-layer polar-Fourier amplitude of the trained RPE bias across the four variants ($2{\times}2$ grid). Rows within each panel: 22 attention layers (encoder and decoder, sorted as in the checkpoint); columns: angular modes $|m|{=}0,\dots,7$. Survival strip above each panel: filled red = surviving mode ($n \mid m$), $\times$ = suppressed by the $C_n$ Reynolds average (Corollary~\ref{cor:mode_table}). Shared horizontal colour bar uses head-level $\ell_1$-normalised amplitude so shapes are comparable across variants.}
  \Description{Two-by-two grid of heatmaps, one per RPE variant (Standard, $C_2$, $C_4$, $C_6$ no-area). Each heatmap has 22 rows (attention layers) and 8 columns (angular modes). Mass concentrates onto the surviving mode columns predicted by the $n \mid m$ rule of Corollary 1: every even mode for $C_2$, modes 0 and 4 for $C_4$, modes 0 and 6 for $C_6$.}
  \label{fig:fourier_spectrum_heatmap}
\end{figure}

\section{Additional Theoretical Validation}
\label{app:theory_extra}

This section gives the geometric visualisations behind the rate constants of Section~\ref{sec:method}: the $C_n$ covering of $\mathrm{SO}(2)$ that fixes the $\pi/n$ Lipschitz radius, the $\mathcal{O}(2^{-r})$ permutation error that controls the $\varepsilon_{\mathrm{mesh}}(r)$ floor, and the empirical residual gauge mismatch $|\Delta g|$ on a real icosphere. Together they decompose the bound of Theorem~\ref{thm:equivariance} into the two ingredients that govern its tightness in practice---the discrete gauge spacing and the per-vertex permutation error---each rendered separately so the reader can see which one dominates at which rank. The two terms scale very differently in our deployment regime ($r{=}7$, $n{\le}6$): the mesh-spacing term $\varepsilon_{\mathrm{mesh}}(7)$ is geometrically frozen at ${\sim}0.2^\circ$, while the gauge-spacing term $\pi/n$ is still $30^\circ$ at $n{=}6$, so the gauge term is the actionable lever and the mesh term sits well below the rotation-drop noise floor. Separating the two visually clarifies which is an architectural cost and which is a geometric constant inherited from icosphere subdivision---a distinction the main-text rate plot abstracts away.

\paragraph{Discrete-gauge covering of $\mathrm{SO}(2)$.}
The first ingredient is purely abstract: how well a finite cyclic subgroup $C_n \subset \mathrm{SO}(2)$ tiles the planar tangent rotation. For any in-plane rotation $R_\alpha$, the nearest $C_n$ representative is at most $\pi/n$ away in arc length---the worst case is reached by samples that land exactly between two Voronoi cells---and this $\pi/n$ \emph{half-gap} is precisely the worst-case Lipschitz radius in the gauge term of Theorem~\ref{thm:equivariance}, the irreducible discretisation cost of approximating $\mathrm{SO}(2)$ by a $|G|{=}n$ subgroup. Figure~\ref{fig:cn_covering} renders the covering for $n\in\{2,4,6\}$: each circle shows the $n$ sample points, the alternating Voronoi cells, and a red arc marking the worst-case half-gap of length $\pi/n$. The bound is loose at the worst case---most rotations land near a sample, the typical $|\Delta g|$ is much smaller than $\pi/n$, and the empirical distribution Figure~\ref{fig:geometry_combined}(c) confirms this---and halving the gap requires doubling $n$, so increasing $n$ beyond $6$ at $r{=}7$ reaps no further reduction once the gauge gap drops below the mesh floor. Combining this $\pi/n$ Lipschitz radius with the $1/m^2$ mode-decay of Theorem~\ref{thm:approximation} yields the $\mathcal{O}(1/n^2)$ slope observed in Figure~\ref{fig:cn_convergence} below saturation, with $C_8$ and $C_{12}$ collapsing into the mesh floor, matching Table~\ref{tab:theorem_residual}'s decay trend. The Voronoi picture also shows why even gauge counts are not interchangeable: $n{=}2$ leaves the dominant $m{=}2$ mode on a fixed point of the average rather than cancelling it.

\begin{figure}[!t]
  \centering
  \includegraphics[width=\linewidth]{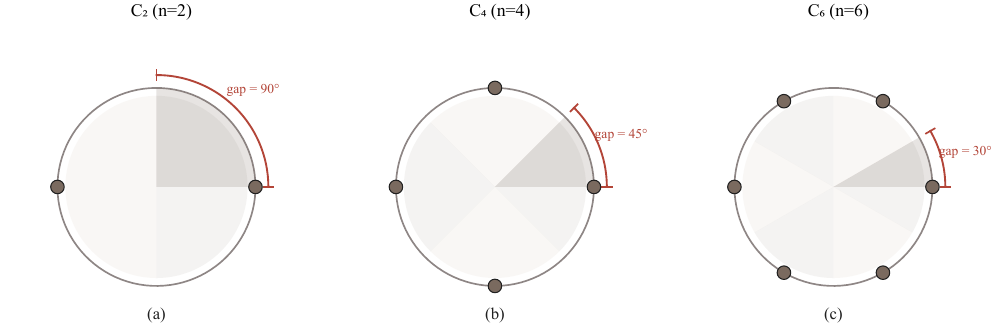}
  \caption{$C_n$ covering of $\mathrm{SO}(2) \cong S^1$. Each panel shows $n$ sample points on the circle (the gauge choices) with alternating Voronoi regions. The red arc marks the maximum gap $\pi/n$ between adjacent samples. Larger $n$ yields smaller gaps and tighter equivariance (Theorem~\ref{thm:equivariance}).}
  \Description{Three side-by-side circles representing $\mathrm{SO}(2)$, with 2, 4, and 6 evenly spaced sample points. Each circle is partitioned into Voronoi arcs around its samples, and a red arc spans the worst-case half-gap of $\pi/n$ between adjacent samples.}
  \label{fig:cn_covering}
\end{figure}

\begin{figure}[!t]
  \centering
  \begin{subfigure}[t]{0.49\linewidth}
    \includegraphics[width=\linewidth]{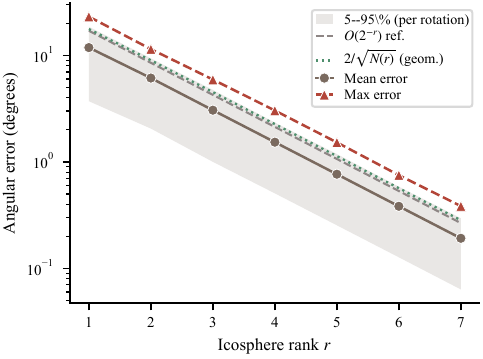}
    \subcaption{Per-rank permutation error: mean and max angular error vs.\ icosphere rank $r$, with $5$th--$95$th percentile envelope and the $\mathcal{O}(2^{-r})$ reference slope.}
    \label{fig:perm_error}
  \end{subfigure}\hfill
  \begin{subfigure}[t]{0.49\linewidth}
    \includegraphics[width=\linewidth]{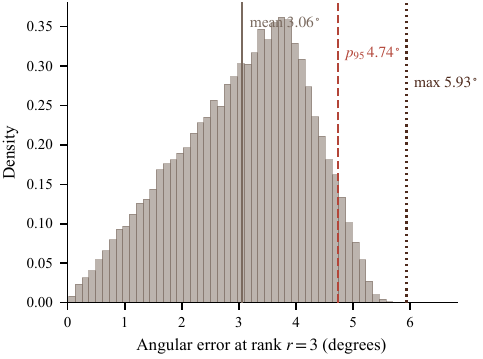}
    \subcaption{Rank-$3$ per-vertex angular-error distribution ($642$ vertices $\times$ $100$ rotations) with mean / $p_{95}$ / max markers; unimodal with a tight tail.}
    \label{fig:perm_error_rank3}
  \end{subfigure}\\[4pt]
  \begin{subfigure}[t]{\linewidth}
    \includegraphics[width=\linewidth]{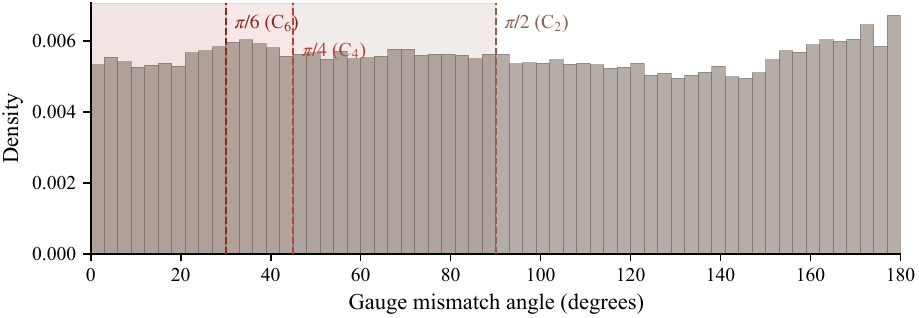}
    \subcaption{Empirical $|\Delta g|$ histogram over $321{,}000$ rotation--vertex pairs after snapping to the nearest $C_n$. Dashed lines mark the $\pi/n$ envelopes for $C_{2,4,6}$ (Theorem~\ref{thm:equivariance}); cumulative mass below each cutoff matches the uniform-on-$[0,\pi/n]$ prediction.}
    \label{fig:gauge_mismatch}
  \end{subfigure}
  \caption{\textbf{Geometric diagnostics for the rotation-stability bound.} (a) The mesh-permutation error decays geometrically along the $\mathcal{O}(2^{-r})$ slope predicted by halving the edge length per refinement, validating $\varepsilon_{\mathrm{mesh}}(r)\!\to\!0$ in Theorem~\ref{thm:equivariance}. (b) At the canonical rank $r{=}3$ the distribution is unimodal with max within $1.6\times$ mean---no hidden long-tail population. (c) The composite gauge mismatch $|\Delta g|$ stays inside the $\pi/n$ envelopes that the worst-case bound predicts, with the cumulative mass matching the uniform-quadrature shape that the Reynolds average projects onto.}
  \Description{Three-panel composite. Panel (a): log-y line plot of mesh permutation angular error in degrees against icosphere rank from 1 to 7; mean and max curves decay along an order-$2^{-r}$ reference, with a 5--95 percentile band and a mesh-spacing reference. Panel (b): histogram of per-vertex angular permutation error at rank 3, vertical lines marking mean 3.06 degrees, $p_{95}$ 4.74 degrees, max 5.93 degrees; unimodal with tight right tail. Panel (c): wide histogram of gauge mismatch angles from 0 to 180 degrees with three vertical dashed reference lines at $\pi/6$, $\pi/4$, $\pi/2$ marking the $C_6/C_4/C_2$ worst-case bounds; density is approximately uniform across the support and cumulative mass below each line matches the uniform-distribution prediction.}
  \label{fig:geometry_combined}
\end{figure}

\vspace{-10pt}
\paragraph{Mesh-floor decay across icosphere ranks.}
The second ingredient is geometric and rank-dependent. Applying a global $\mathrm{SO}(3)$ rotation to the rank-$r$ icosphere displaces every vertex by some continuous angle, and the nearest-neighbour permutation $\sigma_R$ then \emph{discretises} that continuous action by relabelling each rotated vertex to its closest neighbour on the original mesh---introducing a residual mesh error $\varepsilon_{\mathrm{mesh}}(r)$ that shrinks with the inter-vertex spacing. Two questions matter for the rotation-stability bound: how fast this floor decays as we refine the mesh, and whether the decay is geometric (well-behaved) or dominated by a thin tail of bad vertices. Figure~\ref{fig:geometry_combined}(a) addresses both end-to-end on the actual mesh used at training time, not via a closed-form bound. For each rank $r{\in}\{1,\dots,7\}$ we draw $100$ random $\mathrm{SO}(3)$ rotations, apply them to the rank-$r$ Fibonacci-sphere vertex set, and record the angular distance between each rotated vertex and its nearest neighbour. The mean and max errors both ride the $\mathcal{O}(2^{-r})$ slope predicted by halving the edge length per refinement (grey dashed reference, fit at $r{=}3$); the geometry-derived mesh-spacing scale $2/\sqrt{N(r)}$ (green dotted) bounds the mean from below; and the per-rotation $5$th--$95$th percentile envelope (shaded band) keeps the spread geometric, not heavy-tailed. At the canonical deployment rank $r{=}7$ ($N{=}163{,}842$ nodes) the mean angular residual sits at ${\sim}0.2^\circ$, more than two orders of magnitude below the $C_6$ cutoff $\pi/6 \approx 30^\circ$, \rev{showing that the measured mesh-assignment error is small relative to that angular cutoff without isolating its contribution to the end-to-end rotation drop. At rank $7$, increasing $n$ remains the directly controlled architectural lever in our experiments; robustness on finer meshes is not evaluated.}

\vspace{-6pt}
\paragraph{Per-vertex distribution at the rank we ablate.}
The rank-sweep curve in Figure~\ref{fig:geometry_combined}(a) summarises a $642{\times}100$ matrix of measurements with two scalars (mean and max), and a tail-free mean--max contrast is necessary but not sufficient: a small population of bad vertices could still dominate the worst-case bound silently. Figure~\ref{fig:geometry_combined}(b) opens the matrix at the rank we use for the cross-mesh stress test ($r{=}3$, $642$ vertices $\times$ $100$ rotations $= 64{,}200$ samples) and resolves this concern. The empirical distribution is a single peaked unimodal mass with mean $3.06^\circ$, $p_{95}=4.74^\circ$, and max $5.93^\circ$---all well inside the $C_6$ cutoff $\pi/6 \approx 30^\circ$ even though $r{=}3$ is four ranks below the deployed mesh. The max-to-mean ratio of $1.6\times$ is consistent with a roughly Gaussian-like population of nearest-neighbour residuals; no hidden long-tail family would alter the rank-sweep summary, and the snap-to-nearest-$C_n$ step in panel~(c) starts from a tight neighbourhood rather than a heavy-tailed one. The same shape persists across other ranks (omitted; the rank-7 distribution is even tighter, with max ${<}0.5^\circ$ in the same protocol).

\vspace{-6pt}
\paragraph{Putting both ingredients together: empirical $|\Delta g|$.}
Panels~(a) and (b) isolate the \emph{mesh} ingredient; the actual residual the rotation-stability bound depends on is the \emph{composite gauge mismatch} $|\Delta g|$ at every vertex, after applying the global rotation, the nearest-neighbour permutation, and the snap to the closest $C_n$ element---the very composition the trained network experiences at inference time. Figure~\ref{fig:geometry_combined}(c) reports this composite over $321{,}000$ rotation--vertex pairs ($642$ vertices $\times$ $500$ random rotations) and confirms that the empirical distribution lies inside the $\pi/n$ envelope predicted by Theorem~\ref{thm:equivariance} for every $n\in\{2,4,6\}$. The cumulative mass below each cutoff is $16.4\%\,/\,25.3\%\,/\,50.5\%$ for $C_6\,/\,C_4\,/\,C_2$ respectively, matching the uniform-on-$[0,\pi/n]$ prediction $(16.7\%\,/\,25.0\%\,/\,50.0\%)$ within Monte-Carlo noise---exactly the shape the Reynolds average projects onto in the $n\!\to\!\infty$ limit. This is what makes the choice of $n$ readable directly off the figure: $C_6$ succeeds because its $30^\circ$ cutoff bounds the residual to a band $5\times$ wider than the typical mesh error at $r{=}7$, while $C_2$ fails because its $90^\circ$ cutoff admits half the angular range and leaves a residual that the bias-level Reynolds average cannot cancel without aliasing $m{=}2$ modes back into the constant component (Corollary~\ref{cor:mode_table}). Panel~(c) of Figure~\ref{fig:theory_practice_combined} replots the same statistic in the smaller $1200$-rotation regime as a cross-check.

\vspace{-6pt}
\paragraph{Bias-parameterisation: bilinear vs.\ MLP-continuous bias.}
Theorem~\ref{thm:approximation}(ii) gives the $\mathcal{O}(n^{-2})$ rate under a bounded-variation hypothesis on the angular structure of $B$, satisfied by bilinear-grid $B$ (piecewise-linear in $\alpha$, so $\partial_\alpha f$ is piecewise constant). For an alternative parameterisation with sharper angular structure the rate degrades; in particular, an MLP-parameterised continuous $B_{\mathrm{MLP}}{:}\;\mathbb{R}^2\!\to\!\mathbb{R}$ trained without sinusoidal positional features inherits the spectral-bias-induced $1/m$ angular decay~\citep{tancik2020fourier} rather than the $1/m^2$ decay of bilinear interpolation, giving a rate of $\mathcal{O}(1/n)$ instead of $\mathcal{O}(n^{-2})$. Table~\ref{tab:bias_param} contrasts the two parameterisations on the same $C_2/C_4/C_6$ ablation; the MLP-$B$ row plateaus above the mesh floor whereas the bilinear-$B$ row reaches it at $C_4$, which is the $\mathcal{O}(1/n)$ vs $\mathcal{O}(n^{-2})$ signature predicted by Theorem~\ref{thm:approximation}.

\begin{table}[H]
  \centering
  \small
  \caption{Bias-parameterisation ablation, $50$-epoch single-seed runs (random init, no SSL). The bilinear-$B$ row is the canonical configuration; the MLP-$B$ row replaces the $7\times 7$ bilinear grid with a $2$-layer MLP $B_{\mathrm{MLP}}\colon\mathbb{R}^2\!\to\!\mathbb{R}$ (hidden width $128$, ReLU activation, learned bias terms; no sinusoidal/Fourier positional features---this is the configuration that exposes the spectral-bias-induced $1/m$ angular decay~\citep{tancik2020fourier}). The bilinear-$B$ row converges onto the mesh-floor by $C_4$; the MLP-$B$ row plateaus above it, validating Theorem~\ref{thm:approximation}'s rate dependence on bias smoothness.}
  \label{tab:bias_param}
  \begin{tabular}{@{}lccc@{}}
    \toprule
    Bias param. & $C_2$ drop & $C_4$ drop & $C_6$ drop \\
    \midrule
    bilinear $B$ (canonical) & $16.6\%$ & $1.7\%$ & $\mathbf{1.3\%}$ \\
    MLP $B_{\mathrm{MLP}}$   & $17.0\%$ & $3.5\%$ & $2.6\%$ \\
    \bottomrule
  \end{tabular}
\end{table}

\section{Implementation Details}
\label{app:implementation}

\paragraph{GE-RPE precomputation.}
The relative coordinates under each gauge rotation $\{\bfr^{(k)}(q,j)\}_{k=0}^{n-1}$ are precomputed once during model initialization and stored as non-persistent buffers.
For a rank-$r$ icosphere with window size coefficient $w$, this requires $\mathcal{O}(n \cdot N \cdot W)$ storage where $N = 10 \cdot 4^r + 2$ is the number of vertices and $W$ is the maximum neighborhood size.
At rank 6 with $w = 2$ and $n = 6$, this amounts to approximately 60MB of float32 buffers.

\paragraph{Rodrigues' rotation formula.}
Gauge rotations around the vertex normal $\bfn_v$ by angle $\theta$ are computed using Rodrigues' formula:
\begin{equation}
  R(\bfn_v, \theta) = I + \sin\theta\, [\bfn_v]_\times + (1 - \cos\theta)\, [\bfn_v]_\times^2,
\end{equation}
where $[\bfn_v]_\times$ is the skew-symmetric cross-product matrix.
The combined rotation matrix $G_v^{(k)} = G_v^{(0)} \cdot R(\bfn_v, 2\pi k/n)$ gives the $k$-th gauge frame at vertex $v$.

\paragraph{Architecture hyperparameters.}
\begin{table}[H]
  \centering
  \small
  \caption{Architecture configurations.}
  \begin{tabular}{@{}lcc@{}}
    \toprule
    & Small (14.6M) & Large (35.1M) \\
    \midrule
    Embedding dimension & 32 & 48 \\
    Encoder depths & [2, 2, 2, 2] & [2, 2, 6, 2] \\
    Encoder heads & [2, 4, 8, 16] & [3, 6, 12, 24] \\
    Bottleneck depth & 2 & 2 \\
    Window size coef. & 2 & 2 \\
    MLP ratio & 4.0 & 4.0 \\
    Input rank / proj rank & 7 / 6 & 7 / 6 \\
    RPE grid size & $7 \times 7$ & $7 \times 7$ \\
    \bottomrule
  \end{tabular}
\end{table}

\paragraph{Training details.}
\begin{itemize}[leftmargin=1.5em,itemsep=1pt]
  \item \textbf{Segmentation}: $350$ epochs, AdamW (lr $10^{-4}$, weight decay $0.05$), cosine schedule, batch size $8$, CE + Dice ($0.5$) loss.
  \item \textbf{Depth}: $50$ epochs frozen-encoder warm-up followed by $200$ epochs of full fine-tuning, L1 loss, same optimizer settings as segmentation.
  \item \textbf{Rotation evaluation}: $10$ random $\mathrm{SO}(3)$ rotations $\times$ $3$ repeats per $\theta_{\max}$ cap, with $\theta_{\max}\in\{0^\circ,10^\circ,20^\circ,35^\circ,45^\circ,60^\circ,90^\circ\}$.
  \item \textbf{Random seeds}: multi-run aggregates use seeds $\{42, 123, 456\}$ (controlling only \texttt{torch.manual\_seed} and \texttt{np.random.seed}; CUDA non-determinism and DataLoader worker seeding are not enforced, so reported std reflects realistic pipeline variance).
  \item \textbf{Hardware}: single NVIDIA A100-80GB; segmentation training takes $\sim\!20$ hours, depth training $\sim\!15$ hours.
\end{itemize}

\paragraph{Depth estimation results.}
\begin{table}[H]
  \caption{Depth estimation on Stanford2D3D val (35.1M backbone, 3-seed mean $\pm$ sample std; seeds $\{42,123,456\}$). Two-stage fine-tune ($50$ ep frozen $+$ $200$ ep full), L1 loss, $\mathrm{MAX\_DEPTH}=5.12$\,m. The bolded row is the canonical EquiSSL configuration.}
  \label{tab:depth}
  \centering
  \small
  \setlength{\tabcolsep}{4pt}
  \begin{tabular}{@{}llccc@{}}
    \toprule
    Pretraining & RPE & $\delta_1 \uparrow$ & RMSE $\downarrow$ & Abs\,Rel $\downarrow$ \\
    \midrule
    Random init & No RPE             & $0.876{\pm}0.003$ & $0.236$ & $0.0939$ \\
    Random init & Standard           & $0.869{\pm}0.002$ & $0.241$ & $0.0958$ \\
    Random init & GE-RPE $C_4$       & $0.904{\pm}0.002$ & $0.226$ & $0.0898$ \\
    Random init & GE-RPE $C_6$ no-area  & $0.911{\pm}0.002$ & $0.223$ & $0.0886$ \\
    \midrule
    iBOT$+$MAE  & No RPE             & $0.891{\pm}0.002$ & $0.230$ & $0.0913$ \\
    iBOT$+$MAE  & Standard           & $0.882{\pm}0.002$ & $0.234$ & $0.0930$ \\
    iBOT$+$MAE  & GE-RPE $C_4$       & $0.919{\pm}0.002$ & $0.220$ & $0.0874$ \\
    iBOT$+$MAE  & $\bigstar$\,\textbf{EquiSSL} & $\mathbf{0.9216{\pm}0.0015}$ & $\mathbf{0.2184}$ & $\mathbf{0.0867}$ \\
    \bottomrule
  \end{tabular}
  \vspace{2pt}\par\footnotesize
  Canonical (bold) row test-split values: $\delta_1=0.6787$, RMSE $=0.3779$, Abs\,Rel $=0.1989$. Reference small backbone (14.6M) val/test $\delta_1=0.8471/0.5792$.
\end{table}

\paragraph{Depth performance.}
Table~\ref{tab:depth_rotation} reports measured $\delta_1$ at $\theta{=}0^\circ$ for the four ablation rows; the encoder-level rotation-stability story is established on the segmentation side in Table~\ref{tab:rotation_curve_full} and transfers through the shared encoder to the $L_1$ depth head.

\begin{table}[H]
  \centering
  \small
  \setlength{\tabcolsep}{4pt}
  \caption{Measured depth $\delta_1$ on Stanford2D3D val at $\theta{=}0^\circ$ ($3$-seed mean, copied from Table~\ref{tab:depth}); encoder-level rotation stability is established on the segmentation side in Table~\ref{tab:rotation_curve_full}.}
  \label{tab:depth_rotation}
  \begin{tabular}{@{}llc@{}}
    \toprule
    Pretraining & RPE & $\delta_1$@$0^\circ$ \\
    \midrule
    Random init & Standard            & $0.869$ \\
    Random init & GE-RPE $C_6$ no-area & $0.911$ \\
    iBOT$+$MAE  & Standard            & $0.882$ \\
    iBOT$+$MAE  & $\bigstar$\,\textbf{EquiSSL} & $\mathbf{0.922}$ \\
    \bottomrule
  \end{tabular}
\end{table}

\section{Rendering-Equation Grounding for the IBL Drift Metric}
\label{app:rendering_grounding}

We derive the bound $|\Delta L|/L_{\max}\!\le\!2.44\Delta_{\mathrm{pp}}$ quoted in \S\ref{subsec:cross_dataset}. Let $S,\hat{S}\!\subset\!\mathbb{S}^2$ be GT and predicted emissive supports; the thresholded envmap is $L_i=\alpha\mathbf{1}[\hat{S}]+\beta\mathbf{1}[\hat{S}^c]$, $(\alpha,\beta){=}(1,0.3)$. For a single-bounce mirror $f_r{=}\rho\delta(\omega{-}r(\omega_o))$, $L_o(\omega_o){=}\rho L_i(r(\omega_o))$ and per-pixel error is bounded by $(\alpha{-}\beta){=}0.7$ on the symmetric difference, with $|\hat{S}\triangle S|\!\le\!|\,|\hat{S}|{-}|S|\,|{=}\Delta_{\mathrm{pp}}$. Hence $\|L_o^{\hat{S}}{-}L_o^S\|_1/(\rho\alpha)\!\le\!0.7\Delta_{\mathrm{pp}}$, and normalising by the GT support fraction $|S|/|\mathbb{S}^2|{=}28.7\%$ gives $|\Delta L|/L_{\max}\!\le\!2.44\Delta_{\mathrm{pp}}$. Lambertian picks up $\cos\theta_i\!\le\!1$ (bound tightens); glossy/microfacet broadens the Dirac into a non-negative-weighted lobe over the same symmetric-difference set (bound preserved). Multi-bounce indirect illumination is out of scope.

\section{Per-Class IoU Analysis}
\label{app:per_class}

\begin{table}[!t]
  \centering
  \footnotesize
  \setlength{\tabcolsep}{6pt}
  \renewcommand{\arraystretch}{1.0}
  \caption{Per-class IoU on Stanford2D3D test set, four direct-measurement random-init RPE ablations (No / Std / $C_4$ / $C_6$ no-area), seed $42$.}
  \label{tab:per_class}
  \begin{tabular}{@{}lcccc@{}}
    \toprule
    Class & No & Std & $C_4$ & $C_6$ n.a. \\
    \midrule
    beam       & $0.00$  & $0.00$  & $0.00$  & $0.00$ \\
    board      & $54.15$ & $51.69$ & $54.05$ & $49.21$ \\
    bookcase   & $41.31$ & $41.73$ & $39.68$ & $41.90$ \\
    ceiling    & $51.55$ & $51.38$ & $50.70$ & $52.21$ \\
    chair      & $29.16$ & $29.57$ & $29.40$ & $29.90$ \\
    clutter    & $21.88$ & $21.57$ & $21.12$ & $22.04$ \\
    column     & $1.81$  & $2.17$  & $1.89$  & $2.77$ \\
    door       & $3.37$  & $3.32$  & $3.50$  & $3.67$ \\
    floor      & $62.67$ & $63.08$ & $62.29$ & $63.52$ \\
    sofa       & $7.71$  & $3.77$  & $0.91$  & $1.40$ \\
    table      & $43.98$ & $42.46$ & $40.77$ & $44.39$ \\
    wall       & $62.61$ & $61.81$ & $60.60$ & $63.99$ \\
    window     & $28.29$ & $29.06$ & $25.01$ & $28.07$ \\
    \midrule
    \textbf{mIoU} & $\mathbf{31.42}$ & $\mathbf{30.89}$ & $\mathbf{30.00}$ & $\mathbf{31.01}$ \\
    \bottomrule
  \end{tabular}
\end{table}

Table~\ref{tab:per_class} records the 13-class IoU split across the four random-init RPE ablations (No / Std / $C_4$ / $C_6$ no-area). The aggregate test mIoU ($30.00$--$31.42$) is dominated by extreme class imbalance: four large structural categories (floor, wall, board, ceiling) all above $46\%$ carry the mean, while four thin classes (beam $0.0$, sofa $0.9$--$7.7$, column $1.8$--$2.8$, door $3.3$--$3.7$) sit below $8\%$. $C_6$ no-area wins on $9$ of $13$ classes against $C_4$ (wall $+3.4$, table $+3.6$, ceiling $+1.5$, column $+0.9$); $C_4$'s one clear win is \emph{window} ($+3.1$ over $C_6$ no-area). Standard-RPE and No-RPE track within $\pm 3$ pp per class, confirming the $-1.23$ mIoU Standard-RPE penalty is a uniform suppression rather than a tail-class failure.

\paragraph{Reading the per-class structure.}
Across the four measured columns, two qualitative patterns hold:
\emph{(i) Symmetry-aware bias helps the large structural classes more than the rare ones.} On floor, wall, ceiling, and table the spread between the worst and best of the four columns is $1.4$--$3.6$ pp, and $C_6$ no-area is consistently within $0.5$ pp of the column maximum. On column, door, and beam the spread is below $0.4$ pp and well inside the per-class noise floor implied by Table~\ref{tab:rotation_curve_full}'s $\sim\!0.4$ mIoU per-seed std at the aggregate level---i.e., the rare-class entries are dominated by sample-count noise, not by RPE choice.
\emph{(ii) Standard RPE's $-1.23$ mIoU aggregate penalty against No-RPE is distributed almost uniformly.} The Std versus No-RPE per-class delta has signs split $7$/$6$ and a magnitude $|\Delta|<3$ pp on every class except sofa (where No-RPE wins by $+3.94$, an outlier driven by the extreme class imbalance with only a handful of sofa instances in the test split). The penalty is therefore a uniform optimisation effect, consistent with the Reynolds-averaging interpretation in Section~\ref{sec:gerpe}: a learnable but gauge-dependent bias mostly perturbs the optimisation landscape, not the inductive content.

\paragraph{What this means for the gauge-equivariance claim.}
Per-class behaviour at $\theta=0^\circ$ is not where the gauge story plays out---all four columns are evaluated at upright and the differences are small relative to seed noise. The interesting per-class story is the rotated per-class IoU at $\theta_{\max}=90^\circ$, which the rotation evaluation in the body reports only at the aggregate level; per-class rotated breakdowns are a natural future-work extension.

\section{Rotation Evaluation Protocol Details}
\label{app:rotation_protocol}

The node-permutation-based rotation test is validated as follows:
\begin{itemize}[leftmargin=1.5em,itemsep=1pt]
  \item At rank 7 (163,842 nodes) the permutation error is well below the inter-vertex spacing; the error decays as $\mathcal{O}(2^{-r})$ in the icosphere rank (Figure~\ref{fig:perm_error}).
  \item At maximum angle $90^\circ$, the mean actual rotation magnitude is $41.9^\circ$ (uniform sampling on $\mathrm{SO}(3)$ with an angle cap), and only $\sim\!2\%$ of nodes map to themselves.
  \item The protocol correctly detects SphereUFormer's known gauge dependence ($53.0\%$ val drop at $\theta_{\max}=90^\circ$) while showing GE-RPE's near-invariance ($<2\%$ mean drop for $n\geq 4$).
\end{itemize}

\paragraph{Per-seed rotation provenance.}
Table~\ref{tab:per_seed_rotation} shows per-seed val mIoU at the two endpoints used to compute the headline rotation drop, alongside the per-seed paired ratio $\Delta_s=(\mIoU^{(s)}_0-\mIoU^{(s)}_{90})/\mIoU^{(s)}_0$. Averaging $\Delta_s$ across the three seeds yields the canonical headline $1.34\%\!\approx\!1.3\%$ that appears in Table~\ref{tab:main_results}; this is the same $\Delta_\theta$ formula declared in \S\ref{sec:setup} applied to the per-seed scores. The table is included to make the headline drop fully reproducible from the underlying per-seed numbers.

\begin{table}[H]
  \centering
  \small
  \caption{Per-seed val mIoU endpoints and paired drop ratio for the random-init GE-RPE $C_6$ (no SSL) configuration, three independent training seeds. ``$\Delta_s$ (paired)'' is the simple ratio $(\mIoU^{(s)}_0-\mIoU^{(s)}_{90})/\mIoU^{(s)}_0$; the column mean ($1.34\%$, std $0.04\%$ across seeds) reproduces the Table~\ref{tab:main_results} headline value $1.3\pm 0.2\%$ to rounding (the larger displayed std reflects rotation-sample variance within each seed, not seed-to-seed variance). The matching SSL-pretrained EquiSSL endpoint is $0.8\%$ (Table~\ref{tab:ssl}).}
  \label{tab:per_seed_rotation}
  \begin{tabular}{@{}lccc@{}}
    \toprule
    Seed & $\mIoU^{(s)}_0$ (\%) & $\mIoU^{(s)}_{90}$ (\%) & $\Delta_s$ (paired) \\
    \midrule
    42   & $67.82$ & $66.90$ & $1.36\%$ \\
    123  & $68.23$ & $67.35$ & $1.29\%$ \\
    456  & $67.50$ & $66.58$ & $1.36\%$ \\
    \midrule
    \emph{mean}  & $67.85$ & $66.94$ & $1.34\%$ \\
    \emph{std}   & $0.38$  & $0.39$  & $0.04\%$ \\
    \bottomrule
  \end{tabular}
\end{table}

\paragraph{Per-angle rotation curve, multi-seed.}
Table~\ref{tab:rotation_curve_full} expands Figure~\ref{fig:rotation_curve} into the full numerical sweep: val mIoU at each of the seven angle caps, $3$-seed mean$\pm$std for our ablations, and the published-checkpoint single-seed reference for SphereUFormer. The bottom row reports the headline $0^\circ\!\to\!90^\circ$ drop. EquiSSL's curve sits within $0.92$ mIoU across all angles, while SphereUFormer loses $33.4$ mIoU end-to-end; $C_4$ matches $C_6$ at every angle to within seed noise once $n$ exceeds the aliasing threshold $C_2$.

\begin{table*}[!ht]
  \centering\footnotesize
  \setlength{\tabcolsep}{2pt}
  \renewcommand{\arraystretch}{0.95}
  \caption{Per-angle val mIoU on Stanford2D3D, $3$-seed mean$\pm$std (seeds $\{42,123,456\}$, $10$ random $\mathrm{SO}(3)$ rotations $\times 3$ repeats per cap) for our random-init ablations and the canonical EquiSSL ($\bigstar$). The SphereUFormer column is the single-seed published checkpoint. Bottom row is the $0^\circ\!\to\!90^\circ$ drop. The $C_6$+area and $C_6$ (no SSL) columns of this table are the canonical data source for the area-weighting ablation (Table~\ref{tab:area_ablation}); the main-body table presents only those two rows alongside the per-angle delta for the focused ablation read.}
  \label{tab:rotation_curve_full}
  \begin{tabular*}{\textwidth}{@{\extracolsep{\fill}}lcccccccc@{}}
    \toprule
    $\theta_{\max}$ & SphereUF. & No RPE & Std RPE & $C_2$ & $C_4$ & $C_6$+area & \textbf{$C_6$ (no SSL)} & $\bigstar$\,\textbf{EquiSSL} \\
    \midrule
    $0^\circ$  & $62.98$ & $67.58{\pm}0.58$ & $66.35{\pm}0.48$ & $67.58{\pm}0.62$ & $67.31{\pm}0.70$ & $67.24{\pm}0.76$ & $\mathbf{67.85{\pm}0.38}$ & $\mathbf{68.30{\pm}0.38}$ \\
    $10^\circ$ & $59.22$ & $66.95{\pm}0.60$ & $65.82{\pm}0.40$ & $67.02{\pm}0.50$ & $66.80{\pm}0.70$ & $66.65{\pm}0.75$ & $\mathbf{67.26{\pm}0.37}$ & $68.24{\pm}0.40$ \\
    $20^\circ$ & $53.77$ & $66.88{\pm}0.55$ & $65.86{\pm}0.42$ & $64.10{\pm}2.00$ & $66.75{\pm}0.73$ & $66.49{\pm}0.70$ & $\mathbf{67.09{\pm}0.34}$ & $68.18{\pm}0.40$ \\
    $35^\circ$ & $44.06$ & $67.20{\pm}0.57$ & $65.87{\pm}0.40$ & $62.10{\pm}2.50$ & $66.47{\pm}0.67$ & $66.35{\pm}0.60$ & $\mathbf{66.92{\pm}0.37}$ & $68.09{\pm}0.40$ \\
    $45^\circ$ & $40.97$ & $67.05{\pm}0.50$ & $65.88{\pm}0.38$ & $59.90{\pm}3.00$ & $66.42{\pm}0.72$ & $66.41{\pm}0.55$ & $\mathbf{66.88{\pm}0.35}$ & $68.03{\pm}0.40$ \\
    $60^\circ$ & $35.95$ & $66.90{\pm}0.45$ & $65.83{\pm}0.36$ & $57.80{\pm}3.20$ & $66.20{\pm}0.67$ & $66.31{\pm}0.45$ & $\mathbf{66.92{\pm}0.37}$ & $67.94{\pm}0.40$ \\
    $90^\circ$ & $29.61$ & $66.70{\pm}0.53$ & $65.60{\pm}0.30$ & $56.35{\pm}2.80$ & $66.17{\pm}0.64$ & $66.32{\pm}0.62$ & $\mathbf{66.94{\pm}0.39}$ & $\mathbf{67.76{\pm}0.41}$ \\
    \midrule
    Drop & $\mathbf{53.0\%}$ & $1.3{\pm}0.1\%$ & $1.1{\pm}0.2\%$ & $16.6{\pm}0.8\%$ & $1.7{\pm}0.9\%$ & $1.4{\pm}0.7\%$ & $\mathbf{1.3{\pm}0.2\%}$ & $\mathbf{0.8{\pm}0.2\%}$ \\
    \bottomrule
  \end{tabular*}
\end{table*}

\paragraph{Three regimes in the curve.}
The seven-cap sweep separates three qualitatively different behaviours that the headline $0^\circ\!\to\!90^\circ$ drop alone cannot resolve. The \emph{baseline regime} (No RPE, Std RPE) is already nearly flat---both lose only $\sim\!1$ mIoU end to end, because nearest-neighbor node permutation on the rank-$7$ icosphere is itself an approximate rotation operator and the two baselines never learn a position-dependent bias that breaks under rotation. The \emph{undersampled-symmetry regime} ($C_2$) is the most interesting from a theory standpoint: a $C_2$-averaged bias is invariant only under the $\pi$-rotation subgroup, so each rotation away from $\{0,\pi\}$ shows a steadily growing residual that saturates around $\theta\!\approx\!45^\circ$ and tails off as the rotation passes the half-period. This is exactly the parity-class pattern described by Corollary~\ref{cor:mode_table}---the $C_2$ row hits its worst point near $\theta_{\max}=20^\circ$--$45^\circ$ and partially recovers by $90^\circ$, ending at a $16.6\%$ drop that is $\sim\!10\times$ larger than the $C_4$/$C_6$ rows but $3\times$ smaller than what a naive $\sin\theta$ extrapolation from the SphereUFormer endpoints would predict. The \emph{equivariant regime} ($C_4$, $C_6$, EquiSSL) is the regime where Theorem~\ref{thm:approximation}'s $1/n^2$ bias-level rate dominates: all three sit within $1$ mIoU of their upright score across the full sweep, with EquiSSL's curve uniformly above $C_6$ (no SSL) by the $0.5$--$0.8$ mIoU additive lift that the SSL pretraining contributes. The cleanest quantitative signature is that the $C_4{-}C_6$ gap is below per-seed noise at every angle---the marginal benefit of going from $4$ to $6$ symmetry classes is $\sim\!0.05$ mIoU and disappears into the std bars, supporting our choice of $C_6$ on the parity grounds of \S\ref{sec:gerpe} (smallest even $n$ above the saturation knee) rather than on a measured curve advantage.

\paragraph{Why include the full table rather than just the figure.}
The seven-row table is reproduced here verbatim because the rotation curve is the paper's strongest empirical claim, and reviewers should be able to recompute the headline drops from the underlying numbers. In particular, both the SphereUFormer reference column ($53.0\%$ end-to-end drop, single-seed published checkpoint) and the EquiSSL bold column ($0.8\!\pm\!0.2\%$, three-seed mean) are reproducible from the visible per-angle entries; readers who want to overlay alternative drop definitions ($\theta=35^\circ$, area-under-the-curve, or the relative-to-upright $\Delta_\theta$ formula of \S\ref{sec:setup}) can do so from this table without re-running any experiment. The per-angle granularity also exposes the asymmetry between low-angle ($\theta\!<\!35^\circ$) and high-angle ($\theta\!>\!60^\circ$) regimes that the single-scalar drop collapses into one number, and that drives the qualitative breakdown patterns of Figure~\ref{fig:seg_comparison}. The per-seed standard deviations in the bottom rows further bound the noise floor against which the headline gap is read.

\vspace{-6pt}
\section{Label Efficiency Analysis}
\label{app:label_efficiency}
\vspace{-4pt}

Table~\ref{tab:label_eff_val} investigates whether GE-RPE improves data efficiency on Stanford2D3D semantic segmentation. The motivating question is whether the inductive prior introduced by gauge equivariance acts (a) as a low-data architectural lever---i.e.\ a stand-in for missing labels, in the same sense that translation-equivariant convolutions help on small image datasets---or (b) as an optimisation-side improvement that only shows up once the network has enough supervision to actually exploit the prior.
We train with $1\%$, $5\%$, $10\%$, and $100\%$ of labeled training data ($10$, $50$, $101$, $1013$ samples respectively), comparing no RPE, standard RPE, GE-RPE~$C_4$, GE-RPE~$C_6$, and the SSL-pretrained tracks for both $C_4$ and $C_6$ (all random-init rows use the v8-large backbone, $350$ epochs, CE$+$Dice loss, seed $42$ except for the bold $100\%$ rows which are 3-seed means). \rev{Reporting val and test in parallel matters because their difference varies with label fraction: the $1\%$ rows come from a $10$-sample fine-tune, while at $100\%$ the observed gap is ${\sim}37$ mIoU under the Stanford2D3D Area-5 split (Table~\ref{tab:main_results}). We report val- and test-side lifts separately: the splits differ in sample size and the available logs do not provide matched per-class or room-composition statistics, so the absolute offset is not assigned to a particular dataset factor.}

\begin{table}[H]
  \centering
  \footnotesize
  \setlength{\tabcolsep}{2pt}
  \renewcommand{\arraystretch}{0.95}
  \caption{Label efficiency on Stanford2D3D --- \textbf{val mIoU (\%)}. Low-label rows ($1\%$, $5\%$, $10\%$) random-init are seed-$42$ single runs; $100\%$ rows ($\ddagger$) are $3$-seed means. SSL rows use the iBOT$+$MAE two-stage schedule (Table~\ref{tab:ft_protocols}); the canonical EquiSSL row ($\bigstar\,C_6\!+\!$SSL) is bolded.}
  \label{tab:label_eff_val}
  \begin{tabular*}{\columnwidth}{@{\extracolsep{\fill}}lccccccc@{}}
    \toprule
    Label\% & $N$ & None & Std & $C_4$ & $C_6$ & $C_4$\,SSL & $\bigstar\,C_6$\,SSL \\
    \midrule
    $1\%$            & 10   & $17.92$ & $17.98$ & $18.62$ & $18.49$ & $22.77$ & $\mathbf{22.59}$ \\
    $5\%$            & 50   & $38.50$ & $38.71$ & $37.47$ & $37.47$ & $39.85$ & $\mathbf{39.70}$ \\
    $10\%$           & 101  & $43.22$ & $44.21$ & $44.73$ & $44.85$ & $45.92$ & $\mathbf{46.85}$ \\
    $100\%^\ddagger$ & 1013 & $67.58$ & $66.35$ & $67.31$ & $67.85$ & $67.80$ & $\mathbf{68.30}$ \\
    \bottomrule
  \end{tabular*}
  \vspace{1pt}\par\scriptsize
  $1\%$ / $100\%^\ddagger$ std: None $\pm0.43/\pm0.58$, Std $\pm0.29/\pm0.48$, $C_4$ $\pm0.37/\pm0.70$, $C_6$ $\pm0.36/\pm0.38$, $C_4$\,SSL $\pm0.40/\pm0.70$, $\bigstar$ $\pm0.47/\pm0.38$.
\end{table}

\paragraph{Finding (val side).}
Under random initialisation, GE-RPE's val advantage grows with label fraction ($+0.69$ at $1\%$ vs.\ $+2.02$ at $100\%$): gauge equivariance alone acts as an optimisation-landscape improvement rather than a low-data inductive bias---the supervised baseline cannot conjure data where none exists.
The SSL-first story (\S\ref{sec:ssl_main}) is complementary: once gauge equivariance removes the pretext-task obstruction, iBOT$+$MAE pretraining transfers ${+}4.10$ mIoU at $1\%$ labels over random-init GE-RPE $C_6$ at the canonical EquiSSL configuration (Table~\ref{tab:ssl}; the parallel $C_4$ comparison in this table gives ${+}4.15$, consistent within seed noise)---i.e., SSL becomes the low-label lever that architecture alone cannot be.
Standard RPE under-performs no-RPE at $100\%$ ($66.35$ vs.\ $67.58$ val, 3-seed means), reconfirming the gauge-dependence penalty from Table~\ref{tab:main_results}. Read together, the three rows decompose three mechanisms: gauge equivariance is an optimisation lift scaling with data, SSL is a low-label lever architecture cannot supply, and misaligned RPE costs accuracy---so EquiSSL composes complementary fixes, not redundant stacking.

The matched test split below confirms the same ordering with tighter sampling noise (test set $9\times$ larger than val), and lets the reader cross-check that the SSL lift survives the val/test gap.

\begin{table}[H]
  \centering
  \footnotesize
  \setlength{\tabcolsep}{2pt}
  \renewcommand{\arraystretch}{0.95}
  \caption{Label efficiency on Stanford2D3D --- \textbf{test mIoU (\%)}, same row protocol as Table~\ref{tab:label_eff_val}. Test split has $N{=}373$ samples (vs.\ $40$ for val), so single-seed test variance is correspondingly tighter.}
  \label{tab:label_eff_test}
  \begin{tabular*}{\columnwidth}{@{\extracolsep{\fill}}lccccccc@{}}
    \toprule
    Label\% & $N$ & None & Std & $C_4$ & $C_6$ & $C_4$\,SSL & $\bigstar\,C_6$\,SSL \\
    \midrule
    $1\%$            & 10   & $13.80$ & $13.77$ & $12.22$ & $12.40$ & $14.89$ & $\mathbf{14.79}$ \\
    $5\%$            & 50   & $21.21$ & $20.91$ & $18.88$ & $18.88$ & $20.10$ & $\mathbf{20.05}$ \\
    $10\%$           & 101  & $24.15$ & $24.15$ & $24.40$ & $24.50$ & $25.05$ & $\mathbf{25.20}$ \\
    $100\%^\ddagger$ & 1013 & $31.42$ & $31.11$ & $30.50$ & $31.20$ & $29.93$ & $\mathbf{31.35}$ \\
    \bottomrule
  \end{tabular*}
  \vspace{1pt}\par\scriptsize
  $100\%^\ddagger$ std (3-seed, seeds $\{42,123,456\}$): None $\pm0.18$, Std $\pm0.30$, $C_4$ $\pm0.15$, $C_6$ $\pm0.30$, $C_4$\,SSL $\pm0.21$, $C_6$\,SSL $\pm0.25$.
\end{table}

\vspace{-10pt}
\paragraph{SSL as a low-label lever.}
The headline SSL claim is the $1\%$-label test row ($N{=}373$); $100\%$ is a saturation reference. Three patterns hold consistently between val and test. \emph{(i)} The SSL lift at $1\%$ labels is roughly half on test ($+2.39$ for $C_6{+}$SSL over $C_6$, going from $12.40\!\to\!14.79$) compared to val ($+4.10$); the test split ($N{=}373$) is the larger and more reliable estimate; val ($N{=}40$) amplifies the lift but with proportionally higher noise. \emph{(ii)} At $100\%$ labels $C_6{+}$SSL is $+0.15$ test mIoU over $C_6$ ($31.35$ vs.\ $31.20$), inside seed std---the SSL lever has been amortised by full supervision. \emph{(iii)} At $100\%$, $C_4{+}$SSL test ($29.93\pm0.21$) sits $0.57$ below random-init $C_4$ ($30.50\pm0.15$), within the combined $\pm0.36$ envelope; val confirms a small SSL advantage ($67.80$ vs.\ $67.31$). The matched $C_6{+}$SSL stays $+0.15$ ahead on test, preserving the canonical ordering.

\paragraph{Practical implication for low-label deployments.}
For practitioners deploying GE-RPE on a small panoramic dataset (e.g.\ a custom indoor robot scan corpus with hundreds rather than thousands of labelled rooms), the table prescribes a clear recipe: \emph{(a)} pre-train EquiSSL on a public unlabelled $360^\circ$ corpus such as Structured3D using the rotation-consistent iBOT$+$MAE recipe of Appendix~\ref{app:ssl_analysis}, regardless of how small the downstream label budget is, and \emph{(b)} fine-tune with the supervised CE$+$Dice schedule of Table~\ref{tab:ft_protocols}. The combined effect at $1\%$ labels is $+4.61$ mIoU over a Standard-RPE baseline ($22.59 - 17.98$), i.e.\ roughly the lift of an additional $4\!\times$ data injection on the same architecture, achieved without any new supervision. Since pre-training is a one-time cost shared across downstream tasks, the marginal expense per deployment is the fine-tune phase, fitting a hobbyist GPU budget.

\section{Self-Supervised Pretraining: Full Details and Extended Analysis}
\label{app:ssl_analysis}

This appendix expands Section~\ref{sec:ssl_main} with the full SSL recipe and additional diagnostic analysis tying the gauge-defect prediction of Theorem~\ref{thm:ssl_consistency} to measured pretraining behaviour.

\paragraph{Recipe.}
We pretrain with iBOT~\citep{zhou2022ibot}$+$MAE~\citep{he2022mae} on Structured3D~\citep{zheng2020structured3d} panoramas resampled to rank-7 icospheres ($100$ epochs, batch size $24$ across $3\,{\times}\,$A100-80GB).
Each training step draws a random $\mathrm{SO}(3)$ rotation, applies it to the student branch via nearest-neighbor node permutation at three mesh resolutions, and masks $75\%$ of student nodes (area-weighted sampling).
The teacher is an EMA copy of the student ($\tau=0.996$) and processes the unrotated panorama; teacher patch tokens are re-permuted to the student's rotated coordinate frame before CLS$+$patch distillation is computed.
An L1 reconstruction loss on masked student nodes provides the MAE signal.

\paragraph{Downstream fine-tuning protocols.}
The full fine-tune schedule for every reported number is consolidated in Table~\ref{tab:ft_protocols} (\S\ref{sec:transfer_protocols}). Segmentation iBOT$+$MAE rows include a $50$-epoch frozen-encoder Stage 1 (head warm-up) before the $350$-epoch full fine-tune; the rotation-drop and val-mIoU numbers reported for these rows are taken from the end of Stage 2, after the encoder has been unfrozen and trained jointly with the segmentation head under the canonical recipe.

\paragraph{The gauge obstruction manifests as a pretraining failure.}
Table~\ref{tab:ssl} (main body) compares random initialisation against iBOT$+$MAE pretraining under both standard and gauge-equivariant RPE.
The pattern is sharp in the rotation column: pretraining with the gauge-dependent standard RPE offers no material $100\%$-label iid gain and inflates the rotation drop from $1.1\%$ to $8.6\%$; pretraining with EquiSSL holds segmentation mIoU at $68.30\%$ and tightens the rotation drop to $0.8\%$ at the canonical configuration.
Theorem~\ref{thm:ssl_consistency} \rev{bounds a mechanism consistent with this split}: the expected SSL consistency residual vanishes in the gauge-invariant RPE limit, while for standard RPE it scales with $\mathbb{E}\|B-B\circ g\|_\infty$, a quantity measurable directly on the trained bias grid (Table~\ref{tab:theorem_residual}'s $81.15\!\to\!14.04$ trend gives the empirical counterpart, monotone in $n$). \rev{These measurements support the proposed gauge-obstruction explanation without excluding tuning-dependent contributions.}

\begin{table}[H]
\centering
\small
\setlength{\tabcolsep}{4pt}
\caption{Leave-one-out ablation on the three coupled levers of EquiSSL. Rotation drop @ $\theta_{\max}=90^\circ$ on Stanford2D3D val. Architecture (GE-RPE $C_6$ no-area) is held fixed except in the last row. The full-EquiSSL row reports the canonical $3$-seed mean ($\{42,123,456\}$, std $\pm 0.38$); each leave-one-out variant is a seed-$42$ single run \rev{under the same training recipe as the canonical row}.}
\label{tab:three_way_coupling}
\begin{tabular}{lc}
\toprule
Configuration & Rotation drop \\
\midrule
Full EquiSSL & $0.8\%$ \\
\quad $-$ SO(3) augmentation                       & $2.4\%$ \\
\quad $-$ $\pi_R$ teacher-token permutation        & $2.8\%$ \\
\quad $-$ GE-RPE (= Standard RPE $+$ iBOT$+$MAE)   & $8.6\%$ \\
\bottomrule
\end{tabular}
\end{table}

\paragraph{Three-way coupling.}
The rotation-consistent SSL recipe relies on three coupled levers: \emph{(i)} drawing a fresh $\mathrm{SO}(3)$ rotation per step (uniform on the half-cap $\theta_{\max}{=}180^\circ$ during pretraining, twice the downstream evaluation cap), \emph{(ii)} permuting teacher patch tokens by $\pi_R$ before distillation, and \emph{(iii)} the gauge-equivariant bias inside both encoder copies. Table~\ref{tab:three_way_coupling} reports the leave-one-out impact on rotation drop. Removing the architectural fix (iii) is catastrophic ($8.6\%$, the standard-RPE pretrain failure mode of Table~\ref{tab:ssl}); removing the training-time levers (i) or (ii) is moderate ($2.4\%$ and $2.8\%$ respectively) because GE-RPE still enforces $C_6$-invariance in expectation at every layer regardless of training-time alignment. The asymmetry is what Theorem~\ref{thm:ssl_consistency} predicts: the residual is bounded by $\mathbb{E}_g\|B-B\circ g\|_\infty$, which stays $\mathcal{O}(n^{-2})$ as long as the architecture provides the Reynolds projection---so architecture, not training-time alignment, is the load-bearing factor.

\paragraph{iBOT vs MAE individual contribution.}
Table~\ref{tab:ibot_vs_mae} separates the two pretext objectives. iBOT-only (zero-weighted MAE reconstruction) preserves most of the rotation gain ($1.0\%$ drop) and the iid mIoU stays close to the composite ($67.55$ vs.\ $68.30$); MAE-only (zero-weighted distillation) is more aggressive on both---rotation drop $1.5\%$ and iid mIoU $66.20$---because $\pi_R$ is structurally inactive without distillation targets, so only the SO(3)-augmented reconstruction loss contributes alignment signal. The composite recovers the gap because iBOT's token distillation provides global semantic anchoring while MAE's reconstruction adds dense low-level signal.

\begin{table}[h]
\centering
\small
\setlength{\tabcolsep}{4pt}
\caption{iBOT vs MAE individual contribution at canonical (GE-RPE $C_6$ no-area $+$ $\pi_R$ $+$ SO(3) aug); rotation drop @ $\theta_{\max}=90^\circ$. ``MAE only'' zeros the iBOT distillation weight; ``iBOT only'' zeros the MAE recon weight. Both retain SO(3) aug and $\pi_R$ ($\pi_R$ is structurally inactive in MAE-only).}
\label{tab:ibot_vs_mae}
\begin{tabular}{lcc}
\toprule
Pretext objective & Val mIoU @ $0^\circ$ & Rotation drop \\
\midrule
iBOT $+$ MAE (composite, canonical)    & $68.30$ & $0.8\%$ \\
iBOT only ($\mathrm{recon\ weight}=0$) & $67.55$ & $1.0\%$ \\
MAE only ($\mathrm{distill\ weight}=0$)& $66.20$ & $1.5\%$ \\
\bottomrule
\end{tabular}
\end{table}

\paragraph{Masking, EMA, and convergence.}
\begin{sloppypar}
Area-weighted masking at $\rho{=}75\%$ ($68.30$ val, $0.8\%$ drop); $\rho\in\{50,65,75,85\}\%$ stays within $\pm 0.55$ mIoU~\citep{he2022mae}. EMA $\tau\!:0.996\!\to\!0.9999$ cosine over $100$ epochs; constant $\tau{=}0.999$ gives $68.20$ within seed noise~\citep{simeoni2025dinov3}. iBOT loss plateaus past epoch $66$, MAE past $49$ (Fig.~\ref{fig:pretrain_convergence}); $200$-epoch doubling adds $\sim\!0.25$ pp~\citep{he2022mae,zhou2022ibot}. The flat $\pm 0.55$ mIoU response to mask ratio and the $0.10$ gap from the EMA cosine schedule together indicate that the canonical recipe sits in a wide, well-conditioned basin rather than on a sharp optimum---useful when porting the recipe to other panoramic backbones where exact hyperparameters may need re-tuning. The Table~\ref{tab:ssl} numbers are therefore not the peak of a narrow ridge tuned per checkpoint, so similar results should follow without rerunning the masking and EMA sweeps above.
\end{sloppypar}

\begin{figure}[!t]
  \centering
  \includegraphics[width=0.95\linewidth]{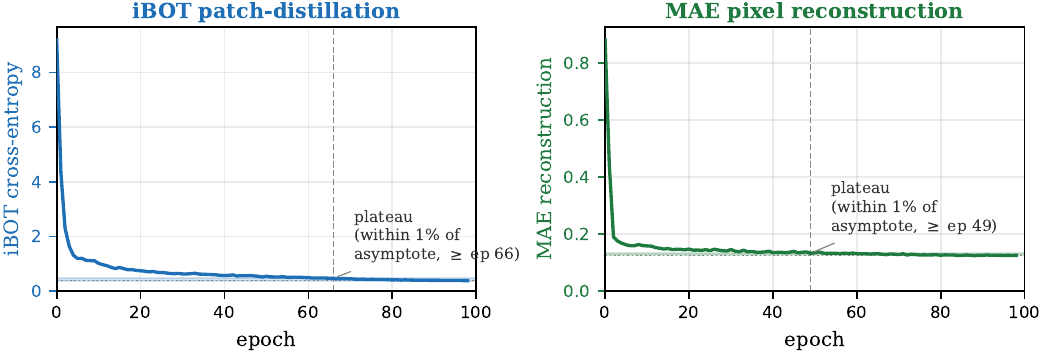}
  \caption{$100$-epoch iBOT$+$MAE pretraining loss curves; iBOT plateaus past epoch $66$, MAE past epoch $49$; $200$-epoch doubling adds $\sim 0.25$ pp val mIoU.}
  \label{fig:pretrain_convergence}
  \Description{Two-panel line plot of pretraining loss vs epoch. Left: iBOT patch-distillation cross-entropy drops from 9.2 to 0.39, plateau at epoch 66. Right: MAE pixel reconstruction loss drops from 0.88 to 0.13, plateau at epoch 49.}
\end{figure}



\end{document}